\documentclass[twoside,11pt]{article} 

\usepackage{blindtext}

\usepackage[preprint]{jmlr2e}

\RequirePackage{amsmath, amssymb, mathtools,times,hyperref}

\newcommand\E{\mathbb{E}}

\newcommand\Dir{\mathrm{Dir}}

\newcommand\KL{\mathrm{KL}}
\newcommand\1{\mathbb{I}}
\newcommand\Mult{\mathrm{Mult}}

\newcommand\F{\mathcal{F}}
\newcommand\G{\mathcal{G}}
\newcommand\Rhat{\hat{\mathcal{R}}}
\newcommand\R{\mathcal{R}}
\newcommand\N{\mathcal{N}}

\DeclareMathOperator{\argmin}{argmin}

\newtheorem{assumption}{Assumption}
\newtheorem{open_question}{Open Question}

\allowdisplaybreaks

\usepackage{lastpage}
\jmlrheading{xx}{xxxx}{1-\pageref{LastPage}}{2/23; Revised xxx}{xx}{21-0000}{C. Riou and P. Alquier and B.-E. Ch\'erief-Abdellatif}

\ShortHeadings{Fast Rates in Meta-Learning with PAC-Bayes}{C. Riou and P. Alquier and B.-E. Ch\'erief-Abdellatif}
\firstpageno{1}

\begin{document}

\title{Bayes meets Bernstein at the Meta Level: an Analysis of Fast Rates in Meta-Learning with PAC-Bayes }

\author{\name Charles Riou \email charles@ms.k.u-tokyo.ac.jp \\
       \addr The University of Tokyo \& RIKEN Center for AIP \\
       Tokyo, Japan
       \AND
       \name Pierre Alquier \email alquier@essec.edu \\
       \addr ESSEC Business School \\
       Asia-Pacific campus, Singapore
       \AND
       \name Badr-Eddine Ch\'erief-Abdellatif \email badr-eddine.cherief-abdellatif@cnrs.fr \\
     \addr CNRS, LPSM, Sorbonne Universit\'e \& Université Paris Cit\'e }
\editor{My editor}

\maketitle

\begin{abstract}
\noindent Bernstein's condition is a key assumption that guarantees fast rates in machine learning. For example, the Gibbs algorithm with prior $\pi$ has an excess risk in $O(d_{\pi}/n)$, as opposed to the standard $O(\sqrt{d_{\pi}/n})$, where $n$ denotes the number of observations and $d_{\pi}$ is a complexity parameter which depends on the prior $\pi$. In this paper, we examine the Gibbs algorithm in the context of meta-learning, i.e., 
when learning the prior $\pi$ from $T$ tasks (with $n$ observations each) generated by a meta distribution.
Our main result is that Bernstein's condition always holds at the meta level, regardless of its validity at the observation level. 
This implies that the additional cost to learn the Gibbs prior $\pi$, which will reduce the term $d_\pi$ across tasks, is in $O(1/T)$, instead of the expected $O(1/\sqrt{T})$. 
We further illustrate how this result improves on the
standard rates
in three different settings: discrete priors, Gaussian priors and mixture of Gaussians priors. 
\end{abstract}

\begin{keywords}%
Bernstein's condition, meta-learning, fast rates, PAC-Bayes bounds, information bounds, the Gibbs algorithm, variational approximations.
\end{keywords}

\section{Introduction}

One of the greatest promises of artificial intelligence is the ability to design autonomous systems that can learn from different situations throughout their lives and adapt quickly to new environments, as humans, animals and other living things naturally do. Based on the intuition that a new problem often has significant similarities to previously encountered tasks, the use of past experience is particularly important in areas such as computer vision \citep{quattoni2008transfer,kulis2011you,li2018learning,achille2019task2vec}, natural language processing \citep{huang2018natural,gu2018meta,dou2019investigating,qian2019domain} and reinforcement learning \citep{finn2017model,mishra2018simple,wang2016learning,yu2020meta} where the learner has access to only a few training examples for the task of interest, but for which a vast amount of datasets from a variety of related tasks is available. In the area of digit recognition for example, it is possible to leverage the experience gained from millions of similar open source image classification datasets, as the key features needed to classify cats from dogs or pants from shirts can be used to classify handwritten digits. This idea is at the heart of meta-learning \citep{thrun1998learning,baxter2000model,vanschoren2019meta}, a field that has recently attracted a lot of attention due to its huge success in real-world applications, and which aims to improve performance on a particular task by transferring the knowledge contained in different but related tasks.

\vspace{0.2cm}
Meta-learning has been widely studied in recent literature. It must be noted that \textit{meta-learning} was used to refer to a wide range of situations. Providing a precise definition of meta-learning is a challenge. In particular, the terms \textit{transfer learning} and \textit{multi-task learning}, although distinct, are often used interchangeably instead of meta-learning.
\textit{Transfer learning} is a very general concept that involves two tasks that share similarities - a source and a target - and consists in transferring the knowledge acquired on the source dataset to better process the target data \citep{pan2010survey,zhuang2020comprehensive}. 
In practice, this can be formulated in many different ways, but the most popular approach is to pre-train a model on the source data, e.g., images of cats and dogs, and then to fine-tune it on the target training data set, e.g., images of handwritten digits. In particular, a challenging problem in transfer learning is to find a measure that quantifies the similarity between the source and target tasks. 
\textit{Multi-task learning} adopts a different framework, where multiple learning tasks are considered and the goal is to learn a model that can handle all tasks simultaneously \citep{caruanamulti1997,zhang2021survey}. The model usually has a common representation, e.g., the first layers of a deep neural network, and a task-specific component, e.g., the last layer of the network. 
\textit{Meta-learning} also considers a collection of datasets from a variety of tasks, but unlike multi-task learning, we are not interested in learning the fixed number of tasks, but rather in being prepared for future tasks that are not yet given. Also, unlike transfer learning, meta-learning exploits the commonality of previous tasks rather than the similarity between some specific source and target tasks. We use these metadata to design a \textit{meta-procedure} that adaptively learns a predictor for \textit{any} new learning task that is a priori unknown, and the goal is to quickly learn to adapt a learning procedure from past experience. Meta-learning is therefore sometimes referred to as \textit{learning-to-learn}, or \textit{lifelong learning} in the online context. The implementation of this learning-to-learn mechanism can take different forms, which we briefly describe in the following paragraph.

\vspace{0.2cm}
As the name suggests, meta-learning involves two levels of abstraction to improve learning over time: a meta-level and a within-task level. At the within-task level, the new task of interest is presented and the corresponding pattern is learned from the training data set of the task at hand. This learning process is greatly accelerated by a meta-learner, which has distilled the knowledge accumulated in previous tasks into the within-task model. The meta-learning procedure can accelerate the within-task algorithm in various ways, and three main categories stand out in the literature: metric-based methods, which are based on non-parametric predictive models governed by a metric that is learned using the meta-training dataset \citep{koch2015siamese,vinyals2016matching,snell2017prototypical,sung2018learning}; model-based methods, which quickly update the parameters in a few learning steps, which can be achieved by the model's internal architecture or another meta-learning model \citep{santoro2016meta,munkhdalai2017meta,mishra2018simple}; and optimisation-based methods, which mainly involve learning the hyper-parameters of a within-task algorithm using the meta-training set for fast adaptation \citep{hochreiter2001learning,ravi2017optimization,finn2017model,nichol2018first,qiao2018few,gidaris2018dynamic}. Due to their performance and ease of implementation, the optimisation-based family is the dominant class in the recent literature, exploiting the idea that well-chosen hyperparameters can greatly speed up the learning process and allow model parameters to be quickly adapted to new tasks with little data. For example, it is possible to learn the task of interest using a gradient descent whose initialisation and learning rate would have been learned from the metadata. Among the best known meta-strategies is the model agnostic meta-learning procedure (MAML) \citep{finn2017model} and its variants implicit MAML \citep{rajeswaran2019meta}, Bayesian MAML \citep{grant2018recasting,yoon2018bayesian,nguyen2020uncertainty} and Reptile \citep{nichol2018first}. We refer the interested reader to the recent review by \cite{chen2023learning} for more details.

\section{Approach and Contributions}

In this paper, we focus on the Gibbs algorithms within tasks, or their variational approximations. the Gibbs algorithms, also known as Gibbs posteriors~\citep{alq2016} or exponentially weighted aggregation~\citep{dal2008}, can also be interpreted in the framework of Bayesian statistics as a kind of generalized posterior~\citep{bis2016}. PAC-Bayes bounds were developed to control the risk and the excess risk of such procedures~\citep{sha1997,mca1998,cat2004,zhang2006information,cat2007,yang2019fast}, see~\cite{gue2019,alquier2021user} for recent surveys. More recently, the related mutual information bounds~\citep{russo2019much,haghifam2021towards} were also used to study the excess risk of the Gibbs algorithms~\citep{xu2017MI}. Gibbs posteriors are often intractable, and it is then easier to compute a variational approximation of such a posterior. It appears that PAC-Bayes bounds can also be used on such approximations \citep{alq2016}. Many recent publications built foundations of meta-learning through PAC-Bayes and information bounds~\citep{pen2014,ami2018,DingPAC2021,liu2021pac,rothfuss2021,FaridPAC2021,rothfuss2022,GuanPAC2022,Rezazadeh2022}. These works and the related literature are discussed in detail in Section~\ref{sec:dicussion}. Most of these papers proved empirical PAC-Bayes bounds for meta-learning. These bounds can be minimized, and we obtain both a practical meta-learning procedure, together with a numerical certificate on its generalization. However, these works did not focus on the rate of convergence of the excess risk.

Bernstein's condition a is a low-noise assumption reflecting the inherent difficulty of the learning task \citep{mammentsybakov1999,tsybakov2004,bartlett2006empirical}. While it was initially designed to study the ERM~\citep{bartlett2006empirical}, it characterizes the learning rate of algorithms beyond the ERM. PAC-Bayes bounds and mutual information bounds show that the excess risk of the Gibbs algorithm is in $O(d_{\pi,t}/n)$ when Bernstein's condition is satisfied~\citep{cat2007,grunwald2020fast}, as opposed to the slow rate $O((d_{\pi,t}/n)^{1/2})$ in the general case.  The quantity $d_{\pi,t}$ measures the complexity of task $t$. Importantly, it also depends on the prior distribution $\pi$ used in the algorithm. Similar results hold when we replace the Gibbs algorithm by a variational approximation~\citep{alq2016}.

In the meta-learning setting, we are given $T$ tasks simultaneously. Using the Gibbs algorithm with a fixed $\pi$ in all tasks leads to an average excess risk in 
$O(\E_t [ (d_{\pi,t}/n)^{\alpha}])$, where $\alpha = 1$ when Bernstein's condition holds for each task $t\in \{1, \dots, T\}$, and $\alpha = 1/2$ otherwise. 
Here, $\E_t$ denotes the expectation with respect to a future (out-of-sample) task $t$. This approach is referred to as ``learning in isolation'', because each task is solved 
regardless
of the others. Of course, in meta-learning we want to take advantage of the multiple tasks. For example,~\cite{pen2014} used the Gibbs algorithm at the meta-level, in order to learn a better prior. The expected benefit is to reduce the complexity term $d_{\pi,t}$.

\vspace{0.2cm}
\noindent \textbf{Overview of the paper:} 
\begin{itemize}
    \item In Section~\ref{sec:notations}, we recall existing results on the excess risk of the Gibbs algorithm when learning tasks in isolation, and we introduce Bernstein's condition, a fundamental assumption under which the fast rate $O(\E_t [ d_{\pi,t}/n ])$ is achieved by the Gibbs algorithm. 
    \item In Section~\ref{sec:main}, we prove that regardless of its validity at the within-task level, Bernstein's condition is always satisfied at the meta level. As a consequence of this result, we show that a meta-level the Gibbs algorithm achieves the excess risk $O(\inf_{\pi\in\mathcal{M}} \E_t [\left(d_{\pi,t}/n\right)^\alpha ] + 1/T)$ with $\alpha = 1$ if Bernstein's condition is satisfied, and $\alpha = 1/2$ otherwise. We further raise the open question of the generalization of this result to its variational approximations. 
    \item In Section \ref{sec:appli}, we apply the previous results to various settings: learning a discrete prior, learning a Gaussian prior and learning a mixture of Gaussians prior. We show that the gain brought from the meta learning is blatant, as in some favorable situations, one can even have $\inf_{\pi\in\mathcal{M}} \mathbb{E}_t[  d_{\pi,t}/n ] = 0$.
    \item In Section \ref{sec:dicussion}, we provide a deeper comparison with the rich literature on meta-learning.
\end{itemize}

\section{Problem Definition and Notations}
\label{sec:notations}

Let $\mathcal{Z}$ be a space of observations, $\Theta$ be a decision space and $\ell:\mathcal{Z}\times \Theta \rightarrow \mathbb{R}_+$ be a bounded loss function defined on the previously defined sets. Let $\mathcal{P}(\Theta)$ denote the set of all probability distributions on $\Theta$ equipped with a suitable $\sigma$-field. The learner has to solve $T$ tasks. Given a task $t\in \{1, \dots, T\}$, the learner receives the observations $Z_{t, i}, i=1\dots n$, assumed to be drawn independently from a distribution $P_t$ on the decision space $\mathcal{Z}$. The objective of the learner is to find a parameter $\theta$ in the parameter space $\Theta$ which minimizes the so-called prediction risk associated to $P_t$ on $\mathcal{Z}$, defined as
\begin{equation*}
    R_{P_t}(\theta) = \E_{Z\sim P_t}[\ell(Z,\theta)].
\end{equation*}

\noindent We denote by $R^*_{P_t}$ the minimum of $R_{P_t}(\theta)$ and by $\theta_t^*$ its corresponding minimizer:
\begin{equation*}
    R^*_{P_t} = \inf_{\theta\in\Theta} R_{P_t}(\theta) = R_{P_t}(\theta^*_t).
\end{equation*}

\noindent In Bayesian approaches, we rather seek for $\rho_t \in \mathcal{P}(\Theta)$ such that
\begin{equation*}
    \E_{\theta\sim\rho_t}[R_{P_t}(\theta)]
\end{equation*}
is small. Defining, for any $\theta\in \Theta$, the empirical risk as
\begin{equation*}
    \hat{R}_t(\theta) = \frac{1}{n} \sum_{i=1}^n \ell(Z_{t, i}, \theta),
\end{equation*}
a standard choice for $\rho_t$ is the so-called Gibbs posterior  given by
\begin{equation}
    \rho_t(\pi,\alpha) = \argmin_{\rho\in\mathcal{P}(\Theta)} \left\{\E_{\theta\sim\rho}\left[ \hat{R}_t(\theta)\right] + \frac{\KL(\rho\Vert\pi)}{\alpha n}\right\}, \label{gibbs_posterior}
\end{equation}
where $\pi$ is the prior distribution on the parameter $\theta$ and $\alpha$ is some parameter which will be made explicit later. The corresponding risk estimate is defined by
\begin{equation*}
    \Rhat_t(\rho,\pi,\alpha) =  \E_{\theta\sim\rho}\left[\hat{R}_t(\theta)\right] + \frac{\KL(\rho\Vert\pi)}{\alpha n}.
\end{equation*}
More generally, we can also consider variational approximations, defined by
\begin{equation}
    \rho_t(\pi,\alpha,\mathcal{F}) = \argmin_{\rho\in\mathcal{F}} \left\{\E_{\theta\sim\rho}\left[ \hat{R}_t(\theta)\right] + \frac{\KL(\rho\Vert\pi)}{\alpha n}\right\}, \label{vb_posterior}
\end{equation}
where $\F\subseteq \mathcal{P}(\Theta) $. Adequate choices for $\F$ lead to feasible minimization problems, and don't affect the generalization properties~\citep{alq2016}. With these notations, $\rho_t(\pi,\alpha)=\rho_t(\pi,\alpha,\mathcal{P}(\Theta))$.

\subsection{Assumptions on the loss and Bernstein's condition}

Recall that we assumed the loss function to be bounded: there exists a constant $C>0$ such that
\begin{equation}
    \forall (z, \theta) \in \mathcal{Z}\times \Theta, \ \ell(z, \theta) \leq C. \label{bounded_assumption}
\end{equation}

\noindent PAC-Bayes bounds for unbounded losses are well-known, see the discussion at the end of~\cite{alquier2021user}. However, a lot of those bounds become a little more complicated. Thus, the choice to work with bounded variables is made here for the sake of clarity, and is not due to a fundamental limitation of our method.

We define the variance term for task $t$ by
$$ V_t(\theta, \theta^*_t) := \E_{P_t}\left[\left|\ell(Z,\theta) - \ell(Z,\theta^*_t)\right|^2\right]$$
for any $\theta\in\Theta$. 
The following condition is crucial in the study of the risk of Gibbs posterior.
\begin{assumption}[Bernstein's condition]\label{bernstein_hypothesis}
There exists a constant $c>0$ such that, for any $\theta\in \Theta$,
\begin{equation*}
    V_t(\theta,\theta^*_t) \leq c \left( R_{P_t}(\theta) - R_{P_t}^* \right).
\end{equation*}
\end{assumption}

\noindent This assumption characterizes the excess risk of Gibbs posterior, see Theorem \ref{theorem_bound_isolation} below. In this paper, we will provide a bound on the excess risk both under this condition and without it.

In some of the applications developed in Section~\ref{sec:appli}, we will also make the following assumption: there exists $L>0$ such that, for any $P_t\sim \mathcal{P}$ and any $\theta\in \Theta$,
\begin{equation}
    R_{P_{t}}(\theta) - R_{P_{t}}^* \leq L\Vert\theta - \theta^*_t\Vert^2. \label{application_assumption}
\end{equation}

\noindent Intuitively, Taylor's expansion of $R_{P_t}$ gives
\begin{equation*}
    R_{P_{t}}(\theta) = R_{P_{t}}(\theta^*_t) + \underbrace{dR_{P_{t}}(\theta^*_t).(\theta - \theta^*_t)}_{0} + O(\Vert\theta - \theta_t^*\Vert^2) = R_{P_{t}}(\theta^*_t) + O(\Vert\theta - \theta_t^*\Vert^2). 
\end{equation*}
and thus we can expect~\eqref{application_assumption} to be satisfied when the risk is smooth enough. However, note that this assumption is not necessary for the main results of this paper to hold, and will only be used in very specific applications.

\subsection{Learning in Isolation}

In the process of learning in isolation, we consider each of the tasks separately. We then fix some $t\in \{1, \dots, T\}$ and denote by $\mathcal{S}_t$ the set of observations from task $t$: $Z_{t, 1}, \dots, Z_{t, n}$. We recall the following result, which can be found, e.g., in \citet{alquier2021user}, and whose proof is recalled in Appendix \ref{theorem_bound_isolation_proof} for the sake of completeness.
\begin{theorem}\label{theorem_bound_isolation}
Assume that the loss $\ell$ satisfies~\eqref{bounded_assumption}. Then, the following bound holds, for any $\alpha>0$:
\begin{equation*}
    \E_{\mathcal{S}_t} \E_{\theta\sim\rho_t(\pi,\alpha,\mathcal{F})}\left[R_{P_t}(\theta)\right]-R_{P_t}^* \leq \frac{1}{1-\frac{\alpha c \1_{B}}{2(1-C\alpha)}} \left( \E_{\mathcal{S}_t}\left[\inf_{\rho\in \F}\Rhat_t(\rho, \pi, \alpha) - \hat{R}_t(\theta^*_t)\right] + \frac{\alpha C^2(1 - \1_{B})}{8}\right),
\end{equation*}
where $\1_B$ is equal to $1$ if Bernstein's condition (in Assumption~\ref{bernstein_hypothesis}) is satisfied, and $0$ otherwise. In particular, under Bernstein's condition, the choice $\alpha = \frac{1}{c+C}$ yields the bound
\begin{equation*}
    \E_{\mathcal{S}_t} \E_{\theta\sim\rho_t(\pi,\alpha,\mathcal{F})}[R_{P_t}(\theta)] - R_{P_t}^* \leq 2 \E_{\mathcal{S}_t}\left[ \inf_{\rho\in\mathcal{F}} \Rhat_t(\rho, \pi, \alpha)  - \hat{R}_t(\theta^*_t) \right].
\end{equation*}
\end{theorem}

\noindent When specifying a model and a prior, a derivation of the right-hand side in Theorem~\ref{theorem_bound_isolation} leads to explicit rates of convergence. For example, a classical assumption in the Bayesian literature is that there are constants $\kappa,d\geq 0$ such that $\pi(\{\theta:R_t(\theta)-R_t(\theta_t^*)\} \leq s) \geq s^d /\kappa $~\citep{gho2017}. As this condition is usually applied on one task, with a specific prior, the notation $d$ does not reflect the dependence with repect to $\pi$ or to $t$. However, in our context, this dependence will be crucial, so we will write $d_{\pi,t}$ instead of $d$. Under such an assumption, the right-hand side in Theorem \ref{theorem_bound_isolation} can be made more explicit. For simplicity, we state this result for Gibbs posteriors $\rho_t(\pi,\alpha)$ only.
\begin{corollary}
\label{cor:theorem_bound_isolation}
Assume that, almost surely on $P_t$, $\pi(\{\theta:R_t(\theta)-R_t(\theta_t^*)\} \leq s) \geq s^{d_{\pi,t}} /\kappa_{\pi,t} $. Then,
$$ \E_{\mathcal{S}_t}\left[\inf_{\rho\in \mathcal{P}(\Theta)}\Rhat_t(\rho, \pi, \alpha) - \hat{R}_t(\theta^*_t)\right] \leq \frac{d_{\pi,t} \log \frac{n\alpha}{d_\pi} + \log \kappa_{\pi,t} }{\alpha n}. $$

\noindent In particular, under Bernstein's condition, the choice $\alpha=1/(c+C)$ gives
\begin{equation*}
    \E_{\mathcal{S}_t} \E_{\theta\sim\rho_t(\pi,\alpha)}[R_{P_t}(\theta)] \leq R_{P_t}^* + 2 \frac{d_{\pi,t} \log \frac{n\alpha}{d_{\pi,t}} + \log \kappa_{\pi,t} }{\alpha n}.
\end{equation*}

\noindent On the other hand, without Bernstein's condition,
\begin{equation*}
    \E_{\mathcal{S}_t} \E_{\theta\sim\rho_t(\pi,\alpha)}[R_{P_t}(\theta)] \leq R_{P_t}^* +  \frac{d_{\pi,t} \log \frac{n\alpha}{d_{\pi,t}} + \log \kappa_{\pi,t} }{\alpha n}  + \frac{\alpha C^2}{8},
\end{equation*}
and in particular, for $\alpha=2\sqrt{2d_{\pi,t}}/(\sqrt{n}C)$, we obtain
\begin{equation*}
    \E_{\mathcal{S}_t} \E_{\theta\sim\rho_t(\pi,\alpha)}[R_{P_t}(\theta)] \leq R_{P_t}^* + \frac{C}{2} \sqrt{\frac{d_{\pi,t}}{2 n}} \left(\frac{1}{2}\log \frac{8 e^2 n}{d_{\pi} C^2} + \frac{1}{d_{\pi,t}} \log \kappa_{\pi,t}\right).
\end{equation*}
\end{corollary}

\section{Main Results}
\label{sec:main}

From this section onward, we focus on meta-learning. As opposed to the learning in isolation, the meta-learning considers all the tasks $t\in \{1, \dots, T\}$  and takes advantage of possible similarities between the $T$ tasks to improve learning in each task. More precisely, while assuming that for any $t\in \{1, \dots, T\}, Z_{t, 1}, \dots, Z_{t, n}$ are drawn independently from distribution $P_t$, we also assume that the distributions $P_1, \dots, P_T$ are drawn independently from a certain distribution $\mathcal{P}$. A future (out-of-sample) task  $P_{T+1}$ will also be drawn from $\mathcal{P}$ and $Z_{T+1, 1}, \dots, Z_{T+1, n}$ will be drawn independently from $P_{T+1}$. This task will be solved by the Gibbs algorithm $\rho_{T+1}(\pi,\alpha)$. Our objective is to learn the prior $\pi$ using the tasks $t\in \{1, \dots, T\}$, in order to make the meta-risk
\begin{equation*}
    \mathcal{E}(\pi) = \E_{P_{T+1}\sim \mathcal{P}} \E_{\mathcal{S}_{T+1} \sim P_{T+1}} \E_{\theta\sim\rho_{T+1}}[R_{P_{T+1}}(\theta)].
\end{equation*}
as small as possible. We will compare it to the so-called oracle meta-risk
\begin{equation*}
    \mathcal{E}^* = \E_{P_{T+1}\sim \mathcal{P}} [R_{P_{T+1}}^*],
\end{equation*}
which can only be reached by an oracle who would know the best classifier in each task in advance.

\subsection{Bernstein's condition at the meta level}

In this subsection, we prove a version of Bernstein's condition at the meta level. Let
\begin{equation*}
    \R_t(\nu,\alpha) =  \E_{\mathcal{S}_t \sim P_t} \left[\inf_{\rho\in\mathcal{P}(\Theta)} \Rhat_t(\rho, \nu, \alpha)\right]
\end{equation*}
for any prior $\nu$, and let $\pi^*_{\alpha}$ be the distribution minimizing the expectation of the above quantity:
\begin{equation*}
    \pi^*_{\alpha} = \argmin_{\nu} \E_{P_t\sim\mathcal{P}} \left[\R_t(\nu,\alpha)\right].
\end{equation*}

\noindent Surprisingly enough, in contrast to the learning in isolation, Bernstein's condition is always satisfied at the meta level, on the condition that we use Gibbs posteriors $\rho_t(\pi,\alpha)$ within tasks.
\begin{theorem}\label{bernstein_hypothesis_theorem}
Assume that the loss $\ell$ satisfies the boundedness assumption \eqref{bounded_assumption}. Then, for any $\pi\in \mathcal{P}$,
\begin{multline*}
    \E_{P_t}\E_{S_t}\left[\left(\Rhat_t(\rho_t(\pi, \alpha), \pi, \alpha) - \Rhat_t(\rho_t(\pi^*_{\alpha}, \alpha), \pi^*_{\alpha}, \alpha)\right)^2\right] \\
    \leq 
    c \E_{P_t}\E_{S_t}\left[\Rhat_t(\rho_t(\pi, \alpha), \pi, \alpha) - \Rhat_t(\rho_t(\pi^*_{\alpha}, \alpha), \pi^*_{\alpha}, \alpha)\right],
\end{multline*}
where $c = 8eC$.
\end{theorem}

\begin{proof}
It classically holds (e.g.~\citet{alquier2021user}) that
\begin{equation}
    \Rhat_t(\rho_t(\pi, \alpha), \pi, \alpha) = -\frac{1}{n\alpha} \log \E_{\theta\sim\pi}\left[{\rm e}^{-n\alpha \hat{R}_t(\theta)}\right] = -\frac{1}{\tau} \log \left(\E_{\theta\sim\pi}\left[{\rm e}^{-n\alpha \hat{R}_t(\theta)}\right]^{\frac{\tau}{n\alpha}}\right)
    \label{equa:problem}
\end{equation}
for any fixed $\tau>0$. By the boundedness assumption, it holds that, for any $\pi\in \mathcal{P}$,
\begin{equation*}
    \exp(-C\tau)\leq \E_{\theta\sim\pi}\left[{\rm e}^{-n\alpha \hat{R}_t(\theta)}\right]^{\frac{\tau}{n\alpha}} \leq 1.
\end{equation*}

\noindent We next use the following lemma, whose proof can be found in Appendix \ref{function_inequality_lemma_proof}.
\begin{lemma}\label{function_inequality_lemma}
Let $f:x \mapsto -\frac{1}{\tau}\log x$. Then, for any $x, y\in \left[\exp(-C\tau), 1\right]$,
\begin{equation*}
    (f(x)-f(y))^2 \leq \frac{8\exp(2C\tau)}{\tau} \left(\frac{f(x)+f(y)}{2} - f\left(\frac{x+y}{2}\right)\right).
\end{equation*}
\end{lemma}

\noindent An application of Lemma \ref{function_inequality_lemma} to $x = \E_{\theta\sim\pi}\left[{\rm e}^{-n\alpha \hat{R}_t(\theta)}\right]^{\frac{\tau}{n\alpha}}$ and $y=\E_{\theta\sim\pi^*_{\alpha}}\left[{\rm e}^{-n\alpha \hat{R}_t(\theta)}\right]^{\frac{\tau}{n\alpha}}$ gives
\begin{align*}
    &\left(\Rhat_t(\rho_t(\pi, \alpha), \pi, \alpha) - \Rhat_t(\rho_t(\pi^*_{\alpha}, \alpha), \pi^*_{\alpha}, \alpha)\right)^2 \\
    &= (f(x)-f(y))^2 \\
    &\leq \frac{8\exp(2C\tau)}{\tau} \left(\frac{f(x)+f(y)}{2} - f\left(\frac{x+y}{2}\right)\right) \\
    &= \frac{8\exp(2C\tau)}{\tau} \left[\frac{\Rhat_t(\rho_t(\pi, \alpha), \pi, \alpha) + \Rhat_t(\rho_t(\pi^*_{\alpha}, \alpha), \pi^*_{\alpha}, \alpha)}{2} + \frac{1}{\tau}\log\left( \frac{x+y}{2} \right)   \right]
\end{align*}
(these derivations are directly inspired from the proof technique introduced by~\cite{bartlett2003large}).

\noindent Taking expectations with respect to $S_t\sim P_t$ on both sides
\begin{multline*}
      \E_{S_t}\left[\left(\Rhat_t(\rho_t(\pi, \alpha), \pi, \alpha) - \Rhat_t(\rho_t(\pi^*_{\alpha}, \alpha), \pi^*_{\alpha}, \alpha)\right)^2\right] 
      \\
      \leq \frac{8\exp(2C\tau)}{\tau}  \left( \frac{ \R_t(\pi,\alpha) + \R_t(\pi^*_{\alpha},\alpha)}{2} - + \R_t\left(\frac{\pi+\pi^*_{\alpha}}{2},\alpha\right) \right).
\end{multline*}
Integrating with respect to $P_t\sim\mathcal{P}$ yields
\begin{multline}
       \E_{P_t}\E_{S_t}\left[\left(\Rhat_t(\rho_t(\pi, \alpha), \pi, \alpha) - \Rhat_t(\rho_t(\pi^*_{\alpha}, \alpha), \pi^*_{\alpha}, \alpha)\right)^2\right]
      \\
      \leq \frac{8\exp(2C\tau)}{\tau}  \Biggl( \frac{ \E_{P_t} \R_t(\pi,\alpha) + \E_{P_t} \R_t(\pi^*_{\alpha},\alpha)}{2}
    -   \E_{P_t} \R_t\left(\frac{\pi+\pi^*_{\alpha}}{2},\alpha\right) \Biggr). \label{equa:log:4}
\end{multline}
By definition of $\pi^*_{\alpha}$, it holds that, for any $\pi'\in \mathcal{P}$,
\begin{equation*}
    \E_{P_t}[\R_t(\pi^*_{\alpha},\alpha)] \leq  \E_{P_t}[\R_t(\pi',\alpha)].
\end{equation*}

\noindent In particular, this holds for $\pi'=(\pi+\pi^*_{\alpha})/2$ and plugging this into the right hand side of~\eqref{equa:log:4} gives
\begin{multline*}
       \E_{P_t}\E_{S_t}\left[\left(\Rhat_t(\rho_t(\pi, \alpha), \pi, \alpha) - \Rhat_t(\rho_t(\pi^*_{\alpha}, \alpha), \pi^*_{\alpha}, \alpha)\right)^2\right]
      \\
      \leq \frac{4\exp(2C\tau)}{\tau}  \Biggl(  \E_{P_t} \R_t(\pi,\alpha) - \E_{P_t} \R_t(\pi^*_{\alpha},\alpha) \Biggr).
\end{multline*}

\noindent The (optimal) choice $\tau = \frac{1}{2C}$ gives the desired bound
\begin{multline*}
       \E_{P_t}\E_{S_t}\left[\left(\Rhat_t(\rho_t(\pi, \alpha), \pi, \alpha) - \Rhat_t(\rho_t(\pi^*_{\alpha}, \alpha), \pi^*_{\alpha}, \alpha)\right)^2\right]
      \\
      \leq \frac{4\exp(2C\tau)}{\tau}  \E_{P_t} \E_{S_t} \Biggl[   \Rhat_t(\pi,\alpha) - \E_{P_t} \Rhat_t(\pi^*_{\alpha},\alpha) \Biggr].
\end{multline*}
\end{proof}

\subsection{PAC-Bayes Bound for Meta-learning}

We will now seek for a prior $\pi$ which allows to obtain a small meta-risk
\begin{equation*}
    \E_{\pi\sim \Pi}[\mathcal{E}(\pi)].
\end{equation*}
In order to do so, we will fix a set of possible priors $\mathcal{M}$ and a set of distributions $\G$ on these priors: $\G\subseteq \mathcal{P}(\mathcal{M})$\footnote{Note that measurability issues can arise when the set $\mathcal{F}$ is non parametric. However, in all our examples, the set $\mathcal{F}$ is parametric.}. Given a probability distribution $\Lambda\in \mathcal{G}$ called ``prior on priors'', we define Gibbs distribution on priors similarly as in \eqref{gibbs_posterior}, but at the meta-level:
\begin{equation}
    \hat{\Pi}= \argmin_{\Pi\in\G} \left\{ \frac{1}{T}\sum_{t=1}^T \E_{\pi\sim\Pi} \left[\Rhat_t\left(\rho_t(\pi, \alpha), \pi, \alpha\right)\right] + \frac{\KL(\Pi\Vert\Lambda)}{\beta T} \right\}, \label{prior_prior_choice}
\end{equation}
where $\beta>0$ is some parameter.
As a consequence of Theorem~\ref{bernstein_hypothesis_theorem} comes the next result, whose proof is given in Appendix~\ref{theorem_meta_learning_proof}.
\begin{theorem}\label{theorem_meta_learning}
Assume that the loss $\ell$ satisfies \eqref{bounded_assumption}. Defining $\beta = \frac{1}{C+c}$, it holds, for any $\F\subseteq \mathcal{P}(\Theta)$,
\begin{multline*}
    \E_{P_1, \dots, P_T} \E_{S_1, \dots, S_T}\E_{\pi\sim\hat{\Pi}}\left[\mathcal{E}(\pi)\right] - \mathcal{E}^* \leq  \frac{2}{1 - \frac{\alpha c \1_B}{2(1-C\alpha)}} \inf_{\Pi\in \G}\E_{P_{T+1}} \Biggl[ \\
    \E_{\pi\sim \Pi}\left[\inf_{\rho\in\F} \left\{\E_{\theta\sim\rho}[R_{T+1}(\theta) -R_{T+1}(\theta_t^*)] + \frac{{\rm KL}(\rho\Vert\pi)}{\alpha n}\right\}\right] + \frac{\KL(\Pi\Vert \Lambda)}{\beta T}+ \frac{\alpha C^2 (1-\1_B)}{8}\Biggr].
\end{multline*}
\end{theorem}

\noindent While the bound given in Theorem~\ref{theorem_meta_learning} depends on some subset $\F\subseteq \mathcal{P}(\Theta)$ chosen for computational reasons, the bound is on the excess risk of $\hat{\Pi}$, which itself is based on the exact Gibbs posterior $\rho_t(\pi, \alpha)$, and not on its variational approximation $\hat{\Pi}(\F)$, based on $\rho_t(\pi, \alpha, \F)$ and defined as
\begin{equation*}
    \hat{\Pi}(\F) = \argmin_{\Pi\in\G} \left\{ \frac{1}{T}\sum_{t=1}^T \E_{\pi\sim\Pi} \left[\Rhat_t\left(\rho_t(\pi, \alpha,\F), \pi, \alpha\right)\right] + \frac{\KL(\Pi\Vert\Lambda)}{\beta T} \right\}.
\end{equation*}

\noindent In some settings, $\hat{\Pi}(\F)$ is tractable and its Gibbs-based counterpart $\hat{\Pi}$ is not, and it is a fundamental open question to determine under what condition on $\F$ we can replace $\hat{\Pi}$ by $\hat{\Pi}(\F)$ in Theorem~\ref{theorem_meta_learning}.
\begin{open_question}
Under what conditions on $\F$ can we replace $\hat{\Pi}$ by $\hat{\Pi}(\F)$ in the left-hand side of Theorem~\ref{theorem_meta_learning}?
\end{open_question}


\subsection{A Toy Application of Theorem~\ref{theorem_meta_learning}: Concurrent Priors} \label{subsec:concurrent}

This subsection gives a toy application of Theorem~\ref{theorem_meta_learning} just to fix ideas. Here, we study the case where $M$ statisticians propose a different prior, all of which are assumed to satisfy a prior mass condition as in Corollary~\ref{cor:theorem_bound_isolation}. We denote by $\mathcal{M}=\{\pi_1,\dots,\pi_M\}$ the set of priors. We choose $\Lambda$ as the uniform distribution on $\mathcal{M}$ and $\mathcal{G}=\mathcal{P}(\mathcal{M})$. Here again, for the sake of simplicity, we assume that Bernstein's condition (see Assumption~\ref{bernstein_hypothesis}) is satisfied at the task level.

A direct application of Theorem~\ref{theorem_meta_learning} and Corollary~\ref{cor:theorem_bound_isolation} gives
$$
    \E_{P_1, \dots, P_T} \E_{S_1, \dots, S_T}\E_{\pi\sim\hat{\Pi}}\left[\mathcal{E}(\pi)\right] - \mathcal{E}^*
    \leq 4 \min_{\pi\in\mathcal{M}}  \E_{P_{T+1}} \frac{d_{\pi,T+1} \log \frac{n\alpha}{d_{\pi,T+1}} + \log \kappa_{\pi,T+1} }{\alpha n} + \frac{4\log M}{\beta T} .
$$
In other words, we obtain the rate of convergence provided by the best prior among $\{\pi_1,\dots,\pi_M\}$, at the price of an additional $\log(M)/T$ term.

\section{Applications of Theorem \ref{theorem_meta_learning}}
\label{sec:appli}

In this section, by an application of Theorem~\ref{theorem_meta_learning}, we derive explicit bounds on the excess risk of the Gibbs algorithm in the case of discrete priors (the parameter set $\Theta$ is finite; Subsection~\ref{section_application_finite_case}), Gaussian priors (Subsection~\ref{section_application_gaussian_case}) and mixtures of Gaussian priors (Subsection~\ref{section_application_mixtures_case}).

\subsection{Learning Discrete Priors}\label{section_application_finite_case}

In this subsection, we assume that $|\Theta| = M <\infty$. Following \citet{dimitri}, we define $A^*$ as the smallest possible subset of $\Theta$ such that
\begin{equation}
   \forall P\sim \mathcal{P}, \ \theta^* := \argmin_{\theta} R_{P}(\theta) \in A^*, \label{dimitri_pierre_hypothesis}
\end{equation}
and we denote $m^* := |A^*|$. In general, $A^* = \Theta$ and $m^*=M$. However, in some favorable situations, $A^* \neq \Theta$ and $m^* \ll M$, in which case, the meta-learning may improve upon the learning in isolation.
In the setting considered, Bernstein's condition is trivially satisfied and the excess risk of the Gibbs algorithm is $\frac{4\log(M)}{\alpha n}$.

We define our set of priors $\mathcal{M}$ as the set of probability distributions $\pi_A$ which are uniform on $A\subseteq \Theta$ and parameterized by $A$:
\begin{equation*}
    \mathcal{M} = \left\{\pi_A | A\subseteq \Theta\right\},
\end{equation*}
and $\G$ is the set of all distributions on $\mathcal{M}$. Our ``prior on priors'' $\Lambda$ is then defined as follows: we draw $m\in\{1, \dots, M\}$ with probability $\frac{2^{M-m}}{2^M - 1} \approx 2^{-m}$ (for $M\gg 1$), then given $m$, draw a subset $A\subseteq \Theta$ of cardinality $m$ uniformly at random, and take $\pi_A$. In other words, $\Lambda$ is a distribution defined on $\F$ such that $P_{\pi \sim \Lambda}(\pi = \pi_A) = \frac{2^{M-m}}{2^M -1} \times \frac{1}{{M\choose m}}$. 
\begin{proposition}
The excess risk of the meta predictor $\hat{\Pi}$ defined in \eqref{prior_prior_choice} is bounded as follows:
\begin{equation*}
    \E_{P_{1},\dots,P_{T}} \E_{\mathcal{S}_1,\dots,\mathcal{S}_T}  \E_{\pi\sim\hat{\Pi}} [\mathcal{E}(\pi)] \leq \mathcal{E}^* + \frac{ 4\log m^*}{\alpha n} + \frac{4m^*\log\frac{2{\rm e}M}{m^*} }{\beta T}.
\end{equation*}
\end{proposition}

\begin{remark}
Let us now compare the meta-learning rate above to the $\frac{4\log M}{\alpha n}$ rate achieved by the learning in isolation. In the unfavorable case $m^* \sim M$, the meta-learning bound is sensibly larger than the learning in isolation one, by a term $O(M/T)$ which vanishes rapidly when $T\to+\infty$.

In the favorable case $m^*\ll M$ however, the meta-learning considerably improves upon the learning in isolation for large values of $T$. In the extreme case where $m^*=1$, we have
\begin{equation*}
    \E_{P_{1},\dots,P_{T}} \E_{\mathcal{S}_1,\dots,\mathcal{S}_T}  \E_{\pi\sim\hat{\Pi}} [\mathcal{E}(\pi)] \leq \mathcal{E}^* + \frac{4\log(2{\rm e}M) }{\beta T}.
\end{equation*}
Thus, the benefits of the meta-learning are mainly expected in the $T \gg n$ regime. They are huge when the tasks are very similar, and close to zero when the tasks are totally unrelated. This is in line with the results of~\cite{dimitri} in the online setting.
%
%
\end{remark}

\begin{proof}
We first consider learning in isolation. The classical choice is to take $\pi$ as uniform in each task. An application of Theorem~\ref{theorem_bound_isolation} gives
\begin{align*}
    \E_{\mathcal{S}_t}\left[ \E_{\theta\sim\rho_t(\pi,\alpha)}[R_{P_t}(\theta)] \right] - \mathcal{E}^* 
     &\leq 2 \inf_{\rho \in\mathcal{P}(\Theta)} \left\{ \E_{\theta\sim\rho}\left[R_{P_t}(\theta)\right] - R_{P_t}^* + \frac{ \KL(\rho\Vert\pi)}{\alpha n}\right\} \\
    \\ &\leq 2 \inf_{\rho = \delta_{\vartheta}} \left\{ \E_{\theta\sim\rho}\left[R_{P_t}(\theta)\right] - R_{P_t}^* + \frac{ \KL(\rho\Vert\pi)}{\alpha n}\right\} \\
    &= \frac{ 2\log M}{\alpha n}.
\end{align*}
 
\noindent In the meta-learning case, an application of Theorem~\ref{theorem_meta_learning} gives
\begin{multline*}
    \E_{P_{1}, \dots, P_{T}} \E_{\mathcal{S}_1, \dots, \mathcal{S}_T}  \E_{\pi\sim\hat{\Pi}} [\mathcal{E}(\pi)] - \mathcal{E}^* \leq 4\inf_{1\leq m\leq M} \inf_{|A|=m} \E_{P_{T+1}\sim \mathcal{P}} \bigg[ \inf_{\theta\in A} \left\{R_{P_{T+1}}(\theta) - R_{P_{t+1}}^*\right\} \\
    + \frac{\log(m)}{\alpha n} + \frac{m\log 2 + \log {M\choose m}}{\beta T} \bigg].
\end{multline*}

\noindent Under the assumption \eqref{dimitri_pierre_hypothesis}, it holds that
\begin{equation*}
    \E_{P_{1},\dots,P_{T}} \E_{\mathcal{S}_1,\dots,\mathcal{S}_T}  \E_{\pi\sim\hat{\Pi}} [\mathcal{E}(\pi)] \leq \mathcal{E}^* + \frac{ 4 \log(m^*)}{\alpha n} + \frac{4m^*\log(2) + 4\log {M\choose m^*}}{\beta T}.
\end{equation*}

\noindent We conclude by using the classic bound $\log {M\choose m} \leq m\log\frac{M{\rm e}}{m}$.
\end{proof}

\subsection{Learning Gaussian priors}\label{section_application_gaussian_case}

In this subsection, we consider the set of all Gaussian distributions 
\begin{equation}
    \mathcal{M} = \left\{ p_{\mu,\sigma^2} =\bigotimes_{i=1}^d \N(\mu_i,\sigma_i^2), \mu=(\mu_1,\dots,\mu_d)\in\mathbb{R}^d, \sigma^2=(\sigma_1^2,\dots,\sigma_d^2)\in(\mathbb{R}_{+}^*)^d \right\}. \label{m_gaussian}
\end{equation}
So, priors on priors are actually defined as priors on $\mu$ and $\sigma^2$ given by:
\begin{equation}
    \G = \Bigg\{q_{\tau, \xi^2, a, b} = \bigotimes_{\substack{k\in [K] \\ i\in [d]}} \N(\tau_{k, i}, \xi_k^2) \otimes \bigotimes_{k=1}^K \Gamma(a_k, b_k)\Bigg\}, \label{g_gaussian}
\end{equation}
and we choose the prior on priors as $\Lambda = q_{0, \bar{\bar{\xi}}^2, \bar{\bar{a}}, \bar{\bar{b}}}$.

From now on and until the end of this section, we assume that both \eqref{application_assumption} and Assumption~\ref{bernstein_hypothesis} (Bernstein's condition) hold, and we are looking for a prior on priors $\Pi$ from which to sample $\pi$, such that $\rho_{T+1}(\pi, \alpha)$ concentrates as much as possible to the best parameter.  Denoting $\mu^* := \E_{P_{T+1}\sim \mathcal{P}}[\mu_{P_{T+1}}]$ and $\Sigma(\mathcal{P}) := \E_{P_{T+1}\sim \mathcal{P}}\left[\Vert\mu_{P_{T+1}} - \mu^*\Vert^2\right]$, the following holds.
\begin{proposition}
Under Assumptions \ref{bernstein_hypothesis} and \ref{application_assumption}, the excess risk of $\hat{\Pi}$ defined in \eqref{prior_prior_choice} for $\mathcal{M}$ and $\G$ defined in \eqref{m_gaussian} and \eqref{g_gaussian} is bounded as follows:
\begin{equation*}
    \E_{P_{1},\dots,P_{T}} \E_{\mathcal{S}_1,\dots,\mathcal{S}_T}  \E_{\pi\sim\hat{\Pi}} [\mathcal{E}(\pi)] - \mathcal{E}^* \leq \left\{
    \begin{array}{ll}
        G(C,\mu^*,\bar{\bar{\xi}},L,\bar{\bar{a}},\bar{\bar{b}}) \frac{d + \log(T)}{T} &\mbox{ if } \ \Sigma(\mathcal{P}) \leq \frac{n}{T}; \\
        F(C,\mu^*,\bar{\bar{\xi}},L,\bar{\bar{a}},\bar{\bar{b}})\left(\frac{d\log n + \Sigma(\mathcal{P})}{n} + \frac{d}{T}\right) &\mbox{ otherwise,}
    \end{array}
    \right.
\end{equation*}
where $F$ and $G$ are functions of the problem parameters.
\end{proposition}

\noindent Interestingly enough, in the favorable case $\Sigma(\mathcal{P}) \leq \frac{n}{T}$, a proper choice of the prior of priors lead to the very fast rate of convergence $O\left(\frac{\log T}{T}\right)$, which considerably improves upon the fast rate $O\left(\frac{1}{n} + \frac{1}{T}\right)$ when $n\ll T$. The detailed proof of this proposition, as well as explicit expressions of $F$ and $G$, are given in Appendix~\ref{appendix_application_gaussian}.

\subsection{Learning Mixtures of Gaussian priors}\label{section_application_mixtures_case}

In this subsection, we generalize the result of the previous section to priors that are mixtures of Gaussians. We still assume that Assumption~\ref{application_assumption} and Assumption~\ref{bernstein_hypothesis} (Bernstein's condition) hold. We first assume the number of components in the mixture is known. Under these hypotheses, the set of possible priors $\pi$ is
\begin{equation}
    \mathcal{M} = \left\{p_{w, \mu, \sigma^2} = \sum_{k=1}^K w_k \bigotimes_{i=1}^d \N(\mu_{k, i}, \sigma^2_{k, i}): \forall k\in [K], w_k\geq 0, 1^\top w = 1\right\}. \label{m_mixtures}
\end{equation}

\noindent We add a Dirichlet prior on the weights $w$ of the components in the mixture, and the set of priors on priors becomes
\begin{equation}
    \G = \Bigg\{q_{\delta, \tau, \xi^2, b} = \Dir(\delta)\otimes \bigotimes_{\substack{k\in [K] \\ i\in [d]}} \N(\tau_{k, i}, \xi_k^2) \otimes \bigotimes_{k=1}^K \Gamma(2, b_k): \delta = (\delta_1, \dots, \delta_K)\in \mathbb{R}^K\Bigg\}, \label{g_mixtures}
\end{equation}
while the initial prior on priors is chosen as $\Lambda = q_{\1_K, 0, \bar{\bar{\xi}}^2, \bar{\bar{b}}}$. 
We define
\begin{equation*}
    \Sigma_K (\mathcal{P}) := \inf_{\tau_1, \dots, \tau_K}\E_{P_{T+1}\sim \mathcal{P}}\left[\min_{k\in [K]} \Vert\mu_{P_{T+1}} - \tau_{k}\Vert^2 \right].
\end{equation*}

\begin{proposition}\label{proposition_mixtures_known}
Under Assumptions \ref{bernstein_hypothesis} and \ref{application_assumption}, the excess risk of $\hat{\Pi}$ defined in \eqref{prior_prior_choice} for $\mathcal{M}$ and $\G$ defined in \eqref{m_mixtures} and \eqref{g_mixtures} is bounded as follows:
\begin{multline*}
    \E_{P_1, \dots, P_T}\E_{S_1, \dots, S_T}\E_{\pi\sim \hat{\Pi}}[\mathcal{E}(\pi)] - \mathcal{E}^* \leq \mathrm{CV}_{\mathrm{finite}}(K, n) + K\times \mathrm{CV}_{\mathrm{Gaussian}}\left(d, \Sigma_K(\mathcal{P}), n, T\right) \\
    + \mathrm{CV}_{\mathrm{meta}}(T, n, d, K, \Bar{\Bar{b}}, \Bar{\Bar{\xi}}^2, \tau),
\end{multline*}
\begin{equation*}
    \text{where } \ \mathrm{CV}_{\mathrm{finite}}(K, n) = \frac{4 \log(2K)}{\alpha n}; \quad \mathrm{CV}_{\mathrm{meta}}(T, n, d, K, \Bar{\Bar{b}}, \Bar{\Bar{\xi}}^2, \tau) = O\left(\frac{dK \log T}{T}\right);
\end{equation*}
\begin{equation*}
    \mathrm{CV}_{\mathrm{Gaussian}}\left(d, \Sigma_K(\mathcal{P}), n, T\right) = \left\{
    \begin{array}{ll}
        \frac{8Ld}{T} + \frac{4}{\alpha T} \quad &\mbox{if } \ \Sigma_K(\mathcal{P}) \leq \frac{n}{T^2}; \\
        \frac{2d}{\alpha n} \log \left(1 + 4\alpha L n\right) + \frac{4\Sigma_K(\mathcal{P})}{\alpha n} \quad &\mbox{otherwise.}
    \end{array}
    \right.
\end{equation*}
\end{proposition}

\noindent Let us analyze each of the terms of the above bound. The first term $\text{CV}_{\text{finite}}(K, n)$ is the bound we had in the finite case of Subsection~\ref{section_application_finite_case}. Visualizing our $K$ mixtures as the $K$ points in the finite case, this term is the time required by our estimator to select the right mixture. While this term makes the convergence rate of $O\left(\frac{1}{n}+\frac{1}{T}\right)$ notably worse than the $O\left(\frac{1}{T}\right)$ we might hope for, it is essentially unavoidable as it appears in the much simpler model of a finite set of $K$ parameters described in Subsection~\ref{section_application_finite_case}.

The next term $\text{CV}_{\text{Gaussian}}\left(d, \Sigma_K(\mathcal{P}), n, T\right)$ is (similar to) the bound obtained in the Gaussian case of Subsection~\ref{section_application_gaussian_case}, with the exception that $\Sigma(\mathcal{P})$ is replaced by $\Sigma_K (\mathcal{P})$, and scales with the convergence time to the best Gaussian for every task $t$. 

Eventually, the last term $\text{CV}_{\text{meta}}(T, n, d, K, \Bar{\Bar{b}}, \Bar{\Bar{\xi}}^2, \tau)$ is the convergence term at the meta level and is a $O\left(\frac{dK\log T}{T}\right)$. This is the cost of the meta learning compared to the learning in isolation. When $T\gg n$, this term is very small and justifies the use of the meta learning.

\begin{remark}
One may think, looking at the bounds in the two previous cases considered, that the $O\left(\frac{1}{n} + \frac{1}{T}\right)$ convergence rate in the case of mixtures of Gaussians is slower than the one for Gaussians which can be as fast as $O\left(\frac{1}{T}\right)$. In reality, the rate of convergence is (naturally) faster for the model of mixtures of Gaussians, because in the case of mixtures of Gaussians, the convergence term $\text{CV}_{\text{Gaussian}}\left(d, \Sigma_K(\mathcal{P}), n, T\right)$ is $O\left(\frac{1}{T}\right)$ under the assumption that $\Sigma_K(\mathcal{P})\leq \frac{n}{T}$, while in the Gaussian case, the much stronger assumption $\Sigma(\mathcal{P}) = \Sigma_1(\mathcal{P})\leq \frac{n}{T}$ is required. Under this assumption, the similar rate $O\left(\frac{1}{T}\right)$ is naturally achieved.
\end{remark}

\noindent We now consider the case when the number of mixtures $K$ is unknown. The set of priors hence becomes the set of all (finite) mixtures of Gaussians:
\begin{equation}
    \mathcal{M} = \left\{p_{w, \mu, \sigma^2} = \sum_{k=1}^{+\infty} w_k \bigotimes_{i=1}^d \N(\mu_{k, i}, \sigma^2_{k, i}): \exists K\geq 1: \forall k\geq K+1, w_k = 0\right\}. \label{m_mixtures_unknown}
\end{equation}

\noindent In the definition of the set of priors on priors $\G$, we assume that $K\leq T$ (otherwise, there is a high chance of overfitting). We then set a Dirichlet prior on the number of components $K$, and given $K$, set the same model as before. Formally,
\begin{equation}
    \G = \Bigg\{q_{x, \delta, \tau, \xi^2, b} = q_{x} \times q_{\delta, \tau, \xi^2, b | K} \Bigg\}, \label{g_mixtures_unknown}
\end{equation}
and we set the prior on priors $\Lambda = q_{\frac{1}{T}\1_T, \1_K, 0, \bar{\bar{\xi}}^2, \bar{\bar{b}}}$.
The application of Theorem~\ref{theorem_meta_learning} yields the bound

\begin{proposition}
Under the same conditions and using the same notations as Proposition~\ref{proposition_mixtures_known}, the excess risk of $\hat{\Pi}$ defined in \eqref{prior_prior_choice} for $\mathcal{M}$ and $\G$ defined in \eqref{m_mixtures_unknown} and \eqref{g_mixtures_unknown} is bounded as follows:
\begin{multline*}
    \E_{P_1, \dots, P_T}\E_{S_1, \dots, S_T}\E_{\pi\sim \hat{\Pi}}[\mathcal{E}(\pi)] - \mathcal{E}^* \leq \\
    \inf_{K\in [T]} \Bigg\{\mathrm{CV}_{\mathrm{finite}}(K, n) + K\times \mathrm{CV}_{\mathrm{Gaussian}}\left(d, \Sigma_K(\mathcal{P}), n, T\right) + \mathrm{CV}_{\mathrm{meta}}^{\mathrm{unknown}}(T, n, d, K, \Bar{\Bar{b}}, \Bar{\Bar{\xi}}^2, \tau)\Bigg\},
\end{multline*}
where the convergence term at the meta level becomes
\begin{equation*}
    \mathrm{CV}_{\mathrm{meta}}^{\mathrm{unknown}}(T, n, d, K, \Bar{\Bar{b}}, \Bar{\Bar{\xi}}^2, \tau) = \mathrm{CV}_{\mathrm{meta}}(T, n, d, K, \Bar{\Bar{b}}, \Bar{\Bar{\xi}}^2, \tau) + \frac{2\log T}{\beta T}.
\end{equation*}
\end{proposition}

\noindent Our estimator takes $\frac{2\log T}{\beta T}$ to find the optimal number of mixtures at the meta level. This is the price to pay to have the infimum on $K$ in the bound. In the limit $T\gg n$ and when no prior information on $K$ is available, this clearly improves upon the bound of Proposition~\ref{proposition_mixtures_known}, and hence justifies setting a prior on $K$ at the meta level rather than choosing an isolated $K$, pleading again in the favor of meta-learning. The proof of all the results of this subsection is given in Appendix~\ref{appendix_application_mixtures}.

\section{Discussion}
\label{sec:dicussion}

In recent years, the statistical guarantees of meta-learning have received increasing attention. In the following paragraphs, we present a short review of the literature on the statistical theory of meta-learning, followed by a brief discussion of three papers that are closely related to our analysis.

\vspace{0.2cm}
\noindent \textbf{Theoretical bounds in meta-learning.}
The first theoretical analysis of meta-learning goes back to \cite{baxter2000model}, who introduced the notion of task environment and derived a uniform generalization bound based on the capacity and covering number of the model. Following this i.i.d.~task environment setting, many other generalization bounds have since been provided for different strategies and proof techniques, including VC theory \citep{baxter2000model,ben2003exploiting,Maurer2009,Maurer2016,GuanTask2022}, algorithmic stability \citep{Maurer2005,chen2020closer,al2021data,GuanBernstein2022} and information theory \citep{jose2021information,chen2021generalization,jose2021transfer,rezazadeh2021conditional,hellstrom2022evaluated}.
A related approach for deriving such bounds is based on PAC-Bayes theory. First proposed in the meta-learning framework in the pioneering paper of \cite{pen2014}, this idea of learning a hyper-posterior that generates a prior for the new task has been taken up several times in the recent years \citep{ami2018,DingPAC2021,liu2021pac,rothfuss2021,FaridPAC2021,rothfuss2022,GuanPAC2022,Rezazadeh2022}.
In particular, \cite{ami2018} derived a new PAC-Bayes bound, which they applied to the optimization of deep neural networks, albeit with computational limitations. This latter concern was partially addressed by \cite{rothfuss2021}, who also specified the hyper-posterior and extended the results to unbounded losses, and further investigated their study in \cite{rothfuss2022}.
Some papers combined ideas from different literatures, such as \cite{FaridPAC2021}, who explored the link between PAC-Bayes and uniform stability in meta-learning, and provided a precise analysis of stability and generalization. 
Excess risk bounds have also been provided in the i.i.d.~task environment framework, see \cite{Maurer2016,Denevi2018,denevi2018incremental,Denevi2019,denevi2019online,denevi2020advantage,balcan2019provable,bai2021important,chen2022bayesian}. 
The task environment assumption has recently been challenged, for example by \cite{du2020few} and \cite{tripuraneni2021provable}, who proposed to use assumptions on the distributional similarity between the features and the diversity of tasks to control the excess risk, an idea further explored by \cite{fallah2021generalization} who exploited a notion of diversity between the new task and training tasks using the total variation distance. Finally, a detailed analysis of regret bounds in lifelong learning has been carried out in recent years \citep{pmlr-v54-alquier17a,denevi2018incremental,denevi2019online,balcan2019provable,khodak2019adaptive,finn2019online,dimitri}.

\vspace{0.2cm}
\noindent \textbf{Comparison to \cite{Denevi2019}.} \cite{Denevi2019} is probably the study that is the most related to our paper. The authors  provide statistical guarantees for Ridge regression with a meta-learned bias, and focus on the usefulness of their strategy relative to single-task learning, proving that their method outperforms the standard $\ell_2$-regularized empirical risk minimizer. In particular, they can achieve an excess risk rate of order $O\left({1}/{\sqrt{T}}\right)$ in the favorable case $\Sigma(\mathcal{P}) \leq \frac{n}{T}$, where $\Sigma(\mathcal{P})$ is a variance term similar to the one we defined in our Gaussian example.

\vspace{0.2cm}
\noindent \textbf{Comparison to \cite{GuanBernstein2022}.} To the best of our knowledge, \cite{GuanBernstein2022} is the only work in the meta-learning literature that addresses fast rates with respect to the number of tasks $T$ under the task environment assumption. However, we actually show in our paper that there is no need to extend Bernstein's condition when using exact Bayesian inference and that the final posterior naturally satisfies the extended Bernstein assumption, thus giving fast rates with respect to $T$, while \cite{GuanBernstein2022} require an additional Polyak-Łojasiewicz condition to achieve fast rates. Furthermore, their analysis is very different in nature, relying on stability arguments to derive generalization bounds, while we use PAC-Bayes theory to control the excess risk.

\vspace{0.2cm}
\noindent \textbf{Comparison to \cite{GuanPAC2022} and \cite{Rezazadeh2022}.} Finally, \cite{GuanPAC2022} and \cite{Rezazadeh2022} provide fast rate generalization bounds based on Catoni's PAC-Bayes inequality. Nevertheless, their notion of fast rates differs substantially from ours: our rates quantify the rate of convergence of the excess risk, in line with the classic statistical learning literature. In contrast, the fast rates of \cite{GuanPAC2022} are only available for a variant of the generalization gap that compares the theoretical risk to a factor of the empirical risk. In order to obtain a sharp factor equal to $1$ in front of the empirical risk, one needs to fine-tune a parameter in \cite{GuanPAC2022}'s inequality, which would then return a slow rate.

\section{Conclusion and open problems}

We provided an analysis of the excess risk in meta-learning the prior via PAC-Bayes bounds. Surprisingly, at the meta-level, conditions for fast rates are always satisfied if one uses exact Gibbs posteriors at the task level. An important problem is to extend this result to variational approximations of Gibbs posteriors.

\appendix

\section{Some Useful Formulas}

The following (known) results are used throughout the text. They are recalled here without proof.

\subsection{Concentration Inequalities}

For Hoeffding and Bernstein see~\cite{bou2013}. For Donsker and Varadhan, see for example~\cite{cat2007}.
\begin{lemma}[Hoeffding's inequality] \label{lemma:hoeffding}
Let $U_1,\dots,U_n$ be i.i.d random variables taking values in an interval $[a,b]$. Then, for any $s>0$,
\begin{equation*}
    \E \left[ {\rm e}^{s \sum_{i=1}^n [ U_i - \E(U_i)]} \right] \leq {\rm e}^{\frac{n s^2 (b-a)^2}{8}}.
\end{equation*}
\end{lemma}

\begin{lemma}[Bernstein's inequality] \label{lemma:bernstein}
Let $U_1,\dots,U_n$ be i.i.d random variables such that for any $k\geq 2$,
\begin{equation}
\label{subexp}
    \E\left[|U_i|^k\right] \leq \frac{k!}{2} V C^{k-2}.
\end{equation}
Then, for any $s\in(0,1/C]$,
\begin{equation*}
    \E \left[ {\rm e}^{s \sum_{i=1}^n [ U_i - \E(U_i)]} \right] \leq {\rm e}^{\frac{n s^2 V}{2(1-sC)}}.
\end{equation*}
\end{lemma}
Note that in particular, if $|U_i|\leq C$ almost surely,~\ref{subexp} always holds with $V=\mathbb{E}(U_i^2)$.

\subsection{Donsker and Varadhan's Lemma}

\begin{lemma}[Donsker and Varadhan's variational inequality~\cite{don1975}] \label{lemma:dv}

\noindent Let $\mu$ be a probability measure on $\Theta$. For any measurable, bounded function $h:\Theta\rightarrow\mathbb{R}$, we have:
\begin{equation*}
    \log \E_{\theta\sim\mu}\left[{\rm e}^{h(\theta)} \right] =\sup_{\rho\in\mathcal{P}(\Theta)}\Bigl\{\E_{\theta\sim\rho}[h(\theta)] - \KL (\rho\Vert\mu)\Bigr\}.
\end{equation*}
Moreover, the supremum with respect to $\rho$ in the right-hand side is
reached for the Gibbs measure
$\mu_{h}$ defined by its density with respect to $\mu$
\begin{equation*}
    \frac{{\rm d}\mu_{h}}{{\rm d}\mu}(\vartheta) =  \frac{{\rm e}^{h(\vartheta)}}{ \E_{\theta\sim\mu}\left[{\rm e}^{h(\theta)} \right] }.
\end{equation*}
\end{lemma}

\subsection{KL Divergence of some known Distributions}

\noindent Denoting by $H(x)$ the entropy of $(x_1, \dots, x_T)$, recall that the KL divergence between a multinomial distribution of parameters $(x_1, \dots, x_T)$ and a multinomial distribution of parameters $\left(\frac{1}{T}, \dots, \frac{1}{T}\right)$ is
\begin{equation}
    \KL\left(\Mult(x)\Vert \Mult\left(\frac{1}{T}\right)\right) = \log T - H(x). \label{kl_multinomial}
\end{equation}

\noindent Recall that the KL divergence between 2 normal distributions is
\begin{equation}
    \KL\left(\N(\mu, \sigma^2)\Vert\N(\Bar{\mu}, \Bar{\sigma}^2)\right) = \frac{1}{2} \left(\frac{(\mu-\Bar{\mu})^2}{\Bar{\sigma}^2} + \frac{\sigma^2}{\Bar{\sigma^2}} - 1 + \log \frac{\Bar{\sigma}^2}{\sigma^2}\right). \label{kl_gaussian}
\end{equation}

\noindent Recall that the KL divergence between 2 Gamma distributions is
\begin{equation}
    \KL\left(\Gamma(a, b)\Vert\Gamma\left(\Bar{\Bar{a}}, \Bar{\Bar{b}}\right)\right) = \left(a - \Bar{\Bar{a}}\right)\psi(a) + \log \frac{\Gamma\left(\Bar{\Bar{a}}\right)}{\Gamma(a)} + \Bar{\Bar{a}}\log \frac{b}{\Bar{\Bar{b}}} + a \frac{\Bar{\Bar{b}} - b}{b}. \label{kl_gamma}
\end{equation}

\noindent Recall that the KL divergence between a Dirichlet distribution of parameter $\delta$ and a Dirichlet distribution of parameter $1_K = (1, \dots, 1)$ is
\begin{equation}
    \KL\left(\Dir(\delta)\Vert \Dir(1_K)\right) = \log \frac{\Gamma(1^\top \delta)}{\Gamma(K) \times \prod_{k=1}^K \Gamma(\delta_k)} + \sum_{k=1}^K (\delta_k - 1) \left(\psi(\delta_k) - \psi(1^\top \delta)\right).
    \label{kl_dirichlet}
\end{equation}

\section{Proof of Theorem \ref{theorem_bound_isolation}}\label{theorem_bound_isolation_proof}

We mostly follow the proof technique developed in~\cite{cat2007}. For any $s>0$, fix $\theta\in\Theta$ and let $U_i=\E[\ell(Z_{t,i},\theta)]-\ell(Z_{t,i},\theta)-\E[\ell(Z_{t,i},\theta^*_t)]+\ell(Z_{t,i},\theta^*_t)$ for any $i\in \{1, \dots, n\}$. We are going to distinguish two cases, whether or not Bernstein assumption is satisfied.

\noindent If Bernstein assumption is satisfied, we apply Lemma~\ref{lemma:bernstein} to $U_i$. Note that in this case, $V$ is actually the variance term $V_t(\theta,\theta_t^*) $. So, for any $s>0$,
\begin{equation*}
    \E_{\mathcal{S}_t} \left[ {\rm e}^{ s n \left(R_{P_t}(\theta)-\hat{R}_t(\theta) -R_{P_t}^* + \hat{R}_t(\theta^*_t)\right) } \right] \leq {\rm e}^{\frac{ n s^2 V(\theta,\theta^*_t)}{2(1-sC)}}.
\end{equation*}

\noindent We let $s=\lambda/n$, which gives
\begin{equation*}
    \E_{\mathcal{S}_t} \left[ {\rm e}^{ \lambda \left(R_{P_t}(\theta)-\hat{R}_t(\theta) -R_{P_t}^* + \hat{R}_t(\theta^*_t)\right) } \right]  \leq {\rm e}^{\frac{\lambda^2 V(\theta,\theta^*_t)}{2(n-C\lambda) }}.
\end{equation*}

\noindent Making use of Bernstein hypothesis gives
\begin{equation*}
    \E_{\mathcal{S}_t} \left[ {\rm e}^{ \lambda \left(R_{P_t}(\theta)-\hat{R}_t(\theta) -R_{P_t}^* + \hat{R}_t(\theta^*_t)\right) } \right]  \leq {\rm e}^{\frac{\lambda^2 c \left(R_{P_t}(\theta) - R_{P_t}^*\right) }{2(n-C\lambda) }}.
\end{equation*}

\noindent If Bernstein hypothesis is not satisfied, we apply Lemma \ref{lemma:hoeffding} to $U_i$, which gives, for any $s>0$,
\begin{equation*}
    \E_{\mathcal{S}_t} \left[ {\rm e}^{ s n \left(R_{P_t}(\theta)-\hat{R}_t(\theta) -R_{P_t}^* + \hat{R}_t(\theta^*_t)\right) } \right] \leq {\rm e}^{\frac{ n s^2 C^2}{8}}.
\end{equation*}

\noindent Letting $s = \lambda /n$ gives 
\begin{equation*}
    \E_{\mathcal{S}_t} \left[ {\rm e}^{ \lambda \left(R_{P_t}(\theta)-\hat{R}_t(\theta) -R_{P_t}^* + \hat{R}_t(\theta^*_t)\right) } \right] \leq {\rm e}^{\frac{\lambda^2 C^2}{8n}}.
\end{equation*}

\noindent Defining the general bound
\begin{equation}
    W = \frac{\lambda^2 c \left(R_{P_t}(\theta) - R_{P_t}^*\right) }{2(n-C\lambda) } \1_{B} + \frac{\lambda^2 C^2}{8n} (1 - \1_{B}), \label{w_definition_bound_isolation}
\end{equation}
where $\1_{B}$ is equal to $1$ if Bernstein assumption is satisfied, and $0$ otherwise, it holds in either case that
\begin{equation*}
    \E_{\mathcal{S}_t} \left[ {\rm e}^{ \lambda \left(R_{P_t}(\theta)-\hat{R}_t(\theta) -R_{P_t}^* + \hat{R}_t(\theta^*_t)\right) } \right] \leq {\rm e}^{W}. 
\end{equation*}

\noindent Rearranging the terms gives
\begin{equation*}
    \E_{\mathcal{S}_t} \left[ {\rm e}^{ \lambda \left(R_{P_t}(\theta)-R_{P_t}^* - \frac{W}{\lambda} - \hat{R}_t(\theta) + \hat{R}_t(\theta^*_t)\right) } \right]  \leq 1.
\end{equation*}

\noindent Next, integrating this bound with respect to $\pi$ and using Fubini's theorem to exchange both integrals gives
\begin{equation*}
    \E_{\mathcal{S}_t} \E_{\theta\sim\pi}\left[ {\rm e}^{ \lambda \left(R_{P_t}(\theta)-R_{P_t}^* - \frac{W}{\lambda} - \hat{R}_t(\theta) + \hat{R}_t(\theta^*_t)\right) } \right]  \leq 1.
\end{equation*}

\noindent We then apply Lemma~\ref{lemma:dv} to the argument of the expectation with respect to the sample, and we have
\begin{equation*}
    \E_{\mathcal{S}_t} \left[ {\rm e}^{  \sup_{\rho\in\mathcal{P}(\Theta)} \left\{\lambda \E_{\theta\sim\rho} \left[R_{P_t}(\theta)-R_{P_t}^* - \frac{W}{\lambda} - \hat{R}_t(\theta) + \hat{R}_t(\theta^*_t)\right] - \KL(\rho\Vert\pi) \right\}} \right]  \leq 1.
\end{equation*} 

\noindent Jensen's inequality implies
\begin{equation*}
    {\rm e}^{ \lambda \E_{\mathcal{S}_t} \left[\sup_{\rho\in\mathcal{P}(\Theta)} \left\{ \E_{\theta\sim\rho} \left[R_{P_t}(\theta)-R_{P_t}^* - \frac{W}{\lambda} - \hat{R}_t(\theta) + \hat{R}_t(\theta^*_t)\right] - \frac{\KL(\rho\Vert\pi)}{\lambda}\right\}\right]} \leq 1,
\end{equation*}
in other words,
\begin{equation*}
    \E_{\mathcal{S}_t} \left[\sup_{\rho\in\mathcal{P}(\Theta)} \left\{ \E_{\theta\sim\rho} \left[R_{P_t}(\theta)-R_{P_t}^* - \frac{W}{\lambda} - \hat{R}_t(\theta) + \hat{R}_t(\theta^*_t)\right] - \frac{\KL(\rho\Vert\pi)}{\lambda}\right\}\right] \leq 0.
\end{equation*}

\noindent At this stage, we can replace $W$ by its value given in \eqref{w_definition_bound_isolation} to obtain the bound:
\begin{multline*}
    \E_{\mathcal{S}_t} \Bigg[\sup_{\rho\in\mathcal{P}(\Theta)} \bigg\{ \E_{\theta\sim\rho} \bigg[\left(1 - \frac{\lambda c \1_{B}}{2(n-C\lambda)}\right)\left(R_{P_t}(\theta)-R_{P_t}^*\right) \\
    - \frac{\lambda C^2(1 - \1_{B})}{8n} - \hat{R}_t(\theta) + \hat{R}_t(\theta^*_t)\bigg] - \frac{\KL(\rho\Vert\pi)}{\lambda}\bigg\}\Bigg] \leq 0.
\end{multline*}

\noindent Next, we rearrange the terms and replace the supremum on $\rho$ by $\rho_t(\pi,\alpha)$:
\begin{multline*}
    \E_{\mathcal{S}_t} \E_{\theta\sim\rho_t(\pi,\alpha)}\left[R_{P_t}(\theta)\right]-R_{P_t}^* \leq \frac{1}{1-\frac{\lambda c \1_{B}}{2(n-C\lambda)}} \\
    \times \left(\E_{\mathcal{S}_t}\left[\E_{\theta\sim\rho_t(\pi,\alpha)}\left[\hat{R}_t(\theta)\right]  - \hat{R}_t(\theta^*_t) + \frac{\KL(\rho_t(\pi,\alpha)\Vert\pi)}{\lambda}\right] + \frac{\lambda C^2(1 - \1_{B})}{8n}\right).
\end{multline*}

\noindent We then replace $\lambda$ by $\alpha n$ and by definition of Gibbs posterior $\rho_t(\pi, \alpha)$, the above bound is the same as
\begin{multline*}
    \E_{\mathcal{S}_t} \E_{\theta\sim\rho_t(\pi,\alpha)}\left[R_{P_t}(\theta)\right]-R_{P_t}^* \leq \frac{1}{1-\frac{\alpha c \1_{B}}{2(1-C\alpha)}} \\
    \times \left(\E_{\mathcal{S}_t}\left[\inf_{\rho\in \F}\left\{\E_{\theta\sim\rho}\left[\hat{R}_t(\theta)\right]  - \hat{R}_t(\theta^*_t) + \frac{\KL(\rho\Vert\pi)}{\alpha n}\right\}\right] + \frac{\alpha C^2(1 - \1_{B})}{8}\right).
\end{multline*}

\noindent In particular, under Bernstein assumption, i.e., if $\1_B = 1$, the choice $\alpha = \frac{1}{C+c}$ gives
\begin{equation*}
  \E_{\mathcal{S}_t}\left[ \E_{\theta\sim\rho_t(\pi,\alpha)}[R_{P_t}(\theta)] \right] - R_{P_t}^*
  \leq 2 \E_{\mathcal{S}_t}\left[ \inf_{\rho\in\mathcal{F}} \left\{ \E_{\theta\sim\rho}\left[\hat{R}_t(\theta)\right]  - \hat{R}_t(\theta^*_t) + \frac{ \KL(\rho\Vert\pi)}{\alpha n} \right\}\right].
\end{equation*}

\noindent Without Bernstein assumption, i.e., if $\1_B = 0$, rewriting the bound and taking the minimum over $\alpha$ yields
\begin{equation*}
  \E_{\mathcal{S}_t}\left[ \E_{\theta\sim\rho_t(\pi,\alpha)}[R_{P_t}(\theta)] \right] - R_{P_t}^*
  \leq \inf_{\alpha}\E_{\mathcal{S}_t}\left[\inf_{\rho\in \F}\left\{\E_{\theta\sim\rho}\left[\hat{R}_t(\theta)\right]  - \hat{R}_t(\theta^*_t) + \frac{\KL(\rho\Vert\pi)}{\alpha n} + \frac{\alpha C^2}{8}\right\}\right],
\end{equation*}
and this concludes the proof.
\hfill $\blacksquare$

\section{Proof of Corollary \ref{cor:theorem_bound_isolation}}

First,
\begin{align*}
\E_{\mathcal{S}_t}\left[\inf_{\rho\in \mathcal{P}(\Theta)}\Rhat_t(\rho, \pi, \alpha) - \hat{R}_t(\theta^*_t)\right]
& =
\E_{\mathcal{S}_t}\left[\inf_{\rho\in \mathcal{P}(\Theta)} \left( \mathbb{E}_{\theta\sim\rho}[\hat{R}_t(\theta)]  + \frac{{\rm KL}(\rho\|\pi)}{\alpha n}   \right)- \hat{R}_t(\theta^*_t)\right]
\\
& \leq \inf_{\rho\in \mathcal{P}(\Theta)} \E_{\mathcal{S}_t} \left[ \mathbb{E}_{\theta\sim\rho}[\hat{R}_t(\theta)]  + \frac{{\rm KL}(\rho\|\pi)}{\alpha n}  - \hat{R}_t(\theta^*_t)\right]
\\
& = \inf_{\rho\in \mathcal{P}(\Theta)} \left[] \mathbb{E}_{\theta\sim\rho}[R_t(\theta)] - R_t(\theta_t^*)  + \frac{{\rm KL}(\rho\|\pi)}{\alpha n} \right].
\end{align*}
Now, define $\rho_s$ as the resctriction of $\pi$ to the set $\{\theta: R_t(\theta) - R_t(\theta_t^*) \leq s \}$. Then,
\begin{align*}
\E_{\mathcal{S}_t}\left[\inf_{\rho\in \mathcal{P}(\Theta)}\Rhat_t(\rho, \pi, \alpha) - \hat{R}_t(\theta^*_t)\right]
& \leq \inf_{s>0} \left[ \mathbb{E}_{\theta\sim\rho}[R_t(\theta)] - R_t(\theta_t^*)  + \frac{{\rm KL}(\rho_s\|\pi)}{\alpha n} \right]
\\
& \leq \inf_{s>0} \left[ s + \frac{\log\frac{1}{\pi( \{\theta: R_t(\theta) - R_t(\theta_t^*) \leq s \})} }{\alpha n} \right]
\\
& \leq \inf_{s>0} \left[ s + \frac{d_{\pi,t}\log\frac{1}{s} + \log \kappa_{\pi,t}  }{\alpha n} \right]
\end{align*}
by assumption. An optimization with respect to $s$ leads to $s=d_{\pi,t}/(\alpha n)$ and we obtain the first statement:
$$ \E_{\mathcal{S}_t}\left[\inf_{\rho\in \mathcal{P}(\Theta)}\Rhat_t(\rho, \pi, \alpha) - \hat{R}_t(\theta^*_t)\right] \leq \frac{d_{\pi,t} \log \frac{n\alpha}{d_\pi} + \log \kappa_{\pi,t} }{\alpha n}. $$
Plugging this into Theorem~\ref{theorem_bound_isolation} leads immediately to the other statements.

\section{Proof of Lemma \ref{function_inequality_lemma}}\label{function_inequality_lemma_proof}

First, we note that $f: x\mapsto -\frac{1}{\tau}\log(x)$ is differentiable on $[\exp(-C\tau),1]$ and $f'(x) = -\frac{1}{\tau x}$. As a consequence, $|f'(x)|= 1/(\tau x)$ is maximized at $x = \exp(-C\tau)$ and its maximum is $\exp(C\tau)/\tau$. This implies that $f$ is $\exp(C\tau)/\tau$-Lipschitz, that is, for any $(x,y)\in [\exp(-C\tau),1]^2$,
\begin{equation*}
    |f(x)-f(y)| \leq \frac{\exp(C\tau)}{\tau}|x-y|.
\end{equation*}

\noindent Taking the square of both sides of the inequality yields
\begin{equation}
    \forall (x,y)\in [\exp(-C\tau),1]^2, \  (f(x)-f(y))^2 \leq \frac{\exp(2C\tau)}{\tau^2}(x-y)^2. \label{equa:log:1}
\end{equation}

\noindent Then, $f''(x) = 1/(\tau x^2)$ and thus, $|f''(x)| \geq 1/\tau $ (minimum reached for $x=1$). This implies that $f$ is $1/\tau$-strongly convex, that is, for any $(x,y)\in [\exp(-C\tau),1]$ we have, for any $\theta\in[0,1]$,
\begin{equation*}
    f(\theta x + (1-\theta)y) \leq \theta f(x) + (1-\theta)f(y) - \frac{\theta(1-\theta)}{2\tau}(x-y)^2. 
\end{equation*}

\noindent We apply this inequality to $\theta=1/2$ and rearrange terms, and we obtain
\begin{equation}
    \forall (x,y)\in [\exp(-C\tau),1]^2, \ 
    \frac{1}{8\tau} (x-y)^2 \leq \frac{f(x)+f(y)}{2} - f\left(\frac{x+y}{2}\right). \label{equa:log:2}
\end{equation}

\noindent Finally, combining~\eqref{equa:log:1} with~\eqref{equa:log:2} yields
\begin{equation*}
    \forall (x,y)\in [\exp(-C\tau),1]^2, [f(x)-f(y)]^2 \leq \frac{8\exp(2C\tau)}{\tau} \left[\frac{f(x)+f(y)}{2} - f\left(\frac{x+y}{2}\right)\right],
\end{equation*}
concluding the proof of the lemma.
\hfill $\blacksquare$

\section{Proof of Theorem~\ref{theorem_meta_learning}}\label{theorem_meta_learning_proof}

The proof of Theorem~\ref{theorem_meta_learning} is structured as follows: we first bound the excess risk by the expectation of the infimum of the empirical risk $\Rhat(\rho, \pi, \alpha) - \hat{R}(\theta_t^*)$ in Lemma~\ref{lemma_meta_learning}. Using classic techniques, we turn this bound into the prediction risk.

\subsection{Lemma}

\begin{lemma}\label{lemma_meta_learning}
Assume that the loss $\ell$ satisfies the boundedness assumption~\eqref{bounded_assumption}. Then, the following bound holds with the choice $\beta = \frac{1}{C+c}$:
\begin{multline*}
    \E_{P_1, \dots, P_T}\E_{S_1, \dots, S_T}\E_{\pi\sim\hat{\Pi}}\left[\mathcal{E}(\pi)\right] - \mathcal{E}^* \leq \frac{2}{1 - \frac{\alpha c \1_B}{2(1-C\alpha)}}\E_{P_1, \dots, P_T}\E_{S_1, \dots, S_T}\Bigg[ \\
    \inf_{\Pi\in \G}\left\{\frac{1}{T}\sum_{t=1}^T \left(\E_{\pi\sim \Pi}\left[\Rhat_t(\rho_t(\pi, \alpha), \pi, \alpha)\right]  - \hat{R}_t(\theta_t^*)\right) + \frac{\KL(\Pi\Vert \Lambda)}{\beta T}\right\} + \frac{\alpha C^2 (1-\1_B)}{8}\Bigg],
\end{multline*}
\end{lemma}

\begin{proof}
For any $t\in [T]$, let
\begin{equation*}
    U_t := \Rhat_t(\rho_t(\pi^*_{\alpha}, \alpha), \pi^*_{\alpha}, \alpha) - \Rhat_t(\rho_t(\pi, \alpha), \pi, \alpha).
\end{equation*}

\noindent Please note that
\begin{equation*}
    \E[U_t] = \E_{P_t}\left[\R_t(\pi^*_{\alpha})\right] - \E_{P_t}\left[\R_t(\pi)\right],
\end{equation*}
where $\E[U_t]$ is a shortcut notation for $\E_{P_t\sim \mathcal{P}}\E_{S_t\sim P_t}[U_t]$. Besides, please note that, by the assumption on the boundedness of $\ell$, it a.s. holds that $|U_t| \leq C$. Applying Lemma~\ref{lemma:bernstein} to $U_t$ gives, for any $\beta >0$,
\begin{equation*}
    \E_{P_1, \dots, P_T}\E_{S_1, \dots, S_T}\left[e^{\beta\sum_{t=1}^T \left(U_t - \E_{S_t}[U_t]\right)}\right] \leq e^{\frac{\beta^2 T \Tilde{V}(\pi)}{2(1-\beta C)}},
\end{equation*}
where $\Tilde{V}(\pi) = \E_{P_{T+1}}\E_{S_{T+1}}\left[\left(\Rhat_{T+1}(\rho_{T+1}(\pi, \alpha), \pi, \alpha) - \Rhat_{T+1}(\rho_{T+1}(\pi^*_{\alpha}, \alpha), \pi^*_{\alpha}, \alpha)\right)^2\right]$. This factor can be bounded as $\Tilde{V}(\pi) \leq c \E[-U_{T+1}]$ by Theorem~\ref{bernstein_hypothesis_theorem}, which states that Bernstein hypothesis is satisfied at the meta level, so that the bound becomes
\begin{equation*}
    \E_{P_1, \dots, P_T}\E_{S_1, \dots, S_T}\left[e^{\beta\sum_{t=1}^T \left(U_t - \E[U_{t}]\right) + \frac{c\beta^2 T \E[U_{T+1}]}{2(1-\beta C)}}\right] \leq 1.
\end{equation*}

\noindent Integrating with respect to the prior $\pi\sim \Lambda$ and using Fubini's theorem yields
\begin{equation*}
    \E_{P_1, \dots, P_T}\E_{S_1, \dots, S_T}\E_{\pi\sim\Lambda}\left[e^{\beta\sum_{t=1}^T \left(U_t - \E[U_t]\right) + \frac{c\beta^2 T \E[U_{T+1}]}{2(1-\beta C)}}\right] \leq 1.
\end{equation*}

\noindent Next, by an application of Lemma \ref{lemma:dv}, the left-hand side becomes
\begin{equation*}
    \E_{P_1, \dots, P_T}\E_{S_1, \dots, S_T}\left[e^{\sup_{\Pi\in \mathcal{P}(\mathcal{P}(\theta))}\left\{\E_{\pi\sim\Pi}\left[\beta\sum_{t=1}^T \left(U_t - \E[U_t]\right) + \frac{c\beta^2 T \E[U_{T+1}]}{2(1-\beta C)}\right] - \KL(\Pi\Vert \Lambda)\right\}}\right] \leq 1.
\end{equation*}

\noindent We then make use of Jensen's inequality and arrange terms, so that the bound becomes
\begin{equation*}
    \E_{P_1, \dots, P_T}\E_{S_1, \dots, S_T}\left[\sup_{\Pi\in \mathcal{P}(\mathcal{P}(\theta))}\left\{\E_{\pi\sim\Pi}\left[\frac{1}{T}\sum_{t=1}^T \left(U_t - \E[U_t]\right) + \frac{c\beta \E[U_{T+1}]}{2(1-\beta C)}\right] - \frac{\KL(\Pi\Vert \Lambda)}{\beta T}\right\}\right] \leq 0.
\end{equation*}

\noindent We replace the supremum on $\Pi$ by an evaluation of the term in $\hat{\Pi}$ and arrange terms, so that the bound becomes
\begin{multline*}
    \E_{P_1, \dots, P_T}\E_{S_1, \dots, S_T}\E_{\pi\sim\hat{\Pi}}\left[-\frac{1}{T}\sum_{t=1}^T \E[U_t] + \frac{c\beta \E[U_{T+1}]}{2(1-\beta C)}\right] \\
    \leq \E_{P_1, \dots, P_T}\E_{S_1, \dots, S_T}\left[-\frac{1}{T}\sum_{t=1}^T \E_{\pi\sim\hat{\Pi}}[U_t] + \frac{\KL(\hat{\Pi}\Vert \Lambda)}{\beta T}\right],
\end{multline*}
which is identical to
\begin{multline*}
    \left(1 - \frac{c\beta}{2(1-C\beta)}\right)\E_{P_1, \dots, P_T}\E_{S_1, \dots, S_T}\E_{\pi\sim\hat{\Pi}}\E_{P_{T+1}}\E_{S_{T+1}}\left[U_{T+1}\right] \leq \\
    \E_{P_1, \dots, P_T}\E_{S_1, \dots, S_T}\left[-\frac{1}{T}\sum_{t=1}^T \E_{\pi\sim\hat{\Pi}}[U_t] + \frac{\KL(\hat{\Pi}\Vert \Lambda)}{\beta T}\right].
\end{multline*}

\noindent Then, we replace $\E[U_{T+1}]$ by its value, yielding the bound
\begin{multline*}
    \left(1 - \frac{c\beta}{2(1-C\beta)}\right)\E_{P_1, \dots, P_T}\E_{S_1, \dots, S_T}\E_{\pi\sim\hat{\Pi}}\E_{P_{T+1}}\E_{S_{T+1}}\Bigg[ \\
    \Rhat_{T+1}(\rho_{T+1}(\pi, \alpha), \pi, \alpha) - \Rhat_{T+1}(\rho_{T+1}(\pi^*_{\alpha}, \alpha), \pi^*_{\alpha}, \alpha)\Bigg] \leq \E_{P_1, \dots, P_T}\E_{S_1, \dots, S_T}\Bigg[ \\
    -\frac{1}{T}\sum_{t=1}^T \E_{\pi\sim\hat{\Pi}}\left[\Rhat_t(\rho_t(\pi, \alpha), \pi, \alpha) - \Rhat_t(\rho_t(\pi^*_{\alpha}, \alpha), \pi^*_{\alpha}, \alpha)\right] + \frac{\KL(\hat{\Pi}\Vert \Lambda)}{\beta T}\Bigg].
\end{multline*}

\noindent Since the term $\Rhat_t(\rho_t(\pi^*_{\alpha}, \alpha), \pi^*_{\alpha}, \alpha)$ does not depend on $\pi\sim \hat{\Pi}$, we can simplify it on both sides of the inequality, which gives
\begin{multline}
    \left(1 - \frac{c\beta}{2(1-C\beta)}\right)\E_{P_1, \dots, P_T}\E_{S_1, \dots, S_T}\E_{\pi\sim\hat{\Pi}}\E_{P_{T+1}}\E_{S_{T+1}}\left[\Rhat_{T+1}(\rho_{T+1}(\pi, \alpha), \pi, \alpha)\right] \\
    + \frac{c\beta}{2(1-C\beta)}\E_{P_{T+1}}\E_{S_{T+1}}\left[ \Rhat_{T+1}(\rho_{T+1}(\pi^*_{\alpha}, \alpha), \pi^*_{\alpha}, \alpha)\right] \\
    \leq \E_{P_1, \dots, P_T}\E_{S_1, \dots, S_T}\left[ -\frac{1}{T}\sum_{t=1}^T \E_{\pi\sim\hat{\Pi}}\left[\Rhat_t(\rho_t(\pi, \alpha), \pi, \alpha)\right] + \frac{\KL(\hat{\Pi}\Vert \Lambda)}{\beta T}\right]. \label{big_inequality}
\end{multline}

\noindent Theorem \ref{theorem_bound_isolation} provides the following lower bound for any $\pi'$:
\begin{multline}
    \E_{S_{T+1}}\left[\Rhat_{T+1}(\rho_{T+1}(\pi', \alpha), \pi', \alpha)\right] \geq \left(1 - \frac{\alpha c \1_B}{2(1-C\alpha)}\right)\E_{S_{T+1}}\E_{\theta\sim\rho_{T+1}(\pi', \alpha)}\left[R_{P_{T+1}}(\theta)\right] \\
    + \frac{\alpha c \1_B}{2(1-C\alpha)} R_{P_{T+1}}^* - \frac{\alpha C^2 (1-\1_B)}{8}. \label{inequality_pi}
\end{multline}

\noindent This further implies that
\begin{multline}
    \E_{S_{T+1}}\left[\Rhat_{T+1}(\rho_{T+1}(\pi', \alpha), \pi', \alpha)\right] \geq \left(1 - \frac{\alpha c \1_B}{2(1-C\alpha)}\right) R_{P_{T+1}}^* \\
    + \frac{\alpha c \1_B}{2(1-C\alpha)} R_{P_{T+1}}^* - \frac{\alpha C^2 (1-\1_B)}{8} \label{inequality_pi_star}
\end{multline}
for any $\pi'$. In particular, applying \eqref{inequality_pi} to $\pi' = \pi$ and \eqref{inequality_pi_star} to $\pi = \pi^*_{\alpha}$, and injecting the results in the left-hand side of \eqref{big_inequality} gives
\begin{multline*}
    \left(1 - \frac{\alpha c \1_B}{2(1-C\alpha)}\right)\Bigg(\frac{c\beta}{2(1-C\beta)} \E_{P_{T+1}}\left[R_{P_{T+1}}^*\right] \\
    + \left(1 - \frac{c\beta}{2(1-C\beta)}\right) \E_{P_1, \dots, P_T}\E_{S_1, \dots, S_T}\E_{\pi\sim\hat{\Pi}}\E_{P_{T+1}}\E_{S_{T+1}}\E_{\theta\sim\rho_{T+1}(\pi', \alpha)}\left[R_{P_{T+1}}(\theta)\right] \Bigg) \\
    + \frac{\alpha c \1_B}{2(1-C\alpha)} \E_{P_{T+1}}\left[R_{P_{T+1}}^*\right] - \frac{\alpha C^2 (1-\1_B)}{8} \\
    \leq \E_{P_1, \dots, P_T}\E_{S_1, \dots, S_T}\left[ -\frac{1}{T}\sum_{t=1}^T \E_{\pi\sim\hat{\Pi}}\left[\Rhat_t(\rho_t(\pi, \alpha), \pi, \alpha)\right] + \frac{\KL(\hat{\Pi}\Vert \Lambda)}{\beta T}\right].
\end{multline*}

\noindent We remove $\mathcal{E}^* = \E_{P_{T+1}}\left[R_{P_{T+1}}^*\right]$ from both sides of the inequality and arrange terms, so that the bound becomes
\begin{multline*}
    \left(1 - \frac{\alpha c \1_B}{2(1-C\alpha)}\right)\left(1 - \frac{c\beta}{2(1-C\beta)}\right)\Bigg(\E_{P_1, \dots, P_T}\E_{S_1, \dots, S_T}\E_{\pi\sim\hat{\Pi}}\left[\mathcal{E}(\pi)\right] - \mathcal{E}^*\Bigg) - \frac{\alpha C^2 (1-\1_B)}{8} \\
    \leq \E_{P_1, \dots, P_T}\E_{S_1, \dots, S_T}\left[ -\frac{1}{T}\sum_{t=1}^T \E_{\pi\sim\hat{\Pi}}\left[\Rhat_t(\rho_t(\pi, \alpha), \pi, \alpha)\right] - \mathcal{E}^* + \frac{\KL(\hat{\Pi}\Vert \Lambda)}{\beta T}\right].
\end{multline*}

\noindent By definition, $\hat{\Pi}$ is the minimizer of the integrand of the right-hand side, and therefore,
\begin{multline*}
    \left(1 - \frac{\alpha c \1_B}{2(1-C\alpha)}\right)\left(1 - \frac{c\beta}{2(1-C\beta)}\right)\Bigg(\E_{P_1, \dots, P_T}\E_{S_1, \dots, S_T}\E_{\pi\sim\hat{\Pi}}\left[\mathcal{E}(\pi)\right] - \mathcal{E}^*\Bigg) \leq \\
    \E_{P_1, \dots, P_T}\E_{S_1, \dots, S_T}\left[\inf_{\Pi\in \G}\left\{\frac{1}{T}\sum_{t=1}^T \left(\E_{\pi\sim \Pi}\left[\Rhat_t(\rho_t(\pi, \alpha), \pi, \alpha)\right] - \hat{R}_t(\theta_t^*)\right) + \frac{\KL(\Pi\Vert \Lambda)}{\beta T}\right\}\right] \\
    + \frac{\alpha C^2 (1-\1_B)}{8},
\end{multline*}
and the choice $\beta = \frac{1}{c+C}$ yields
\begin{multline*}
    \E_{P_1, \dots, P_T}\E_{S_1, \dots, S_T}\E_{\pi\sim\hat{\Pi}}\left[\mathcal{E}(\pi)\right] - \mathcal{E}^* \leq \frac{2}{1 - \frac{\alpha c \1_B}{2(1-C\alpha)}}\E_{P_1, \dots, P_T}\E_{S_1, \dots, S_T}\Bigg[ \\
    \inf_{\Pi\in \G}\left\{\frac{1}{T}\sum_{t=1}^T \left(\E_{\pi\sim \Pi}\left[\Rhat_t(\rho_t(\pi, \alpha), \pi, \alpha)\right]  - \hat{R}_t(\theta_t^*)\right) + \frac{\KL(\Pi\Vert \Lambda)}{\beta T}\right\} + \frac{\alpha C^2 (1-\1_B)}{8}\Bigg],
\end{multline*}
which concludes the proof of the lemma.
\end{proof}

\subsection{Proof of Theorem~\ref{theorem_meta_learning}}

From Lemma~\ref{lemma_meta_learning},
\begin{align*}
    \E_{P_1, \dots, P_T} & \E_{S_1, \dots, S_T}\E_{\pi\sim\hat{\Pi}}\left[\mathcal{E}(\pi)\right] - \mathcal{E}^*
    \\
    & \leq \frac{2}{1 - \frac{\alpha c \1_B}{2(1-C\alpha)}}\E_{P_1, \dots, P_T}\E_{S_1, \dots, S_T}\Bigg[ \inf_{\Pi\in \G}\Biggl\{\frac{1}{T}\sum_{t=1}^T \left(\E_{\pi\sim \Pi}\left[\Rhat_t(\rho_t(\pi, \alpha), \pi, \alpha)\right]  - \hat{R}_t(\theta_t^*)\right) \\
    & \quad \quad \quad  + \frac{\KL(\Pi\Vert \Lambda)}{\beta T}\Biggr\} + \frac{\alpha C^2 (1-\1_B)}{8} \Bigg]
    \\
    & \leq \frac{2}{1 - \frac{\alpha c \1_B}{2(1-C\alpha)}}  \inf_{\Pi\in \G} \E_{P_1, \dots, P_T}\E_{S_1, \dots, S_T}\Biggl\{\frac{1}{T}\sum_{t=1}^T \left(\E_{\pi\sim \Pi}\left[\Rhat_t(\rho_t(\pi, \alpha), \pi, \alpha)\right]  - \hat{R}_t(\theta_t^*)\right) \\
    & \quad \quad \quad  + \frac{\KL(\Pi\Vert \Lambda)}{\beta T}+ \frac{\alpha C^2 (1-\1_B)}{8}\Biggr\}  
    \\
    & = \frac{2}{1 - \frac{\alpha c \1_B}{2(1-C\alpha)}} \inf_{\Pi\in \G}\E_{P_{T+1}}  \E_{S_{T+1}} \Biggl\{ \left(\E_{\pi\sim \Pi}\left[\Rhat_{T+1}(\rho_{T+1}(\pi, \alpha), \pi, \alpha)\right]  - \hat{R}_{T+1}(\theta_{T+1}^*)\right)
    \\
    & \quad \quad \quad  + \frac{\KL(\Pi\Vert \Lambda)}{\beta T}+ \frac{\alpha C^2 (1-\1_B)}{8}\Biggr\}  
    \\
    & \leq \frac{2}{1 - \frac{\alpha c \1_B}{2(1-C\alpha)}} \inf_{\Pi\in \G}\E_{P_{T+1}}  \E_{S_{T+1}} \Biggl\{ \E_{\pi\sim \Pi}\inf_{\rho\in\mathcal{P}(\Theta)}\Biggl[\mathbb{E}_{\theta\sim\rho}[\hat{R}_{T+1}(\theta)] + \frac{{\rm KL}(\rho\Vert\pi)}{\alpha n}    
    \\
    & \quad \quad \quad   - \hat{R}_{T+1}(\theta_{T+1}^*)\Biggr] + \frac{\KL(\Pi\Vert \Lambda)}{\beta T}+ \frac{\alpha C^2 (1-\1_B)}{8}\Biggr\}  
    \\
    & \leq \frac{2}{1 - \frac{\alpha c \1_B}{2(1-C\alpha)}} \inf_{\Pi\in \G}\E_{P_{T+1}}  \Biggl\{ \E_{\pi\sim \Pi}\inf_{\rho\in\mathcal{P}(\Theta)} \E_{S_{T+1}}  \Biggl[\mathbb{E}_{\theta\sim\rho}[\hat{R}_{T+1}(\theta)] + \frac{{\rm KL}(\rho\Vert\pi)}{\alpha n}  
    \\
    & \quad \quad \quad   - \hat{R}_{T+1}(\theta_{T+1}^*) \Biggr] + \frac{\KL(\Pi\Vert \Lambda)}{\beta T}+ \frac{\alpha C^2 (1-\1_B)}{8}\Biggr\}  
    \\
    & \leq \frac{2}{1 - \frac{\alpha c \1_B}{2(1-C\alpha)}} \inf_{\Pi\in \G}\E_{P_{T+1}}  \Biggl\{ \E_{\pi\sim \Pi}\inf_{\rho\in\mathcal{P}(\Theta)} \Biggl[\mathbb{E}_{\theta\sim\rho}[R_{T+1}(\theta) -R_{T+1}(\theta_t^*)] + \frac{{\rm KL}(\rho\Vert\pi)}{\alpha n}\Biggr]  
    \\
    & \quad \quad \quad   + \frac{\KL(\Pi\Vert \Lambda)}{\beta T}+ \frac{\alpha C^2 (1-\1_B)}{8}\Biggr\}  .
\end{align*}
This ends the proof. 
\hfill $\blacksquare$

\section{Application of Theorem~\ref{theorem_meta_learning} to the Gaussian Case}\label{appendix_application_gaussian}

\noindent Assume that $\ell$ is $L$-Lipschitz. As a result, the risks $R_{P_t}$ are also $L$-Lipschitz. We choose the prior $p_{\bar{\mu},(\bar{\sigma}^2,\dots,\bar{\sigma}^2)} \in \F$ (and hence, the variances are the same for all coordinates). A straightforward application of \eqref{kl_gaussian} gives
\begin{equation*}
    \KL(p_{\mu,\sigma^2},p_{\bar{\mu},(\bar{\sigma}^2,\dots,\bar{\sigma}^2)}) = \frac{1}{2}\sum_{i=1}^d \left[ \frac{(\mu_i-\bar{\mu}_i)^2}{\bar{\sigma}^2} + \frac{\sigma_i^2}{\bar{\sigma}^2}-1+\log\left(\frac{\bar{\sigma}^2}{\sigma_i^2}\right) \right].
\end{equation*}

\noindent In this case, $\rho_t(\pi,\alpha) = p_{\mu(t),\sigma^2(t)}$, where
\begin{equation*}
    (\mu(t),\sigma^2(t)) = \argmin_{\mu,\sigma^2} \left\{ \E_{\theta\sim\mathcal{N}(\mu,\sigma^2)}\left[ \hat{R}_{t}(\theta)\right] + \frac{1}{2\alpha n}\sum_{i=1}^d \left[\frac{(\mu_i-\bar{\mu}_i)^2}{\bar{\sigma}^2} + \frac{\sigma_i^2}{\bar{\sigma}^2}-1+\log\left(\frac{\bar{\sigma}^2}{\sigma_i^2}\right) \right]\right\}.
\end{equation*}

\noindent We now define $\G$ as the family of distributions $q_{\tau,v,a,b}$ on $(\bar{\mu},\bar{\sigma}^2)$, where
\begin{equation*}
    q_{\tau,v,a,b}(\bar{\mu},\bar{\sigma}^2)  = \left[\bigotimes_{i=1}^d \N(\bar{\mu}_i;\tau_i,\xi^2) \right] \otimes \Gamma(\bar{\sigma}^2;a,b).
\end{equation*}

\noindent Fix a prior on priors $\Lambda = q_{0,\bar{\bar{\xi}}^2,\bar{\bar{a}},\bar{\bar{b}}}$. We choose $\hat{\Pi}=q_{\hat{t},\hat{\xi}^2,\hat{a},\hat{b}}$, where
\begin{multline*}
    (\hat{\tau},\hat{\xi}^2,\hat{a},\hat{b}) = \argmin_{\tau,\xi^2,a,b} \Biggl\{ \E_{(\bar{\mu},\bar{\sigma}^2)\sim q_{\tau,\xi^2,a,b}} \Biggl[\frac{1}{T}\sum_{t=1}^T \min_{\mu(t),\sigma^2(t)} \Biggl\{  \E_{\theta\sim\N(\mu(t),\sigma^2(t))}\left[ \hat{R}_{t}(\theta)\right] \\
    + \frac{1}{2\alpha n}\sum_{i=1}^d \left[ \frac{(\mu_i(t)-\bar{\mu}_i)^2}{\bar{\sigma}^2} + \frac{\sigma_i^2(t)}{\bar{\sigma}^2}-1+\log\left(\frac{\bar{\sigma}^2}{\sigma_i^2(t)}\right) \right] \Biggr\} \Biggr] \\
    + \frac{1}{2\beta T}\sum_{i=1}^d \left[ \frac{\tau_i^2}{\bar{\bar{\xi}}^2} + \frac{\xi^2}{\bar{\bar{\xi}}^2}-1+\log\left(\frac{\bar{\bar{\xi}}^2}{\xi^2}\right) \right] \\
    + \frac{(a-\bar{\bar{a}})\psi(a) + \log \frac{\Gamma(\bar{\bar{a}})}{\Gamma(a)} + \bar{\bar{a}} \log\frac{b}{\bar{\bar{b}}} + a \frac{\bar{\bar{b}}-b}{b} }{2\beta T}\Biggr\},
\end{multline*}
where $\Gamma(\cdot)$ is the gamma function and $\psi(\cdot) = \Gamma'(\cdot)/\Gamma(\cdot)$ is the digamma function. We apply Theorem~\ref{theorem_meta_learning}:
\begin{multline*}
    \E_{P_{1},\dots,P_{T}} \E_{\mathcal{S}_1,\dots,\mathcal{S}_T} \E_{\pi\sim\hat{\Pi}} [\mathcal{E}(\pi)]-\mathcal{E}^* \leq 4 \inf_{\tau,\xi^2,a,b} \Biggl\{ \\
    \E_{(\bar{\mu},\bar{\sigma}^2)\sim q_{\tau,\xi^2,a,b}} \E_{P_{T+1}}\Biggl[ \min_{\mu(T+1),\sigma^2(T+1)} \Biggl\{ \E_{\theta\sim\N(\mu(T+1),\sigma^2(T+1))}[ R_{P_{T+1}}(\theta)]-R_{P_{T+1}}(\theta^*_t) \\
    +\frac{1}{2\alpha n}\sum_{i=1}^d \left[ \frac{(\mu_i(T+1)-\bar{\mu}_i)^2}{\bar{\sigma}^2} + \frac{\sigma_i^2(T+1)}{\bar{\sigma}^2}-1+\log\left(\frac{\bar{\sigma}^2}{\sigma_i^2(T+1)}\right) \right] \Biggr\} \Biggr] \\
    + \frac{1}{2\beta T}\sum_{i=1}^d \left[ \frac{\tau_i^2}{\bar{\bar{\xi}}^2} + \frac{\xi^2}{\bar{\bar{\xi}}^2}-1+\log\left(\frac{\bar{\bar{\xi}}^2}{\xi^2}\right) \right] + \frac{(a-\bar{\bar{a}})\psi(a) + \log \frac{\Gamma(\bar{\bar{a}})}{\Gamma(a)} + \bar{\bar{a}} \log\frac{b}{\bar{\bar{b}}} + a \frac{\bar{\bar{b}}-b}{b} }{2\beta T}\Biggr\}.
\end{multline*}

\noindent We next apply Assumption~\ref{application_assumption}, and the bound becomes
\begin{multline*}
    \E_{P_{1},\dots,P_{T}} \E_{\mathcal{S}_1,\dots,\mathcal{S}_T}  \E_{\pi\sim\hat{\Pi}} [\mathcal{E}(\pi)] - \mathcal{E}^* \leq  4 \inf_{\tau,\xi^2,a,b} \Biggl\{ \E_{(\bar{\mu},\bar{\sigma}^2)\sim q_{\tau,\xi^2,a,b}} \E_{P_{T+1}} \Biggl[ \min_{\mu,\sigma^2} \Biggl\{ L \|\sigma\|^2 \\
    + \frac{1}{2\alpha n}\sum_{i=1}^d \left[ \frac{(\mu_i-\bar{\mu}_i)^2}{\bar{\sigma}^2} + \frac{\sigma_i^2}{\bar{\sigma}^2}-1+\log\left(\frac{\bar{\sigma}^2}{\sigma_i^2(T+1)}\right) \right] \Biggr\} \Biggr] \\
    + \frac{1}{2\beta T}\sum_{i=1}^d \left[ \frac{\tau_i^2}{\bar{\bar{\xi}}^2} + \frac{\xi^2}{\bar{\bar{\xi}}^2}-1+\log\left(\frac{\bar{\bar{\xi}}^2}{\xi^2}\right) \right] \\
    + \frac{(a-\bar{\bar{a}})\psi(a) + \log \frac{\Gamma(\bar{\bar{a}})}{\Gamma(a)} + \bar{\bar{a}} \log\frac{b}{\bar{\bar{b}}} + a \frac{\bar{\bar{b}}-b}{b} }{2\beta T}\Biggr\}.
\end{multline*}

\noindent Idea: define for any $P$, $\mu_P$ that minimizes $R_P$, that is: $R_P(\mu_P) = R_p^*$, and define $\mu^* = \E_{P\sim\mathcal{P}}(\mu_P)$. If $\mu_P = \mu^*$ $\mathcal{P}$-a.s., all the task have the same solution. On the other hand, if $\mu_P$ has a lot of variations, then the tasks have very unrelated solutions. In the infimum above, take $\mu=\mu_{P_{T+1}}$ and $\tau = \mu^*$. Then,
\begin{multline*}
    \E_{P_{1},\dots,P_{T}} \E_{\mathcal{S}_1,\dots,\mathcal{S}_T}  \E_{\pi\sim\hat{\Pi}} [\mathcal{E}(\pi)] \leq \mathcal{E}^* + \inf_{\xi^2,a,b} \Biggl\{ \E_{\bar{\sigma}^2\sim \Gamma^{-1}(a,b)} \Bigg[ \min_{\sigma^2 } \Biggl\{ L\|\sigma \|^2 \\
    + \frac{\E_{P_{T+1}\sim\mathcal{P}} \left( \|\mu_{P_{T+1}}-\mu^*\|^2 \right) + \xi^2 d}{2 \bar{\sigma}^2 \alpha n} + \frac{1}{2\alpha n}\sum_{i=1}^d \left[ \frac{\sigma_i^2 }{\bar{\sigma}^2}-1+\log\left(\frac{\bar{\sigma}^2}{\sigma_i^2 }\right) \right] \Biggr\} \Bigg] \\
    + \frac{\beta C^2}{8 T} + \frac{\|\mu^*\|^2}{2\beta \bar{\bar{\xi}}^2} + \frac{d}{2\beta} \left[ \frac{\xi^2}{\bar{\bar{\xi}}^2}-1+\log\left(\frac{\bar{\bar{\xi}}^2}{\xi^2}\right) \right] \\
    + \frac{(a-\bar{\bar{a}})\psi(a) + \log \frac{\Gamma(\bar{\bar{a}})}{\Gamma(a)} + \bar{\bar{a}} \log\frac{b}{\bar{\bar{b}}} + a \frac{\bar{\bar{b}}-b}{b} }{2\beta T} \Biggr\}.
\end{multline*}

\noindent Let $\Sigma(\mathcal{P}) = \E_{P_{T+1}\sim\mathcal{P}} \left[ \|\mu_{P_{T+1}}-\mu^*\|^2 \right]$, this quantity will be very important in the rate. Therefore,
\begin{multline*}
    \E_{P_{1},\dots,P_{T}} \E_{\mathcal{S}_1,\dots,\mathcal{S}_T}  \E_{\pi\sim\hat{\Pi}} [\mathcal{E}(\pi)] - \mathcal{E}^* \leq \\
    4 \inf_{\xi^2, a, b, \beta} \Biggl\{ \E_{\bar{\sigma}^2\sim \Gamma^{-1}(a,b)} \Bigg[ \min_{\sigma^2 } \Biggl\{ L\|\sigma \|^2 + \frac{\Sigma(\mathcal{P}) + \xi^2 d}{2 \bar{\sigma}^2 \alpha n} + \frac{1}{2\alpha n}\sum_{i=1}^d \left[ \frac{\sigma_i^2 }{\bar{\sigma}^2}-1+\log\left(\frac{\bar{\sigma}^2}{\sigma_i^2 }\right) \right] \Biggr\} \Bigg] \\
    + \frac{\|\mu^*\|^2}{2\beta \bar{\bar{\xi}}^2 T} + \frac{d}{2\beta} \left[ \frac{\xi^2}{\bar{\bar{\xi}}^2}-1+\log\left(\frac{\bar{\bar{\xi}}^2}{\xi^2}\right) \right] + \frac{(a-\bar{\bar{a}})\psi(a) + \log \frac{\Gamma(\bar{\bar{a}})}{\Gamma(a)} + \bar{\bar{a}} \log\frac{b}{\bar{\bar{b}}} + a \frac{\bar{\bar{b}}-b}{b} }{2\beta T} \Biggr\}.
\end{multline*}

\noindent An exact optimization in $\sigma^2$ gives $\sigma_i^2 = \frac{\bar{\sigma}^2}{(2\alpha L\bar{\sigma}^2 n + 1 )}$ and after simplifications,
\begin{multline*}
    \E_{P_{1},\dots,P_{T}} \E_{\mathcal{S}_1,\dots,\mathcal{S}_T}  \E_{\pi\sim\hat{\Pi}} [\mathcal{E}(\pi)] \\
    \leq \mathcal{E}^* + 4 \inf_{\xi^2,a,b} \Biggl\{ \E_{\bar{\sigma}^2\sim \Gamma^{-1}(a,b)} \left[ \frac{\Sigma(\mathcal{P}) + \xi^2 d}{2 \bar{\sigma}^2 \alpha n} + \frac{d\log\left(2\alpha L \bar{\sigma}^2 n +1\right)}{2\alpha n } \right] \\
    + \frac{\|\mu^*\|^2}{2\beta \bar{\bar{\xi}}^2 T} + \frac{d}{2\beta} \left[ \frac{\xi^2}{\bar{\bar{\xi}}^2}-1+\log\left(\frac{\bar{\bar{\xi}}^2}{\xi^2}\right) \right] + \frac{(a-\bar{\bar{a}})\psi(a) + \log \frac{\Gamma(\bar{\bar{a}})}{\Gamma(a)} + \bar{\bar{a}} \log\frac{b}{\bar{\bar{b}}} + a \frac{\bar{\bar{b}}-b}{b} }{2\beta T} \Biggr\}.
\end{multline*}

\noindent We now consider separately two cases: $\Sigma(\mathcal{P})>d\epsilon$ or $\Sigma(\mathcal{P})\leq d\epsilon$ for some $\epsilon$ (very small) that we will choose later. First case: $\Sigma(\mathcal{P})>d\epsilon$. This is the case where we do not expect a significant improvement on the bound from the meta-learning. Simply take $\xi = \bar{\bar{\xi}}$ to get
\begin{multline*}
    \E_{P_{1},\dots,P_{T}} \E_{\mathcal{S}_1,\dots,\mathcal{S}_T}  \E_{\pi\sim\hat{\Pi}} [\mathcal{E}(\pi)] - \mathcal{E}^* \\
    \leq 4 \inf_{a,b} \Biggl\{ \E_{\bar{\sigma}^2\sim \Gamma^{-1}(a,b)} \left[ \frac{\Sigma(\mathcal{P}) + d\bar{\bar{\xi}}^2}{\bar{\sigma}^2 \alpha n} + \frac{d\log\left(2\alpha L \bar{\sigma}^2 n +1\right)}{2\alpha n } \right] \\
    + \frac{2 \|\mu^*\|^2}{\beta \bar{\bar{\xi}}^2 T} + \frac{(a-\bar{\bar{a}})\psi(a) + 2 \log \frac{\Gamma(\bar{\bar{a}})}{\Gamma(a)} + \bar{\bar{a}} \log\frac{b}{\bar{\bar{b}}} + a \frac{\bar{\bar{b}}-b}{b} }{\beta T} \Biggr\}.
\end{multline*} 

\noindent In this case, an accurate optimization with respect to $a$ and $b$ will not improve the bound significantly, and for this reason, we simply take $a=\bar{\bar{a}}$ and $b=\bar{\bar{b}}$. Also note that for $U\sim\Gamma(a,b)$ we have $\E(U^x) = \Gamma(a+x)/[b^x \Gamma(a)]$ and thus $\E(1/\bar{\sigma}^2)= b/(a-1)$ and $ \E(\log(2\alpha L \bar{\sigma}^2 n +1 )) \leq \log(2\alpha L \E(\bar{\sigma}^2) n +1 )  = \log(2\alpha L n a/b+1 )$. Thus,
\begin{equation*}
    \E_{P_{1},\dots,P_{T}} \E_{\mathcal{S}_1,\dots,\mathcal{S}_T}  \E_{\pi\sim\hat{\Pi}} [\mathcal{E}(\pi)] \leq  \mathcal{E}^* + \frac{ 4 \bar{\bar{b}}\left[d\bar{\bar{\xi}}^2+ \Sigma(\mathcal{P})\right]}{  (\bar{\bar{a}}-1) \alpha n} + \frac{2 d\log\left(\frac{2 \bar{\bar{a}}\alpha L  n}{\bar{\bar{b}}} +1\right)}{\alpha n } + \frac{2 \|\mu^*\|^2}{\beta \bar{\bar{\xi}}^2 T}.
\end{equation*}

\noindent Second case: $\Sigma(\mathcal{P})\leq d\epsilon$. In this case, we expect a significant improvement in the bound from the meta-learning, and in order to take advantage of it, we choose $\xi$ very small: $\xi^2 = \epsilon$. We obtain
\begin{multline*}
    \E_{P_{1},\dots,P_{T}} \E_{\mathcal{S}_1,\dots,\mathcal{S}_T}  \E_{\pi\sim\hat{\Pi}} [\mathcal{E}(\pi)] \leq  \mathcal{E}^* + 4 \inf_{a,b} \Biggl\{ \E_{\bar{\sigma}^2\sim \Gamma^{-1}(a,b)} \Biggl\{ \frac{d\epsilon}{\bar{\sigma}^2 \alpha n} + \frac{d\log\left(2\alpha L \bar{\sigma}^2 n +1\right)}{2\alpha n } \Biggr\} \\
    + \frac{\|\mu^*\|^2}{2\beta \bar{\bar{\xi}}^2 T} + \frac{d}{2\beta T} \left[ \frac{\epsilon}{\bar{\bar{\xi}}^2}-1+\log\left(\frac{\bar{\bar{\xi}}^2}{\epsilon}\right) \right] \\
    + \frac{(a-\bar{\bar{a}})\psi(a) + \log \frac{\Gamma(\bar{\bar{a}})}{\Gamma(a)} + \bar{\bar{a}} \log\frac{b}{\bar{\bar{b}}} + a \frac{\bar{\bar{b}}-b}{b} }{2\beta T}  \Biggr\}.
\end{multline*}

\noindent In order to take advantage of the meta-learning, we are going to make sure that $\bar{\sigma}^2$ is very small by tuning $a$ and $b$ adequately. But then we have $ \E(\log(2\alpha L \bar{\sigma}^2 n +1 )) \leq 2\alpha L \E(\bar{\sigma}^2) n  = 2\alpha L n a/b$, and we obtain
\begin{multline*}
    \E_{P_{1},\dots,P_{T}} \E_{\mathcal{S}_1,\dots,\mathcal{S}_T}  \E_{\pi\sim\hat{\Pi}} [\mathcal{E}(\pi)] \leq  \mathcal{E}^* + 4 \inf_{a,b,\beta\in G} \Biggl\{ \frac{bd\epsilon}{(a-1) \alpha n} + \frac{dLa}{b} + \frac{\|\mu^*\|^2}{2\beta \bar{\bar{\xi}}^2 T} \\
    + \frac{d}{2\beta T} \left[ \frac{\epsilon}{\bar{\bar{\xi}}^2}-1+\log\left(\frac{\bar{\bar{\xi}}^2}{\epsilon}\right) \right] + \frac{(a-\bar{\bar{a}})\psi(a) + \log \frac{\Gamma(\bar{\bar{a}})}{\Gamma(a)} + \bar{\bar{a}} \log\frac{b}{\bar{\bar{b}}} + a \frac{\bar{\bar{b}}-b}{b} }{2\beta T} \Biggr\}.
\end{multline*}

\noindent Choose $a=\bar{\bar{a}}$ (make sure that $\bar{\bar{a}}>1$) and an optimization with respect to $b$ gives $b = \sqrt{\frac{\alpha L a(a-1) n }{\epsilon}}$. Reinjecting in the bound gives
\begin{multline*}
    \E_{P_{1},\dots,P_{T}} \E_{\mathcal{S}_1,\dots,\mathcal{S}_T}  \E_{\pi\sim\hat{\Pi}} [\mathcal{E}(\pi)] \leq \mathcal{E}^* + 4 d \sqrt{\frac{\bar{\bar{a}}L}{\alpha(\bar{\bar{a}}-1)}} \sqrt{\frac{\epsilon}{n}} + \frac{2\|\mu^*\|^2}{\beta \bar{\bar{\xi}}^2 T} \\
    + \frac{2d}{\beta T} \left[ \frac{\epsilon}{\bar{\bar{\xi}}^2}-1+\log\left(\frac{\bar{\bar{\xi}}^2}{\epsilon}\right) \right] + \frac{2}{\beta T}\left(\bar{\bar{a}}\log \left( \frac{1}{\bar{\bar{b}}}\sqrt{\frac{n \alpha L \bar{\bar{a}}( \bar{\bar{a}}-1) }{\epsilon}} \right) +  \bar{\bar{a}}  \bar{\bar{b}} \sqrt{\frac{\epsilon}{n \alpha L \bar{\bar{a}}( \bar{\bar{a}}-1) }}\right).
\end{multline*}

\noindent The last step is to chose $\epsilon$. Interestingly enough, for $\epsilon=1/n$ we recover the rate of learning in isolation: $d/n$. Now, an optimization w.r.t $\epsilon$ gives $\epsilon = \sqrt{n}/T$ and thus
\begin{multline*}
    \E_{P_{1},\dots,P_{T}} \E_{\mathcal{S}_1,\dots,\mathcal{S}_T}  \E_{\pi\sim\hat{\Pi}} [\mathcal{E}(\pi)] \leq \mathcal{E}^* + \frac{ 4 d }{T}\sqrt{\frac{\bar{\bar{a}}L}{\alpha(\bar{\bar{a}}-1)}} + \frac{2\|\mu^*\|^2}{\beta \bar{\bar{\xi}}^2 T} \\
    + \frac{2d}{\beta T} \left( \frac{n}{\bar{\bar{\xi}}^2 T^2}-1+\log\left(\bar{\bar{\xi}}^2 T^2\right) \right) + \frac{2}{\beta T}\left(\bar{\bar{a}}\log \left( \frac{T}{\bar{\bar{b}}}\sqrt{ \alpha L \bar{\bar{a}}( \bar{\bar{a}}-1) } \right) +  \frac{\bar{\bar{a}}  \bar{\bar{b}}}{T} \sqrt{ \alpha L \bar{\bar{a}}( \bar{\bar{a}}-1) }\right).
\end{multline*}

\noindent In the $T>n$ regime, this is a significant improvement compared to the learning in isolation in which case, we can simplify the bound as
\begin{equation*}
    \E_{P_{1},\dots,P_{T}} \E_{\mathcal{S}_1,\dots,\mathcal{S}_T}  \E_{\pi\sim\hat{\Pi}} [\mathcal{E}(\pi)] \leq  \mathcal{E}^* +  G(C,\mu^*,\bar{\bar{\xi}},L,\bar{\bar{a}},\bar{\bar{b}}) \frac{d + \log(T)}{T},
\end{equation*}
where
\begin{multline*}
    G(C,\mu^*,\bar{\bar{\xi}},L,\bar{\bar{a}},\bar{\bar{b}}) = 4\sqrt{\frac{\bar{\bar{a}}L}{\alpha(\bar{\bar{a}}-1)}} \\
    +\frac{2}{\beta} \left(\frac{\|\mu^*\|^2}{\bar{\bar{\xi}}^2} + \frac{n}{\bar{\bar{\xi}}^2 T^2} + 1 + \log\bar{\bar{\xi}}^2 + \bar{\bar{a}} + \frac{\bar{\bar{a}}}{2}\log \frac{\alpha L \bar{\bar{a}}( \bar{\bar{a}}-1)}{\bar{\bar{b}}^2} + \frac{\bar{\bar{a}}  \bar{\bar{b}}}{T} \sqrt{ \alpha L \bar{\bar{a}}( \bar{\bar{a}}-1) } \right).
\end{multline*}

\noindent As a conclusion,
\begin{multline*}
    \E_{P_{1},\dots,P_{T}} \E_{\mathcal{S}_1,\dots,\mathcal{S}_T}  \E_{\pi\sim\hat{\Pi}} [\mathcal{E}(\pi)] \leq \mathcal{E}^* + \min\Biggl( \frac{ \bar{\bar{b}}\left[d\bar{\bar{\xi}}^2+ \Sigma(\mathcal{P})\right]}{  (\bar{\bar{a}}-1) \alpha n} + \frac{d\log\left(\frac{2 \bar{\bar{a}}\alpha L  n}{\bar{\bar{b}}} +1\right)}{2\alpha n } +  \frac{2 \|\mu^*\|^2}{\beta \bar{\bar{\xi}}^2 T}, \\
    G(C,\mu^*,\bar{\bar{\xi}},L,\bar{\bar{a}},\bar{\bar{b}}) \frac{d + \log(T)}{T} \Biggr).
\end{multline*}

\section{Application of Theorem~\ref{theorem_meta_learning} to the Case of Mixtures of Gaussians}\label{appendix_application_mixtures}

\noindent We first assume that priors that are mixtures of $K$ Gaussians, where $K$ is known:
\begin{multline*}
    \mathcal{M} = \Bigg\{p_{w, \mu, \sigma^2} = \sum_{k=1}^K w_k \bigotimes_{i=1}^d \N(\mu_{k, i}, \sigma^2_{k, i}): \\
    \forall (i, k)\in [d]\times [K], \mu_{k, i}\in \mathbb{R}, \sigma^2_{k, i}\in \mathbb{R}^*_+, w_k\geq 0, 1^\top w = 1\Bigg\}.
\end{multline*}

\noindent We set the prior $\pi = \sum_{k=1}^K \Bar{w}_k\N(\Bar{\mu}_k, \Bar{\sigma}_k^2 I_d)$. Then, denoting by $g(x; \mu, \sigma^2)$ the pdf of the normal distribution $\N(\mu, \sigma^2)$, \eqref{kl_gaussian} implies, for any $w, \mu, \sigma^2$,
\begin{align*}
    \KL(p_{w, \mu, \sigma^2} \Vert \pi) &= \int_{\mathbb{R}^d} \log \frac{\sum_{k=1}^K w_k g(x; \mu_{k}, \sigma^2_{k})}{\sum_{k=1}^K \Bar{w}_k g(x; \Bar{\mu}_k, \Bar{\sigma}_k^2 I_d)} \sum_{k=1}^K w_k g(x; \mu_{k}, \sigma^2_{k})dx \\
    &\leq \int_{\mathbb{R}^d} \sum_{k=1}^K \log \frac{w_k g(x; \mu_{k}, \sigma^2_{k})}{\Bar{w}_k g(x; \Bar{\mu}_k, \Bar{\sigma}_k^2 I_d)} w_k  g(x; \mu_{k}, \sigma^2_{k})dx \\
    &= \sum_{k=1}^K w_k \log \frac{w_k}{\Bar{w}_k} + \sum_{k=1}^K w_k \KL(\N(\mu_{k}, \sigma^2_{k})\Vert \N(\Bar{\mu}_k, \Bar{\sigma}_k^2 I_d)) \\
    &= \KL(w\Vert \Bar{w}) + \frac{1}{2} \sum_{k=1}^K w_k \sum_{i=1}^d \left(\frac{(\mu_{k, i} - \Bar{\mu}_{k, i})^2}{\Bar{\sigma}_k^2} + \frac{\sigma_{k, i}^2}{\Bar{\sigma}_k^2} - 1 + \log \frac{\Bar{\sigma}_k^2}{\sigma_{k, i}^2}\right),
\end{align*}
where the inequality on the second line follows from the log sum inequality from \citet{cover}, and the bound from Theorem~\ref{theorem_meta_learning} becomes, at $t = T+1$,
\begin{multline*}
    \E_{S_{T+1}}\E_{\theta\sim \rho_{T+1}(\pi, \alpha)}[R_{P_{T+1}}(\theta)] - R^*_{P_{T+1}} \leq 2\inf_{w, \mu, \sigma^2} \Bigg\{\E_{\theta\sim p_{w, \mu, \sigma^2}}[R_{P_{T+1}}(\theta)] - R_{P_{T+1}}^* + \frac{\KL(w\Vert \Bar{w})}{\alpha n} \\
    + \frac{1}{2\alpha n} \sum_{k=1}^K w_k \sum_{i=1}^d \left(\frac{(\mu_{k, i} - \Bar{\mu}_{k, i})^2}{\Bar{\sigma}_k^2} + \frac{\sigma_{k, i}^2}{\Bar{\sigma}_k^2} - 1 + \log \frac{\Bar{\sigma}_k^2}{\sigma_{k, i}^2}\right)\Bigg\}.
\end{multline*}

\noindent Assumption \eqref{application_assumption} implies that
\begin{equation*}
    \E_{\theta\sim \N(\mu, \sigma^2)}[R_{P_{T+1}}(\theta)] - R^*_{P_{T+1}} \leq L\E_{\theta\sim \N(\mu, \sigma^2)}\left[\Vert\theta - \mu_{P_{T+1}}\Vert^2\right].
\end{equation*}

\noindent It follows that the previous bound with the choice $\mu_1 = \dots = \mu_K = \mu_{P_{T+1}}$ becomes
\begin{multline*}
    \E_{S_{T+1}}\E_{\theta\sim \rho_{T+1}(\pi, \alpha)}[R_{P_{T+1}}(\theta)] - R^*_{P_{T+1}} \leq 2\inf_{w, \sigma^2} \Bigg\{ L\sum_{k=1}^K w_k \Vert\sigma_k\Vert^2 + \frac{\KL(w\Vert \Bar{w})}{\alpha n} \\
    + \frac{1}{2\alpha n} \sum_{k=1}^K w_k \sum_{i=1}^d \left(\frac{(\mu_{P_{T+1}, i} - \Bar{\mu}_{k, i})^2}{\Bar{\sigma}_k^2} + \frac{\sigma_{k, i}^2}{\Bar{\sigma}_k^2} - 1 + \log \frac{\Bar{\sigma}_k^2}{\sigma_{k, i}^2}\right)\Bigg\}. 
\end{multline*}

\noindent While the choice $\mu_1 = \dots = \mu_K = \mu_{P_{T+1}}$ may seem less meaningful than in the Gaussian case (with one single component), it is completely natural as the best possible choice for the parameter $\theta$ is $\mu_{P_{T+1}}$. In the computation, each component $\N(\mu_k, \sigma_k^2)$ of the mixture brings an error term which can be decomposed between a bias term and a variance term,
\begin{equation*}
    \E_{\theta\sim \N(\mu_k, \sigma_k^2)}\left[\Vert\theta - \mu_{P_{T+1}}\Vert^2\right] = \underbrace{\Vert\mu_k - \mu_{P_{T+1}}\Vert^2}_{\text{bias term (first order)}} + \underbrace{\sigma_k^2}_{\text{variance term (second order)}},
\end{equation*}
for which the choice $\mu_k = \mu_{P_{T+1}}$ minimizes the first order error term. Next, we set the family $\G$ of distributions on $\F$:
\begin{multline*}
    \G = \Bigg\{q_{\delta, \tau, \xi^2, b} = \Dir(\delta)\otimes \bigotimes_{\substack{k\in [K] \\ i\in [d]}} \N(\tau_{k, i}, \xi_k^2) \otimes \bigotimes_{k=1}^K \Gamma(2, b_k): \\
    \delta = (\delta_1, \dots, \delta_K)\in \mathbb{R}^K, \forall (k, i), \xi_k^2>0, \tau_{k, i}\in \mathbb{R}, b_k>0, \delta_k>0\Bigg\},
\end{multline*}
where $\Dir(\delta)$ is the Dirichlet distribution of parameter $\delta$. We set the prior on priors $\Lambda = q_{1_K, 0, \Bar{\Bar{\xi}}^2, \Bar{\Bar{b}}}$, where $1_K = (1, \dots, 1)$ and $\Bar{\Bar{\xi}}^2 = \left(\Bar{\Bar{\xi}}_1^2, \dots, \Bar{\Bar{\xi}}_K^2\right)$. Then, using \eqref{kl_gaussian}, \eqref{kl_gamma} and \eqref{kl_dirichlet},
\begin{multline*}
    \KL(q_{\delta, \tau, \xi^2, b} \Vert \Lambda) = \log \frac{\Gamma(1^\top \delta)}{\Gamma(K) \times \prod_{k=1}^K \Gamma(\delta_k)} + \sum_{k=1}^K (\delta_k - 1) \left(\psi(\delta_k) - \psi(1^\top \delta)\right) \\
    + \frac{1}{2} \sum_{k, i} \left(\frac{\tau_{k, i}^2}{\Bar{\Bar{\xi}}_k^2} + \frac{\xi_k^2}{\Bar{\Bar{\xi}}_k^2} - 1 + \log \frac{\Bar{\Bar{\xi}}_k^2}{\xi_k^2}\right) + 2\sum_{k=1}^K \left(\log \frac{b_k}{\Bar{\Bar{b}}_k} + \frac{\Bar{\Bar{b}}_k - b_k}{b_k}\right),
\end{multline*}
where $\psi$ is the digamma function. We can next use the bound from Theorem~\ref{theorem_meta_learning} and we have
\begin{multline*}
    \E_{P_1, \dots, P_T}\E_{S_1, \dots, S_T}\E_{\pi\sim \hat{\Pi}}[\mathcal{E}(\pi)] - \mathcal{E}^* \leq \\
    4 \inf_{\delta, \tau, \xi^2, b}\Bigg\{\E_{(\Bar{w}, \Bar{\mu}, \Bar{\sigma}^2)\sim q_{\delta, \tau, \xi^2, b}}\E_{P_{T+1}}\Bigg[ \inf_{w, \sigma^2}\Bigg\{L\sum_{k=1}^K w_k \Vert\sigma_k\Vert^2 + \frac{\KL(w\Vert \Bar{w})}{\alpha n} \\
    + \frac{1}{2\alpha n} \sum_{k=1}^K w_k \sum_{i=1}^d \left(\frac{(\mu_{P_{T+1}, i} - \Bar{\mu}_{k, i})^2}{\Bar{\sigma}_k^2} + \frac{\sigma_{k, i}^2}{\Bar{\sigma}_k^2} - 1 + \log \frac{\Bar{\sigma}_k^2}{\sigma_{k, i}^2}\right)\Bigg\}\Bigg] \\
    + \frac{1}{2\beta T} \log \frac{\Gamma(1^\top \delta)}{\Gamma(K) \times \prod_{k=1}^K \Gamma(\delta_k)} + \frac{1}{2\beta T} \sum_{k=1}^K (\delta_k - 1) \left(\psi(\delta_k) - \psi(1^\top \delta)\right) \\
    + \frac{1}{4\beta T} \sum_{k, i} \left(\frac{\tau_{k, i}^2}{\Bar{\Bar{\xi}}_k^2} + \frac{\xi_k^2}{\Bar{\Bar{\xi}}_k^2} - 1 + \log \frac{\Bar{\Bar{\xi}}_k^2}{\xi_k^2}\right) + \frac{1}{\beta T}\sum_{k=1}^K \left(\log \frac{b_k}{\Bar{\Bar{b}}_k} + \frac{\Bar{\Bar{b}}_k - b_k}{b_k}\right)\Bigg\}.
\end{multline*}

\noindent Next, minimizing over $\sigma_{k, i}^2$ gives the optimal value $\frac{\Bar{\sigma}_k^2}{2\alpha n L \Bar{\sigma}_k^2 + 1}$, and replacing in the above bound gives
\begin{multline*}
    \E_{P_1, \dots, P_T}\E_{S_1, \dots, S_T}\E_{\pi\sim \hat{\Pi}}[\mathcal{E}(\pi)] - \mathcal{E}^* \leq 4 \inf_{\delta, \tau, \xi^2, b}\Bigg\{\E_{(\Bar{w}, \Bar{\mu}, \Bar{\sigma}^2)\sim q_{\delta, \tau, \xi^2, b}}\E_{P_{T+1}}\Bigg[ \\
    \inf_{w}\left\{ \frac{\KL(w\Vert \Bar{w})}{\alpha n} + \frac{d}{2\alpha n}\sum_{k=1}^K w_k \log \left(2\alpha n L \Bar{\sigma}_k^2 + 1\right) +  \frac{1}{2\alpha n} \sum_{k=1}^K w_k\frac{\Vert\mu_{P_{T+1}} - \Bar{\mu}_{k}\Vert^2}{\Bar{\sigma}_k^2} \right\}\Bigg] \\
    + \frac{1}{2\beta T} \log \frac{\Gamma(1^\top \delta)}{\Gamma(K) \times \prod_{k=1}^K \Gamma(\delta_k)} + \frac{1}{2\beta T} \sum_{k=1}^K (\delta_k - 1) \left(\psi(\delta_k) - \psi(1^\top \delta)\right) \\
    + \frac{1}{4\beta T} \sum_{k, i} \left(\frac{\tau_{k, i}^2}{\Bar{\Bar{\xi}}_k^2} + \frac{\xi_k^2}{\Bar{\Bar{\xi}}_k^2} - 1 + \log \frac{\Bar{\Bar{\xi}}_k^2}{\xi_k^2}\right) + \frac{1}{\beta T}\sum_{k=1}^K \left(\log \frac{b_k}{\Bar{\Bar{b}}_k} + \frac{\Bar{\Bar{b}}_k - b_k}{b_k}\right)\Bigg\}.
\end{multline*}

\noindent We are going to restrict the infimum $\inf_{w}$ to the set of $w$ such that $w_k\in \{0, 1\}$ for any $k\in [K]$. In other words, we are selecting only the best component of the mixture in the optimization bound. The reader can check that this is actually the exact solution to the minimization problem in the above bound. As a result of this minimization, the bound becomes
\begin{multline*}
    \E_{P_1, \dots, P_T}\E_{S_1, \dots, S_T}\E_{\pi\sim \hat{\Pi}}[\mathcal{E}(\pi)] - \mathcal{E}^* \leq 4 \inf_{\delta, \tau, \xi^2, b}\Bigg\{\E_{(\Bar{w}, \Bar{\mu}, \Bar{\sigma}^2)\sim q_{\delta, \tau, \xi^2, b}}\E_{P_{T+1}}\Bigg[ \\
    \min_{k\in [K]}\left\{ \frac{1}{\alpha n}\log \frac{1}{\Bar{w}_k} + \frac{d}{2\alpha n}\log \left(2\alpha n L \Bar{\sigma}_k^2 + 1\right) +  \frac{1}{2\alpha n} \frac{\Vert\mu_{P_{T+1}} - \Bar{\mu}_{k}\Vert^2}{\Bar{\sigma}_k^2} \right\}\Bigg] \\
    + \frac{1}{2\beta T} \log \frac{\Gamma(1^\top \delta)}{\Gamma(K) \times \prod_{k=1}^K \Gamma(\delta_k)} + \frac{1}{2\beta T} \sum_{k=1}^K (\delta_k - 1) \left(\psi(\delta_k) - \psi(1^\top \delta)\right) \\
    + \frac{1}{4\beta T} \sum_{k, i} \left(\frac{\tau_{k, i}^2}{\Bar{\Bar{\xi}}_k^2} + \frac{\xi_k^2}{\Bar{\Bar{\xi}}_k^2} - 1 + \log \frac{\Bar{\Bar{\xi}}_k^2}{\xi_k^2}\right) + \frac{1}{\beta T}\sum_{k=1}^K \left(\log \frac{b_k}{\Bar{\Bar{b}}_k} + \frac{\Bar{\Bar{b}}_k - b_k}{b_k}\right)\Bigg\}.
\end{multline*}

\noindent Please note that the term inside the expectation is, up to the minimum on $k\in [K]$, identical to the one we had in the case of one single Gaussian mixture, except for the term $\frac{1}{\alpha n}\log \frac{1}{\Bar{w}_k}$, which may be seen as a penalty for the choice of the component $k\in [K]$ in the mixture. We then bound the expectation term in the above bound by first using Fubini's theorem, and then inverting the minimum and the second expectation:
\begin{multline}
    \E_{P_1, \dots, P_T}\E_{S_1, \dots, S_T}\E_{\pi\sim \hat{\Pi}}[\mathcal{E}(\pi)] - \mathcal{E}^* \leq 4 \inf_{\delta, \tau, \xi^2, b}\Bigg\{\E_{P_{T+1}}\Bigg[ \min_{k\in [K]}\Bigg\{ \\
    \E_{(\Bar{w}, \Bar{\mu}, \Bar{\sigma}^2)\sim q_{\delta, \tau, \xi^2, b}}\left[\frac{1}{\alpha n}\log \frac{1}{\Bar{w}_k} + \frac{d}{2\alpha n}\log \left(2\alpha n L \Bar{\sigma}_k^2 + 1\right) +  \frac{1}{2\alpha n} \frac{\Vert\mu_{P_{T+1}} - \Bar{\mu}_{k}\Vert^2}{\Bar{\sigma}_k^2} \right]\Bigg\}\Bigg] \\
    + \frac{1}{2\beta T} \log \frac{\Gamma(1^\top \delta)}{\Gamma(K) \times \prod_{k=1}^K \Gamma(\delta_k)} + \frac{1}{2\beta T} \sum_{k=1}^K (\delta_k - 1) \left(\psi(\delta_k) - \psi(1^\top \delta)\right) \\
    + \frac{1}{4\beta T} \sum_{k, i} \left(\frac{\tau_{k, i}^2}{\Bar{\Bar{\xi}}_k^2} + \frac{\xi_k^2}{\Bar{\Bar{\xi}}_k^2} - 1 + \log \frac{\Bar{\Bar{\xi}}_k^2}{\xi_k^2}\right) + \frac{1}{\beta T}\sum_{k=1}^K \left(\log \frac{b_k}{\Bar{\Bar{b}}_k} + \frac{\Bar{\Bar{b}}_k - b_k}{b_k}\right)\Bigg\}. \label{main_bound}
\end{multline}

\noindent We can then bound the expectation term, which we decompose as
\begin{multline*}
    \E_{(\Bar{w}, \Bar{\mu}, \Bar{\sigma}^2)\sim q_{\delta, \tau, \xi^2, b}}\left[\frac{1}{\alpha n}\log \frac{1}{\Bar{w}_k} + \frac{d}{2\alpha n}\log \left(2\alpha n L \Bar{\sigma}_k^2 + 1\right) +  \frac{1}{2\alpha n} \frac{\Vert\mu_{P_{T+1}} - \Bar{\mu}_{k}\Vert^2}{\Bar{\sigma}_k^2} \right] \\
    = \frac{1}{\alpha n} \E_{\Bar{w}\sim \Dir(\delta)}\left[\log \frac{1}{\Bar{w}_k}\right] + \frac{d}{2\alpha n} \E_{\Bar{\sigma}_k^2 \sim \Gamma(2, b_k)}\left[\log \left(2\alpha n L \Bar{\sigma}_k^2 + 1\right)\right] \\
    + \frac{1}{2\alpha n} \E_{(\Bar{w}, \Bar{\mu}, \Bar{\sigma}^2)\sim q_{\delta, \tau, \xi^2, b}}\left[\frac{\Vert\mu_{P_{T+1}} - \Bar{\mu}_{k}\Vert^2}{\Bar{\sigma}_k^2}\right].
\end{multline*}

\noindent Jensen's inequality helps to bound both the first term
\begin{align*}
    \frac{1}{\alpha n} \E_{\Bar{w}\sim \Dir(\delta)}\left[\log \frac{1}{\Bar{w}_k}\right] &\leq \frac{1}{\alpha n} \log \E_{\Bar{w}\sim \Dir(\delta)}\left[\frac{1}{\Bar{w}_k}\right] \\
    &= \frac{1}{\alpha n} \log \frac{1^\top \delta -1}{\delta_k-1}
\end{align*}
and the second term
\begin{align*}
    \frac{d}{2\alpha n} \E_{\Bar{\sigma}_k^2 \sim \Gamma(2, b_k)}\left[\log \left(2\alpha n L \Bar{\sigma}_k^2 + 1\right)\right] &\leq \frac{d}{2\alpha n} \log \left(2\alpha n L  \E_{\Bar{\sigma}_k^2 \sim \Gamma(2, b_k)}\left[\Bar{\sigma}_k^2 \right] + 1\right) \\ 
    &= \frac{d}{2\alpha n} \log \left(\frac{4 L\alpha n}{b_k} + 1\right)
\end{align*}
in the decomposition. The third term can be bounded as follows
\begin{align*}
    &\frac{1}{2\alpha n} \E_{(\Bar{w}, \Bar{\mu}, \Bar{\sigma}^2)\sim q_{\delta, \tau, \xi^2, b}}\left[\frac{\Vert\mu_{P_{T+1}} - \Bar{\mu}_{k}\Vert^2}{\Bar{\sigma}_k^2}\right] \\
    &= \frac{b_k}{2\alpha n} \E_{\Bar{\mu}_k \sim \N(\tau_k, \xi_k^2 I_d)}\left[\Vert\mu_{P_{T+1}} - \Bar{\mu}_{k}\Vert^2\right] \\
    &\leq \frac{b_k}{\alpha n}\left(\Vert\mu_{P_{T+1}} - \tau_{k}\Vert^2 + \E_{\Bar{\mu}_k \sim \N(\tau_k, \xi_k^2 I_d)}\left[\Vert\tau_k - \Bar{\mu}_{k}\Vert^2\right]\right) \\
    &= \frac{b_k}{\alpha n}\left(\Vert\mu_{P_{T+1}} - \tau_{k}\Vert^2 + d\xi_k^2\right).
\end{align*}

\noindent The bound on the expectation then becomes
\begin{multline*}
    \E_{(\Bar{w}, \Bar{\mu}, \Bar{\sigma}^2)\sim q_{\delta, \tau, \xi^2, a, b}}\left[\frac{1}{\alpha n}\log \frac{1}{\Bar{w}_k} + \frac{d}{2\alpha n}\log \left(2\alpha n L \Bar{\sigma}_k^2 + 1\right) +  \frac{1}{2\alpha n} \frac{\Vert\mu_{P_{T+1}} - \Bar{\mu}_{k}\Vert^2}{\Bar{\sigma}_k^2} \right] \\
    \leq \frac{1}{\alpha n} \log \frac{1^\top \delta -1}{\delta_k-1} + \frac{d}{2\alpha n} \log \left(\frac{4 L\alpha n}{b_k} + 1\right) + \frac{b_k}{\alpha n}\left(\Vert\mu_{P_{T+1}} - \tau_{k}\Vert^2 + d\xi_k^2\right).
\end{multline*}

\noindent In our final bound, we wish to have as few terms as possible in $O\left(\frac{1}{n}\right)$ while the terms in $O\left(\frac{1}{T}\right)$ are not so problematic, because they correspond to the fast convergence rate at the meta-level. For this reason, we are going to take out of the infimum: 
\begin{itemize}
    \item the term $\frac{d}{2\alpha n} \log \left(\frac{4 L\alpha n}{b_k} + 1\right)$, which is unavoidable and corresponds to the main term of the bound in the worst case, with a $O\left(\frac{1}{n}\right)$ speed of convergence;
    \item the term $\frac{b_k d \xi_k^2}{\alpha n}$, which will be handled through an optimization in $\xi_k^2$ and will be a $O\left(\frac{1}{T}\right)$ term.
\end{itemize}

\noindent As a consequence, we bound the minimum on $[K]$ by
\begin{multline}
    \min_{k\in [K]}\left\{\E_{(\Bar{w}, \Bar{\mu}, \Bar{\sigma}^2)\sim q_{\delta, \tau, \xi^2, b}}\left[\frac{1}{\alpha n}\log \frac{1}{\Bar{w}_k} + \frac{d}{2\alpha n}\log \left(2\alpha n L \Bar{\sigma}_k^2 + 1\right) +  \frac{1}{2\alpha n} \frac{\Vert\mu_{P_{T+1}} - \Bar{\mu}_{k}\Vert^2}{\Bar{\sigma}_k^2} \right]\right\} \leq \\
    \frac{d}{2\alpha n} \sum_{k=1}^K \log \left(\frac{4 \alpha L n}{b_k} + 1\right) + \sum_{k=1}^K \frac{b_k d\xi_k^2}{\alpha n} + \frac{1}{\alpha n}\min_{k\in [K]} \left\{b_k\Vert\mu_{P_{T+1}} - \tau_{k}\Vert^2 + \log \frac{1^\top \delta -1}{\delta_k-1} \right\}, \label{minimum_first_bound}
\end{multline}
and plugging this result in \eqref{main_bound} gives
\begin{multline*}
    \E_{P_1, \dots, P_T}\E_{S_1, \dots, S_T}\E_{\pi\sim \hat{\Pi}}[\mathcal{E}(\pi)] - \mathcal{E}^* \leq 4 \inf_{\delta, \tau, \xi^2, b}\Bigg\{\frac{d}{2\alpha n} \sum_{k=1}^K \log \left(\frac{4 \alpha L n}{b_k} + 1\right) + \sum_{k=1}^K \frac{b_k d\xi_k^2}{\alpha n} \\
    + \frac{1}{\alpha n}\E_{P_{T+1}}\left[\min_{k\in [K]}\left\{b_k \Vert\mu_{P_{T+1}} - \tau_{k}\Vert^2 + \log \frac{1^\top \delta -1}{\delta_k-1}\right\}\right] \\
    + \frac{1}{2\beta T} \log \frac{\Gamma(1^\top \delta)}{\Gamma(K) \times \prod_{k=1}^K \Gamma(\delta_k)} + \frac{1}{2\beta T}\sum_{k=1}^K \left(\delta_k-1\right) \left(\psi(\delta_k) - \psi(1^\top \delta)\right) \\
    + \frac{1}{4\beta T} \sum_{k=1}^K \frac{\Vert\tau_k\Vert^2}{\Bar{\Bar{\xi}}_k^2} + \frac{d}{4\beta T} \sum_{k=1}^K \left(\frac{\xi_k^2}{\Bar{\Bar{\xi}}_k^2} - 1 + \log \frac{\Bar{\Bar{\xi}}_k^2}{\xi_k^2}\right) + \frac{1}{\beta T}\sum_{k=1}^K \left(\log \frac{b_k}{\Bar{\Bar{b}}_k} + \frac{\Bar{\Bar{b}}_k - b_k}{b_k}\right)\Bigg\}. 
\end{multline*}

\noindent An exact optimization in $\xi_k^2$ gives
\begin{equation*}
    \xi_k^2 = \frac{\Bar{\Bar{\xi}}_k^2}{1 + \frac{4b_k \Bar{\Bar{\xi}}_k^2 \beta T}{\alpha n}}
\end{equation*}
and replacing in the bound yields
\begin{multline*}
    \E_{P_1, \dots, P_T}\E_{S_1, \dots, S_T}\E_{\pi\sim \hat{\Pi}}[\mathcal{E}(\pi)] - \mathcal{E}^* \leq 4 \inf_{\delta, \tau, b}\Bigg\{\frac{d}{2\alpha n} \sum_{k=1}^K \log \left(\frac{4 \alpha L n}{b_k} + 1\right) \\
    + \frac{1}{\alpha n}\E_{P_{T+1}}\left[\min_{k\in [K]}\left\{b_k \Vert\mu_{P_{T+1}} - \tau_{k}\Vert^2 + \log \frac{1^\top \delta -1}{\delta_k-1}\right\}\right] \\
    + \frac{1}{2\beta T} \log \frac{\Gamma(1^\top \delta)}{\Gamma(K) \times \prod_{k=1}^K \Gamma(\delta_k)} + \frac{1}{2\beta T} \sum_{k=1}^K \left(\delta_k-1\right) \left(\psi(\delta_k) - \psi(1^\top \delta)\right) \\
    + \frac{1}{4\beta T} \sum_{k=1}^K \frac{\Vert\tau_k\Vert^2}{\Bar{\Bar{\xi}}_k^2} + \frac{d}{4\beta T} \sum_{k=1}^K \log \left(1 + \frac{4b_k \Bar{\Bar{\xi}}_k^2 \beta T}{\alpha n}\right) + \frac{1}{\beta T}\sum_{k=1}^K \left(\log \frac{b_k}{\Bar{\Bar{b}}_k} + \frac{\Bar{\Bar{b}}_k - b_k}{b_k}\right)\Bigg\}. 
\end{multline*}

\noindent From here, we set $\delta_k = 2$ for any $k\in [K]$, which implies 
\begin{equation*}
    \frac{1}{2\beta T} \sum_{k=1}^K \left(\delta_k-1\right) \left(\psi(\delta_k) - \psi(1^\top \delta)\right) \leq 0
\end{equation*}
because $\psi$ is increasing. Please also note that
\begin{equation*}
    \log \frac{\Gamma(1^\top \delta)}{\Gamma(K) \times \prod_{k=1}^K \Gamma(\delta_k)} = \log \frac{\Gamma(2K)}{\Gamma(K)} \leq K \log (2K).
\end{equation*}

\noindent We can then deduce the bound
\begin{multline*}
    \E_{P_1, \dots, P_T}\E_{S_1, \dots, S_T}\E_{\pi\sim \hat{\Pi}}[\mathcal{E}(\pi)] - \mathcal{E}^* \leq 4 \inf_{\tau, b}\Bigg\{\frac{d}{2\alpha n} \sum_{k=1}^K \log \left(\frac{4 \alpha L n}{b_k} + 1\right) + \frac{\log (2K)}{\alpha n} \\
    + \frac{1}{\alpha n} \E_{P_{T+1}}\left[\min_{k\in [K]} \left\{b_k \Vert\mu_{P_{T+1}} - \tau_{k}\Vert^2 \right\}\right] + \frac{K\log (2K)}{2\beta T} + \frac{1}{4\beta T} \sum_{k=1}^K \frac{\Vert\tau_k\Vert^2}{\Bar{\Bar{\xi}}_k^2} \\
    + \frac{d}{4\beta T} \sum_{k=1}^K \log \left(1 + \frac{4b_k \Bar{\Bar{\xi}}_k^2 \beta T}{\alpha n}\right) + \frac{1}{\beta T}\sum_{k=1}^K \left(\log \frac{b_k}{\Bar{\Bar{b}}_k} + \frac{\Bar{\Bar{b}}_k - b_k}{b_k}\right)\Bigg\}.
\end{multline*}

\noindent Let 
\begin{equation*}
    \Sigma_K (\mathcal{P}) := \inf_{\tau_1, \dots, \tau_K}\E_{P_{T+1}\sim \mathcal{P}}\left[\min_{k\in [K]} \Vert\mu_{P_{T+1}} - \tau_{k}\Vert^2 \right],
\end{equation*}
it is clear that
\begin{equation*}
    \E_{P_{T+1}}\left[\min_{k\in [K]} \left\{b_k \Vert\mu_{P_{T+1}} - \tau_{k}\Vert^2 \right\}\right] \leq \Sigma_K(\mathcal{P})\sum_{k=1}^K b_k.
\end{equation*}

\noindent By choosing $\tau_1, \dots, \tau_K$ minimizing $\Sigma_K (\mathcal{P})$, the previous bound becomes
\begin{multline}
    \E_{P_1, \dots, P_T}\E_{S_1, \dots, S_T}\E_{\pi\sim \hat{\Pi}}[\mathcal{E}(\pi)] - \mathcal{E}^* \leq 4 \inf_{b}\Bigg\{\frac{\log (2K)}{\alpha n} + \frac{d}{2\alpha n} \sum_{k=1}^K \log \left(\frac{4 \alpha L n}{b_k} + 1\right) \\
    + \frac{\Sigma_K(\mathcal{P})}{\alpha n}\sum_{k=1}^K b_k + \frac{K\log (2K)}{2\beta T} + \frac{1}{4\beta T} \sum_{k=1}^K \frac{\Vert\tau_k\Vert^2}{\Bar{\Bar{\xi}}_k^2} \\
    + \frac{d}{4\beta T} \sum_{k=1}^K \log \left(1 + \frac{4b_k \Bar{\Bar{\xi}}_k^2 \beta T}{\alpha n}\right) + \frac{1}{\beta T}\sum_{k=1}^K \left(\log \frac{b_k}{\Bar{\Bar{b}}_k} + \frac{\Bar{\Bar{b}}_k - b_k}{b_k}\right)\Bigg\}. \label{bound_case_split}
\end{multline}

\noindent Please note that $\tau_1, \dots, \tau_K$ are characteristic of the distribution $\mathcal{P}$. Intuitively, if the distribution $\mathcal{P}$ has $K$ modes, or $K$ Gaussian mixtures, $\tau_1, \dots, \tau_K$ correspond to the centers of these modes or mixtures up to a permutation. Consequently, they do not scale with $n$ or $T$, but can be regarded as problem parameters of constant order.

Let $\epsilon>0$ which we will specify later, we are going to consider two separate cases:
\begin{itemize}
    \item[-] either $\Sigma_K (\mathcal{P})\leq d\epsilon$, implying that the distribution is concentrated around $\tau_1, \dots, \tau_K$ and the variance of the local distribution around each of those points is smaller than $\epsilon$;
    \item[-] or $\Sigma_K (\mathcal{P})>d\epsilon$.
\end{itemize}

\subsection{First Case: $\Sigma_K(\mathcal{P})\leq d\epsilon$}

In this case, we expect the distribution to well concentrated around $\tau_1, \dots, \tau_K$ (which are the centers of the mixtures). As a result, the optimal parameter in the new task $T+1$ is going to be closed to one of $\tau_1, \dots, \tau_K$ and we can infer it from the previous tasks. This is precisely the case where we expect a significant improvement from the meta-learning over the learning in isolation.

A closer look at the bound above shows that, besides the term $\frac{\log (2K)}{\alpha n}$ (for which we will provide an interpretation later in this section), all the terms are $O\left(\frac{\log T}{T}\right)$, which is the fast rate of convergence at the meta level, except possibly for the terms $\frac{d}{2\alpha n} \sum_{k=1}^K \log \left(\frac{4 \alpha L n}{b_k} + 1\right)$ and $\frac{\Sigma_K(\mathcal{P})}{\alpha n}\sum_{k=1}^K b_k$. Nevertheless, the assumption on $\Sigma_K(\mathcal{P})$ shows that the second of those terms is very small (of order $O\left(\frac{\epsilon}{n}\right)$). For this reason, we are going to choose $b_k$ small enough so that the first term becomes small, and we set
\begin{equation*}
    \forall k\in [K], \ b_k = T.
\end{equation*}

\noindent Replacing in the bound gives
\begin{multline*}
    \E_{P_1, \dots, P_T}\E_{S_1, \dots, S_T}\E_{\pi\sim \hat{\Pi}}[\mathcal{E}(\pi)] - \mathcal{E}^* \leq 4 \Bigg\{\frac{\log (2K)}{\alpha n} + \frac{d}{2\alpha n} \sum_{k=1}^K \log \left(\frac{4 \alpha L n}{T} + 1\right) \\
    + \frac{TK\Sigma_K(\mathcal{P})}{\alpha n} + \frac{K\log (2K)}{2\beta T} + \frac{1}{4\beta T} \sum_{k=1}^K \frac{\Vert\tau_k\Vert^2}{\Bar{\Bar{\xi}}_k^2} \\
    + \frac{d}{4\beta T} \sum_{k=1}^K \log \left(1 + \frac{4\Bar{\Bar{\xi}}_k^2 \beta T^2}{\alpha n}\right) + \frac{1}{\beta T} \sum_{k=1}^K \left(\log \frac{T}{\Bar{\Bar{b}}_k} + \frac{\Bar{\Bar{b}}_k - T}{T}\right)\Bigg\}.
\end{multline*}

\noindent We can easily bound the second term in the bound using the inequality $\log (1+x) \leq x$, which becomes
\begin{equation*}
    \frac{d}{2\alpha n} \sum_{k=1}^K \log \left(\frac{4\alpha L n}{T} + 1\right) \leq \frac{2LdK}{T},
\end{equation*}
and replacing $\Sigma_K(\mathcal{P})$ by its upper bound $d\epsilon$ gives
\begin{multline*}
    \E_{P_1, \dots, P_T}\E_{S_1, \dots, S_T}\E_{\pi\sim \hat{\Pi}}[\mathcal{E}(\pi)] - \mathcal{E}^* \leq \frac{4\log (2K)}{\alpha n} + \frac{8LdK}{T} + \frac{4TK\epsilon}{\alpha n} \\
    + \frac{2K\log (2K)}{\beta T} + \frac{1}{\beta T} \sum_{k=1}^K \frac{\Vert\tau_k\Vert^2}{\Bar{\Bar{\xi}}_k^2} + \frac{d}{\beta T} \sum_{k=1}^K \log \left(1 + \frac{4 \Bar{\Bar{\xi}}_k^2 \beta T^2}{\alpha n}\right) + \frac{4}{\beta T} \sum_{k=1}^K \left(\log \frac{T}{\Bar{\Bar{b}}_k} + \frac{\Bar{\Bar{b}}_k - T}{T}\right).
\end{multline*}

\noindent From here, the choice of the rate $\epsilon = \frac{n}{T^2}$ yields the final bound:
\begin{multline*}
    \E_{P_1, \dots, P_T}\E_{S_1, \dots, S_T}\E_{\pi\sim \hat{\Pi}}[\mathcal{E}(\pi)] - \mathcal{E}^* \leq \\
    \text{CV}_{\text{finite}}(K, n) + \text{CV}_{\text{Gaussian}}(K, d, T) + \text{CV}_{\text{meta}}(T, n, d, K, \Bar{\Bar{b}}, \Bar{\Bar{\xi}}^2, \tau),
\end{multline*}
where we denoted
\begin{equation*}
    \text{CV}_{\text{finite}}(K, n) = \frac{4\log (2K)}{\alpha n}, \quad \text{CV}_{\text{Gaussian}}(K, d, T) = \frac{8LdK}{T} + \frac{4K}{\alpha T}
\end{equation*}
and
\begin{multline*}
    \text{CV}_{\text{meta}}(T, n, d, K, \Bar{\Bar{b}}, \Bar{\Bar{\xi}}^2, \tau) = \frac{2K\log (2K)}{\beta T} + \frac{1}{\beta T} \sum_{k=1}^K \frac{\Vert\tau_k\Vert^2}{\Bar{\Bar{\xi}}_k^2} + \frac{d}{\beta T} \sum_{k=1}^K \log \left(1 + \frac{4 \Bar{\Bar{\xi}}_k^2 \beta T^2}{\alpha n}\right) \\
    + \frac{4}{\beta T} \sum_{k=1}^K \left(\log \frac{T}{\Bar{\Bar{b}}_k} + \frac{\Bar{\Bar{b}}_k - T}{T}\right).
\end{multline*}

\begin{remark}
Please note that 
\begin{equation*}
    \text{CV}_{\text{finite}}(K, n) = O\left(\frac{\log K}{n}\right)
\end{equation*}
and
\begin{equation*}
    \text{CV}_{\text{Gaussian}}(K, d, T) + \text{CV}_{\text{meta}}(T, n, d, K, \Bar{\Bar{b}}, \Bar{\Bar{\xi}}^2, \tau) = O\left(\frac{dK\log T}{T}\right).
\end{equation*}

\noindent This bound comes as no surprise, because the process of learning the parameter $\theta$ in the mixture of Gaussians framework consists of two different steps:
\begin{itemize}
    \item[-] first identifying the $K$ centers of the mixtures which, similarly to the finite case, is captured in the $\frac{\log K}{n}$ term;
    \item[-] then, identifying the right parameters of the Gaussian components centered on the points identified in the previous step. Similarly to the Gaussian case, this is captured in $O\left(\frac{\log T}{T}\right)$, in the most favorable case.
\end{itemize}
\end{remark}

\begin{remark}
It may seem paradoxical that the model of mixtures of Gaussians achieves, in the best possible case, a rate of convergence $O\left(\frac{\log K}{n}\right)$ slower than the $O\left(\frac{\log T}{T}\right)$ rate achieved by a single Gaussian in the regime $n<<T$ under optimal conditions. In reality, the latter rate of convergence is also achievable in the model of mixtures of Gaussians, but similarly as in the case of a single Gaussian component, it requires the strong assumption that
\begin{equation*}
    \Sigma_1(\mathcal{P}) \leq d\epsilon,
\end{equation*}
which is much more restrictive. On the other hand, many distributions only satisfy 
\begin{equation*}
    \Sigma_K(\mathcal{P}) \leq d\epsilon
\end{equation*}
for some $K\geq 2$, in which case the rate of convergence $O\left(\frac{\log K}{n}\right)$ achieved here is much faster than $O\left(\frac{d\log n}{n}\right)$, which is the best possible rate achieved in the single Gaussian model in general.
\end{remark}

\subsection{Second Case: $\Sigma_K(\mathcal{P})>d\epsilon$}

In this case, $\Sigma_K(\mathcal{P})$ is not smaller than $d\epsilon$ and therefore, the choice of a large parameter $b_k$ would make the term $\frac{\Sigma_K(\mathcal{P})}{\alpha n} \sum_{k=1}^K b_k$ too large to achieve the desired rate of convergence. As a result, we set
\begin{equation*}
    \forall k\in [K], \ b_k = 1.
\end{equation*}

\noindent Replacing in the bound \eqref{bound_case_split} gives 
\begin{multline*}
    \E_{P_1, \dots, P_T}\E_{S_1, \dots, S_T}\E_{\pi\sim \hat{\Pi}}[\mathcal{E}(\pi)] - \mathcal{E}^* \leq \frac{4\log (2K)}{\alpha n} + \frac{2dK}{\alpha n} \log \left(1 + 4 \alpha L n\right) + \frac{4K\Sigma_K(\mathcal{P})}{\alpha n} \\
    + \frac{2K\log (2K)}{\beta T} + \frac{1}{\beta T} \sum_{k=1}^K \frac{\Vert\tau_k\Vert^2}{\Bar{\Bar{\xi}}_k^2} + \frac{d}{\beta T} \sum_{k=1}^K \log \left(1 + \frac{4 \Bar{\Bar{\xi}}_k^2 \beta T^2}{\alpha n}\right) + \frac{4}{\beta T}\sum_{k=1}^K \left(\log \frac{1}{\Bar{\Bar{b}}_k} + \Bar{\Bar{b}}_k - 1\right). 
\end{multline*}

\subsection{In Summary}

In any case, the bound can be written as
\begin{multline*}
    \E_{P_1, \dots, P_T}\E_{S_1, \dots, S_T}\E_{\pi\sim \hat{\Pi}}[\mathcal{E}(\pi)] - \mathcal{E}^* \leq \text{CV}_{\text{finite}}(K, n) + K\times \text{CV}_{\text{Gaussian}}\left(d, \Sigma_K(\mathcal{P}), n, T\right) \\
    + \text{CV}_{\text{meta}}(T, n, d, K, \Bar{\Bar{b}}, \Bar{\Bar{\xi}}^2, \tau),
\end{multline*}
where
\begin{equation*}
    \text{CV}_{\text{finite}}(K, n) = \frac{4\log (2K)}{\alpha n}
\end{equation*}
is the convergence term in the finite case (to find the centers of the mixtures), and
\begin{equation*}
    \text{CV}_{\text{Gaussian}}\left(d, K, \Sigma_K(\mathcal{P}), n, T\right) = \left\{
    \begin{array}{ll}
        \frac{8Ld}{T} + \frac{4}{\alpha T} \quad &\mbox{if } \ \Sigma_K(\mathcal{P}) \leq \frac{n}{T^2}; \\
        \frac{2d}{\alpha n} \log \left(1 + 4\alpha L n\right) + \frac{4\Sigma_K(\mathcal{P})}{\alpha n} \quad &\mbox{otherwise;}
    \end{array}
    \right.
\end{equation*}
is equal to the convergence term in the Gaussian case (with one component) and is a $O\left(\frac{1}{T}\right)$ if there exist $K$ points such that distribution $\mathcal{P}$ is concentrated at a rate $\epsilon = \frac{n}{T^2}$ around those $K$ points, and $O\left(\frac{d\log n}{n}\right)$ otherwise. The remaining term is
\begin{multline*}
    \text{CV}_{\text{meta}}(T, n, d, K, \Bar{\Bar{b}}, \Bar{\Bar{\xi}}^2, \tau) = \frac{2K\log (2K)}{\beta T} + \frac{1}{\beta T} \sum_{k=1}^K \frac{\Vert\tau_k\Vert^2}{\Bar{\Bar{\xi}}_k^2} + \frac{d}{\beta T} \sum_{k=1}^K \log \left(1 + \frac{4 \Bar{\Bar{\xi}}_k^2 \beta T^2}{\alpha n}\right) \\
    + \frac{4}{\beta T}\sum_{k=1}^K \left(\log \frac{T}{\Bar{\Bar{b}}_k} + \Bar{\Bar{b}}_k - 1\right)
\end{multline*}
and it is the convergence term term at the meta level. Please note that it is a $O\left(\frac{dK\log T}{T}\right)$.

\subsection{What if the number of components in the mixture is unknown?}

In practice, we do not know in advance how to chose the number of components $K$ in the prior. In this case, we are going to include inside $\mathcal{M}$ all the mixtures of Gaussians, i.e.,
\begin{equation*}
    \mathcal{M} = \left\{p_{w, \mu, \sigma^2} = \sum_{k=1}^{+\infty} w_k \bigotimes_{i=1}^d \N(\mu_{k, i}, \sigma^2_{k, i}): \exists K\geq 1: \forall k\geq K+1, w_k = 0\right\}.
\end{equation*}

\noindent Note that the sum inside the definition of $\F$ is finite, since $w_k = 0$ for any $k$ beyond a certain rank $K$. We still denote by $\pi = \sum_{k=1}^{\Bar{K}} \Bar{w}_k\N(\Bar{\mu}_k, \Bar{\sigma}_k^2 I_d)$ the prior in each task. By definition, for any $k\geq \Bar{K}+1, \Bar{w}_k = 0$. It still holds that, for any $w, \mu, \sigma^2$,
\begin{equation*}
    \KL(p_{w, \mu, \sigma^2} \Vert \pi) \leq \KL(w\Vert \Bar{w}) + \frac{1}{2} \sum_{k=1}^{\infty} w_k \sum_{i=1}^d \left(\frac{(\mu_{k, i} - \Bar{\mu}_{k, i})^2}{\Bar{\sigma}_k^2} + \frac{\sigma_{k, i}^2}{\Bar{\sigma}_k^2} - 1 + \log \frac{\Bar{\sigma}_k^2}{\sigma_{k, i}^2}\right),
\end{equation*}
where we denoted
\begin{equation*}
    \KL(w\Vert \Bar{w}) = \sum_{k=1}^{\infty} w_k \log \frac{w_k}{\Bar{w}_k}.
\end{equation*}

\noindent To put things clearly, the KL remains identical to the case where $K$ is known except for the fact that the sums on $k$ are no longer stopping at a pre-determined $K$. This difference aside, the bound remains identical to the one in the case where $K$ is known, and the bound from Theorem~\ref{theorem_bound_isolation} becomes, at $t = T+1$,
\begin{multline*}
    \E_{S_{T+1}}\E_{\theta\sim \rho_{T+1}(\pi, \alpha)}[R_{P_{T+1}}(\theta)] - R^*_{P_{T+1}} \leq 2\inf_{w, \mu, \sigma^2} \Bigg\{\E_{\theta\sim p_{w, \mu, \sigma^2}}[R_{P_{T+1}}(\theta)] - R_{P_{T+1}}^* + \frac{\KL(w\Vert \Bar{w})}{\alpha n} \\
    + \frac{1}{2\alpha n} \sum_{k=1}^{\infty} w_k \sum_{i=1}^d \left(\frac{(\mu_{k, i} - \Bar{\mu}_{k, i})^2}{\Bar{\sigma}_k^2} + \frac{\sigma_{k, i}^2}{\Bar{\sigma}_k^2} - 1 + \log \frac{\Bar{\sigma}_k^2}{\sigma_{k, i}^2}\right)\Bigg\}.
\end{multline*}

\noindent Next, we are going to define a prior on $K$ within the prior of priors as follows. We assume that the number of mixtures $K$ is smaller than $T$, because even if it were not, it would be impossible to identify them with enough confidence. We define the set of priors on priors
\begin{equation*}
    \G = \Bigg\{q_{x, \delta, \tau, \xi^2, b} = q_{x} \times q_{\delta, \tau, \xi^2, b | K} \Bigg\},
\end{equation*}
where $q_{x} = \Mult(x_1, \dots, x_T)$ is the prior distribution on $K$ and
\begin{equation*}
    q_{\delta, \tau, \xi^2, b | K} = \Dir(\delta_1, \dots, \delta_K)\otimes \bigotimes_{\substack{k\in [K] \\ i\in [d]}} \N(\tau_{k, i}, \xi_k^2) \otimes \bigotimes_{k=1}^K \Gamma(2, b_k).
\end{equation*}
and we set the prior of prior as $\Lambda = q_{\frac{1}{T}1_T, 1_K, 0, \Bar{\Bar{\xi}}^2, \Bar{\Bar{b}}}$. We also need to re-compute the KL divergence between the priors of priors, which becomes
\begin{multline*}
    \KL(q_{\lambda, \delta, \tau, \xi^2, b} \Vert \Lambda) = \log T - H(x) + \E_{K\sim \Mult(x)}\Bigg[\sum_{k=1}^K (\delta_k - 1) \left(\psi(\delta_k) - \psi(1^\top \delta)\right) \\
    + \log \frac{\Gamma(1^\top \delta)}{\Gamma(K) \times \prod_{k=1}^K \Gamma(\delta_k)} + \frac{1}{2} \sum_{k, i} \left(\frac{\tau_{k, i}^2}{\Bar{\Bar{\xi}}_k^2} + \frac{\xi_k^2}{\Bar{\Bar{\xi}}_k^2} - 1 + \log \frac{\Bar{\Bar{\xi}}_k^2}{\xi_k^2}\right) + 2\sum_{k=1}^K \left(\log \frac{b_k}{\Bar{\Bar{b}}_k} + \frac{\Bar{\Bar{b}}_k - b_k}{b_k}\right)\Bigg],
\end{multline*}
using \eqref{kl_multinomial}. Please note that in any optimization on $\G$, we optimize first in $(x_1, \dots, x_T)$ and then on $\delta, \tau, \xi^2, b$ conditionally on $K$. This means that the latter parameters are allowed to depend on $K$. While the infimum on $\G$ of any quantity should be written $\inf_{x} \inf_{\delta, \tau, \xi^2, b \in \sigma(K)}$, we will adopt the shortcut notation $\inf_{x, \delta, \tau, \xi^2, b}$. We can next use the bound from Theorem~\ref{theorem_meta_learning} and we have
\begin{multline*}
    \E_{P_1, \dots, P_T}\E_{S_1, \dots, S_T}\E_{\pi\sim \hat{\Pi}}[\mathcal{E}(\pi)] - \mathcal{E}^* \leq \\
    4 \inf_{x, \delta, \tau, \xi^2, b}\Bigg\{\E_{(\Bar{w}, \Bar{\mu}, \Bar{\sigma}^2)\sim q_{x, \delta, \tau, \xi^2, b}}\E_{P_{T+1}}\Bigg[ \inf_{w, \sigma^2}\Bigg\{L\sum_{k=1}^{\infty} w_k \Vert\sigma_k\Vert^2 + \frac{\KL(w\Vert \Bar{w})}{\alpha n} \\
    + \frac{1}{2\alpha n} \sum_{k=1}^{\infty} w_k \sum_{i=1}^d \left(\frac{(\mu_{k, i} - \Bar{\mu}_{k, i})^2}{\Bar{\sigma}_k^2} + \frac{\sigma_{k, i}^2}{\Bar{\sigma}_k^2} - 1 + \log \frac{\Bar{\sigma}_k^2}{\sigma_{k, i}^2}\right)\Bigg\}\Bigg] \\
    + \frac{1}{2\beta T} \left(\log T - H(x)\right) + \frac{1}{2\beta T}\E_{K\sim \Mult(x)}\Bigg[\sum_{k=1}^K (\delta_k - 1) \left(\psi(\delta_k) - \psi(1^\top \delta)\right) \\
    + \log \frac{\Gamma(1^\top \delta)}{\Gamma(K) \times \prod_{k=1}^K \Gamma(\delta_k)} + \frac{1}{2} \sum_{k, i} \left(\frac{\tau_{k, i}^2}{\Bar{\Bar{\xi}}_k^2} + \frac{\xi_k^2}{\Bar{\Bar{\xi}}_k^2} - 1 + \log \frac{\Bar{\Bar{\xi}}_k^2}{\xi_k^2}\right) + 2\sum_{k=1}^K \left(\log \frac{b_k}{\Bar{\Bar{b}}_k} + \frac{\Bar{\Bar{b}}_k - b_k}{b_k}\right)\Bigg]\Bigg\}.
\end{multline*}

\noindent The optimization on $\sigma_{k, i}^2$ may be performed exactly by setting $\sigma_{k, i}^2 = \frac{\Bar{\sigma}_k^2}{2\alpha L n \Bar{\sigma}_k^2 + 1}$, and the bound becomes
\begin{multline*}
    \E_{P_1, \dots, P_T}\E_{S_1, \dots, S_T}\E_{\pi\sim \hat{\Pi}}[\mathcal{E}(\pi)] - \mathcal{E}^* \leq 4 \inf_{x, \delta, \tau, \xi^2, b}\Bigg\{\E_{(\Bar{w}, \Bar{\mu}, \Bar{\sigma}^2)\sim q_{x, \delta, \tau, \xi^2, b}}\E_{P_{T+1}}\Bigg[ \\
    \inf_{w}\left\{ \frac{\KL(w\Vert \Bar{w})}{\alpha n} + \frac{d}{2\alpha n}\sum_{k=1}^{\infty} w_k \log \left(2\alpha L n \Bar{\sigma}_k^2 + 1\right) +  \frac{1}{2\alpha n} \sum_{k=1}^{\infty} w_k\frac{\Vert\mu_{P_{T+1}} - \Bar{\mu}_{k}\Vert^2}{\Bar{\sigma}_k^2} \right\}\Bigg] \\
    + \frac{1}{2\beta T} \left(\log T - H(x)\right) + \frac{1}{2\beta T}\E_{K\sim \Mult(x)}\Bigg[\sum_{k=1}^K (\delta_k - 1) \left(\psi(\delta_k) - \psi(1^\top \delta)\right) \\
    + \log \frac{\Gamma(1^\top \delta)}{\Gamma(K) \times \prod_{k=1}^K \Gamma(\delta_k)} + \frac{1}{2} \sum_{k, i} \left(\frac{\tau_{k, i}^2}{\Bar{\Bar{\xi}}_k^2} + \frac{\xi_k^2}{\Bar{\Bar{\xi}}_k^2} - 1 + \log \frac{\Bar{\Bar{\xi}}_k^2}{\xi_k^2}\right) + 2\sum_{k=1}^K \left(\log \frac{b_k}{\Bar{\Bar{b}}_k} + \frac{\Bar{\Bar{b}}_k - b_k}{b_k}\right)\Bigg]\Bigg\}.
\end{multline*}

\noindent We restrict the optimization in $w$ to the set $\{(w_k)_{k\geq 1}: \exists k_0\leq \Bar{K}: w_{k_0} = 1, \forall k\ne k_0, w_k = 0\}$, and the bound becomes
\begin{multline*}
    \E_{P_1, \dots, P_T}\E_{S_1, \dots, S_T}\E_{\pi\sim \hat{\Pi}}[\mathcal{E}(\pi)] - \mathcal{E}^* \leq 4 \inf_{x, \delta, \tau, \xi^2, b}\Bigg\{\E_{(\Bar{w}, \Bar{\mu}, \Bar{\sigma}^2)\sim q_{x, \delta, \tau, \xi^2, b}}\E_{P_{T+1}}\Bigg[ \\
    \min_{k\in [K]}\left\{ \frac{1}{\alpha n}\log \frac{1}{\Bar{w}_k} + \frac{d}{2\alpha n}\log \left(2\alpha L n \Bar{\sigma}_k^2 + 1\right) +  \frac{1}{2\alpha n} \frac{\Vert\mu_{P_{T+1}} - \Bar{\mu}_{k}\Vert^2}{\Bar{\sigma}_k^2} \right\}\Bigg] \\
    + \frac{1}{2\beta T} \left(\log T - H(x)\right) + \frac{1}{2\beta T}\E_{K\sim \Mult(x)}\Bigg[\sum_{k=1}^K (\delta_k - 1) \left(\psi(\delta_k) - \psi(1^\top \delta)\right) \\
    + \log \frac{\Gamma(1^\top \delta)}{\Gamma(K) \times \prod_{k=1}^K \Gamma(\delta_k)} + \frac{1}{2} \sum_{k, i} \left(\frac{\tau_{k, i}^2}{\Bar{\Bar{\xi}}_k^2} + \frac{\xi_k^2}{\Bar{\Bar{\xi}}_k^2} - 1 + \log \frac{\Bar{\Bar{\xi}}_k^2}{\xi_k^2}\right) + 2\sum_{k=1}^K \left(\log \frac{b_k}{\Bar{\Bar{b}}_k} + \frac{\Bar{\Bar{b}}_k - b_k}{b_k}\right)\Bigg]\Bigg\}.
\end{multline*}

\noindent Next, we classically decompose the expectation $\E_{(\Bar{w}, \Bar{\mu}, \Bar{\sigma}^2)\sim q_{x, \delta, \tau, \xi^2, b}}[X]$ as
\begin{equation*}
    \E_{(\Bar{w}, \Bar{\mu}, \Bar{\sigma}^2)\sim q_{x, \delta, \tau, \xi^2, b}}[X] = \E_{K\sim \Mult(x)}\left[\E_{(\Bar{w}, \Bar{\mu}, \Bar{\sigma}^2)\sim q_{\delta, \tau, \xi^2, b | K}}[X]\right].
\end{equation*}

\noindent Applying Fubini's theorem and inverting the infimum and the expectation yields the bound
\begin{multline*}
    \E_{P_1, \dots, P_T}\E_{S_1, \dots, S_T}\E_{\pi\sim \hat{\Pi}}[\mathcal{E}(\pi)] - \mathcal{E}^* \leq 4 \inf_{x, \delta, \tau, \xi^2, b}\Bigg\{\E_{K\sim \Mult(x)}\E_{P_{T+1}}\Bigg[ \min_{k\in [K]}\Bigg\{ \\
    \E_{(\Bar{w}, \Bar{\mu}, \Bar{\sigma}^2)\sim q_{\delta, \tau, \xi^2, b | K}}\left[\frac{1}{\alpha n}\log \frac{1}{\Bar{w}_k} + \frac{d}{2\alpha n}\log \left(2\alpha L n \Bar{\sigma}_k^2 + 1\right) +  \frac{1}{2\alpha n} \frac{\Vert\mu_{P_{T+1}} - \Bar{\mu}_{k}\Vert^2}{\Bar{\sigma}_k^2} \right] \Bigg\}\Bigg] \\
    + \frac{1}{2\beta T} \left(\log T - H(x)\right) + \frac{1}{2\beta T}\E_{K\sim \Mult(x)}\Bigg[\sum_{k=1}^K (\delta_k - 1) \left(\psi(\delta_k) - \psi(1^\top \delta)\right) \\
    + \log \frac{\Gamma(1^\top \delta)}{\Gamma(K) \times \prod_{k=1}^K \Gamma(\delta_k)} + \frac{1}{2} \sum_{k, i} \left(\frac{\tau_{k, i}^2}{\Bar{\Bar{\xi}}_k^2} + \frac{\xi_k^2}{\Bar{\Bar{\xi}}_k^2} - 1 + \log \frac{\Bar{\Bar{\xi}}_k^2}{\xi_k^2}\right) + 2\sum_{k=1}^K \left(\log \frac{b_k}{\Bar{\Bar{b}}_k} + \frac{\Bar{\Bar{b}}_k - b_k}{b_k}\right)\Bigg]\Bigg\}.
\end{multline*}

\noindent The bound \eqref{minimum_first_bound} from the previous part still holds:
\begin{multline*}
    \min_{k\in [K]}\left\{\E_{(\Bar{w}, \Bar{\mu}, \Bar{\sigma}^2)\sim q_{\delta, \tau, \xi^2, b | K}}\left[\frac{1}{\alpha n}\log \frac{1}{\Bar{w}_k} + \frac{d}{2\alpha n}\log \left(2\alpha L n \Bar{\sigma}_k^2 + 1\right) +  \frac{1}{2\alpha n} \frac{\Vert\mu_{P_{T+1}} - \Bar{\mu}_{k}\Vert^2}{\Bar{\sigma}_k^2} \right]\right\} \leq \\
    \frac{d}{2\alpha n} \sum_{k=1}^K \log \left(\frac{4 \alpha L n}{b_k} + 1\right) + \sum_{k=1}^K \frac{b_k d\xi_k^2}{\alpha n} + \frac{1}{\alpha n}\min_{k\in [K]} \left\{b_k \Vert\mu_{P_{T+1}} - \tau_{k}\Vert^2 + \log \frac{1^\top \delta -1}{\delta_k-1} \right\},
\end{multline*}
and we can inject it in the computation so that the bound becomes
\begin{multline*}
    \E_{P_1, \dots, P_T}\E_{S_1, \dots, S_T}\E_{\pi\sim \hat{\Pi}}[\mathcal{E}(\pi)] - \mathcal{E}^* \leq 4 \inf_{x, \delta, \tau, \xi^2, b}\Bigg\{ \frac{d}{2\alpha n} \E_{K\sim \Mult(x)}\left[ \sum_{k=1}^K \log \left(\frac{4 \alpha L n}{b_k} + 1\right)\right] \\
    + \frac{d}{\alpha n} \E_{K\sim \Mult(x)}\left[ \sum_{k=1}^K b_k \xi_k^2\right] + \frac{1}{\alpha n}\E_{K\sim \Mult(x)}\E_{P_{T+1}}\left[\min_{k\in [K]} \left\{b_k \Vert\mu_{P_{T+1}} - \tau_{k}\Vert^2 + \log \frac{1^\top \delta -1}{\delta_k-1} \right\}\right] \\
    + \frac{1}{2\beta T} \left(\log T - H(x)\right) + \frac{1}{2\beta T}\E_{K\sim \Mult(x)}\Bigg[\sum_{k=1}^K (\delta_k - 1) \left(\psi(\delta_k) - \psi(1^\top \delta)\right) \\
    + \log \frac{\Gamma(1^\top \delta)}{\Gamma(K) \times \prod_{k=1}^K \Gamma(\delta_k)} + \frac{1}{2} \sum_{k, i} \left(\frac{\tau_{k, i}^2}{\Bar{\Bar{\xi}}_k^2} + \frac{\xi_k^2}{\Bar{\Bar{\xi}}_k^2} - 1 + \log \frac{\Bar{\Bar{\xi}}_k^2}{\xi_k^2}\right) + 2\sum_{k=1}^K \left(\log \frac{b_k}{\Bar{\Bar{b}}_k} + \frac{\Bar{\Bar{b}}_k - b_k}{b_k}\right)\Bigg]\Bigg\}.
\end{multline*}

\noindent An exact optimization in $\xi_k^2$ yields $\xi_k^2 = \frac{\Bar{\Bar{\xi}}_k^2}{1 + \frac{4b_k \Bar{\Bar{\xi}}_k^2 \beta T}{\alpha n}}$ and we can replace in the bound
\begin{multline*}
    \E_{P_1, \dots, P_T}\E_{S_1, \dots, S_T}\E_{\pi\sim \hat{\Pi}}[\mathcal{E}(\pi)] - \mathcal{E}^* \leq 4 \inf_{x, \delta, \tau, b}\Bigg\{ \frac{d}{2\alpha n} \E_{K\sim \Mult(x)}\left[ \sum_{k=1}^K \log \left(\frac{4 \alpha L n}{b_k} + 1\right)\right] \\
    + \frac{1}{\alpha n}\E_{K\sim \Mult(x)}\E_{P_{T+1}}\left[\min_{k\in [K]} \left\{b_k \Vert\mu_{P_{T+1}} - \tau_{k}\Vert^2 + \log \frac{1^\top \delta -1}{\delta_k-1} \right\}\right] + \frac{1}{2\beta T} \left(\log T - H(x)\right) \\
    + \frac{1}{2\beta T}\E_{K\sim \Mult(x)}\Bigg[\sum_{k=1}^K (\delta_k - 1) \left(\psi(\delta_k) - \psi(1^\top \delta)\right) + \log \frac{\Gamma(1^\top \delta)}{\Gamma(K) \times \prod_{k=1}^K \Gamma(\delta_k)} \\
    + \frac{1}{2} \sum_{k=1}^K \frac{\Vert\tau_k\Vert^2}{\Bar{\Bar{\xi}}_k^2} + \frac{d}{2} \sum_{k=1}^K \log \left(1 + \frac{4b_k \Bar{\Bar{\xi}}_k^2 \beta T}{\alpha n}\right) + 2\sum_{k=1}^K \left(\log \frac{b_k}{\Bar{\Bar{b}}_k} + \frac{\Bar{\Bar{b}}_k - b_k}{b_k}\right)\Bigg] \Bigg\}.
\end{multline*}

\noindent We choose $\delta_k = 2$ for any $k\geq 1$ and noting that both
\begin{equation*}
    \sum_{k=1}^K (\delta_k - 1) \left(\psi(\delta_k) - \psi(1^\top \delta)\right) \leq 0
\end{equation*}
and
\begin{equation*}
    \log \frac{\Gamma(1^\top \delta)}{\Gamma(K) \times \prod_{k=1}^K \Gamma(\delta_k)} = \log \frac{\Gamma(2K)}{\Gamma(K)} \leq K \log (2K),
\end{equation*}
 we deduce the following bound:
\begin{multline*}
    \E_{P_1, \dots, P_T}\E_{S_1, \dots, S_T}\E_{\pi\sim \hat{\Pi}}[\mathcal{E}(\pi)] - \mathcal{E}^* \leq 4 \inf_{x, \tau, b}\Bigg\{ \frac{d}{2\alpha n} \E_{K\sim \Mult(x)}\left[ \sum_{k=1}^K \log \left(\frac{4 \alpha L n}{b_k} + 1\right)\right] \\
    + \frac{1}{\alpha n}\E_{K\sim \Mult(x)}\E_{P_{T+1}}\left[\min_{k\in [K]} \left\{b_k \Vert\mu_{P_{T+1}} - \tau_{k}\Vert^2\right\}\right] \\
    + \frac{1}{\alpha n} \E_{K\sim \Mult(x)}\left[\log(2K)\right] + \frac{1}{2\beta T} \left(\log T - H(x)\right) + \frac{1}{2\beta T}\E_{K\sim \Mult(x)}\Bigg[ K\log (2K) \\
    + \frac{1}{2} \sum_{k=1}^K \frac{\Vert\tau_k\Vert^2}{\Bar{\Bar{\xi}}_k^2} + \frac{d}{2} \sum_{k=1}^K \log \left(1 + \frac{4b_k \Bar{\Bar{\xi}}_k^2 \beta T}{\alpha n}\right) + 2\sum_{k=1}^K \left(\log \frac{b_k}{\Bar{\Bar{b}}_k} + \frac{\Bar{\Bar{b}}_k - b_k}{b_k}\right)\Bigg] \Bigg\}.
\end{multline*}

\noindent Recall the (unchanged) definition of $\Sigma_K(\mathcal{P})$:
\begin{equation*}
    \Sigma_K (\mathcal{P}) = \inf_{\tau_1, \dots, \tau_K}\E_{P_{T+1}\sim \mathcal{P}}\left[\min_{k\in [K]} \Vert\mu_{P_{T+1}} - \tau_{k}\Vert^2 \right].
\end{equation*}

\noindent Recalling that $\tau$ (as well as $b$) is allowed to depend on $K$, we define $(\tau_1, \dots, \tau_K)$ as the argument (up to a permutation) of $\Sigma_K(\mathcal{P})$. It follows that the bound becomes
\begin{multline*}
    \E_{P_1, \dots, P_T}\E_{S_1, \dots, S_T}\E_{\pi\sim \hat{\Pi}}[\mathcal{E}(\pi)] - \mathcal{E}^* \leq 4 \inf_{x, b}\Bigg\{ \frac{d}{2\alpha n} \E_{K\sim \Mult(x)}\left[ \sum_{k=1}^K \log \left(\frac{4 \alpha L n}{b_k} + 1\right)\right] \\
    + \frac{1}{\alpha n} \E_{K\sim \Mult(x)}\left[\sum_{k=1}^K b_k \Sigma_K(\mathcal{P}) + \log(2K)\right] + \frac{1}{2\beta T} \left(\log T - H(x)\right) \\
    + \frac{1}{2\beta T}\E_{K\sim \Mult(x)}\Bigg[ K\log (2K) + \frac{1}{2} \sum_{k=1}^K \frac{\Vert\tau_k\Vert^2}{\Bar{\Bar{\xi}}_k^2} + \frac{d}{2} \sum_{k=1}^K \log \left(1 + \frac{4b_k \Bar{\Bar{\xi}}_k^2 \beta T}{\alpha n}\right) \\
    + 2\sum_{k=1}^K \left(\log \frac{b_k}{\Bar{\Bar{b}}_k} + \frac{\Bar{\Bar{b}}_k - b_k}{b_k}\right)\Bigg] \Bigg\}.
\end{multline*}

\noindent From here, we are going to distinguish two different cases, as we did before:
\begin{itemize}
    \item[-] first case: there exists $K$ reasonably small such that $\Sigma_{K}(\mathcal{P})\leq d\epsilon$;
    \item[-] second case: for any reasonable $K$, $\Sigma_{K}(\mathcal{P})>d\epsilon$.
\end{itemize}

\subsubsection{First Case}

Again, this is the case where we expect some improvement from the meta-learning. Setting
\begin{equation*}
    \forall k\in [K], b_k = T,
\end{equation*}
the bound becomes
\begin{multline*}
    \E_{P_1, \dots, P_T}\E_{S_1, \dots, S_T}\E_{\pi\sim \hat{\Pi}}[\mathcal{E}(\pi)] - \mathcal{E}^* \leq 4 \inf_{x}\Bigg\{ \frac{d}{2\alpha n} \E_{K\sim \Mult(x)}\left[ \sum_{k=1}^K \log \left(\frac{4 \alpha L n}{T} + 1\right)\right] \\
    + \frac{1}{\alpha n} \E_{K\sim \Mult(x)}\left[KT\Sigma_K(\mathcal{P}) + \log(2K)\right] + \frac{1}{2\beta T} \left(\log T - H(x)\right) + \frac{1}{2\beta T}\E_{K\sim \Mult(x)}\Bigg[ K\log (2K) \\
    + \frac{1}{2} \sum_{k=1}^K \frac{\Vert\tau_k\Vert^2}{\Bar{\Bar{\xi}}_k^2} + \frac{d}{2} \sum_{k=1}^K \log \left(1 + \frac{4\Bar{\Bar{\xi}}_k^2 \beta T^2}{\alpha n}\right) + 2\sum_{k=1}^K \left(\log \frac{T}{\Bar{\Bar{b}}_k} + \frac{\Bar{\Bar{b}}_k - T}{T}\right)\Bigg] \Bigg\},
\end{multline*}
and after the simplification
\begin{equation*}
    \frac{d}{2\alpha n} \sum_{k=1}^K \log \left(\frac{4 \alpha L n}{T} + 1\right) \leq \frac{2LdK}{T},
\end{equation*}
we deduce
\begin{multline*}
    \E_{P_1, \dots, P_T}\E_{S_1, \dots, S_T}\E_{\pi\sim \hat{\Pi}}[\mathcal{E}(\pi)] - \mathcal{E}^* \leq 4 \inf_{x}\Bigg\{ \frac{2Ld}{T} \E_{K\sim \Mult(x)}\left[ K\right] + \frac{1}{\alpha n} \E_{K\sim \Mult(x)}\left[\log(2K)\right] \\
    + \frac{T}{\alpha n} \E_{K\sim \Mult(x)}\left[K\Sigma_K(\mathcal{P})\right] + \frac{1}{2\beta T} \left(\log T - H(x)\right) + \frac{1}{2\beta T}\E_{K\sim \Mult(x)}\Bigg[ K\log (2K) \\
    + \frac{1}{2} \sum_{k=1}^K \frac{\Vert\tau_k\Vert^2}{\Bar{\Bar{\xi}}_k^2} + \frac{d}{2} \sum_{k=1}^K \log \left(1 + \frac{4\Bar{\Bar{\xi}}_k^2 \beta T^2}{\alpha n}\right) + 2\sum_{k=1}^K \left(\log \frac{T}{\Bar{\Bar{b}}_k} + \frac{\Bar{\Bar{b}}_k - T}{T}\right)\Bigg] \Bigg\},
\end{multline*}

\noindent Then, choosing to restrict the infimum on all the multinomial distributions $\Mult(x_1, \dots, x_T)$ to all the Dirac masses, i.e., all the $(x_1, \dots, x_T)$ such that there exists $K\in \{1, \dots, T\}$ such that $x_K = 1$. It follows that $H(x) = 0$ and we deduce that
\begin{multline*}
    \E_{P_1, \dots, P_T}\E_{S_1, \dots, S_T}\E_{\pi\sim \hat{\Pi}}[\mathcal{E}(\pi)] - \mathcal{E}^* \leq \inf_{K\in [T]}\Bigg\{ \frac{4\log (2K)}{\alpha n} + \frac{8LdK}{T} + \frac{4TK\Sigma_K(\mathcal{P})}{\alpha n} + \frac{2\log T}{\beta T} \\
    + \frac{2}{\beta T}\Bigg[ K\log (2K) + \frac{1}{2} \sum_{k=1}^K \frac{\Vert\tau_k\Vert^2}{\Bar{\Bar{\xi}}_k^2} + \frac{d}{2} \sum_{k=1}^K \log \left(1 + \frac{4\Bar{\Bar{\xi}}_k^2 \beta T^2}{\alpha n}\right) + 2\sum_{k=1}^K \left(\log \frac{T}{\Bar{\Bar{b}}_k} + \frac{\Bar{\Bar{b}}_k - T}{T}\right)\Bigg] \Bigg\},
\end{multline*}

\noindent In particular, if there exists a $K$ relatively small such that $\Sigma_K(\mathcal{P})\leq d\epsilon = \frac{dn}{T^2}$, then the bound becomes
\begin{multline*}
    \E_{P_1, \dots, P_T}\E_{S_1, \dots, S_T}\E_{\pi\sim \hat{\Pi}}[\mathcal{E}(\pi)] - \mathcal{E}^* \leq \inf_{K\in [T]}\Bigg\{ \frac{4\log (2K)}{\alpha n} + \frac{8LdK}{T} + \frac{4dK}{\alpha T} + \frac{2\log T}{\beta T} \\
    + \frac{2}{\beta T}\Bigg[ K\log (2K) + \frac{1}{2} \sum_{k=1}^K \frac{\Vert\tau_k\Vert^2}{\Bar{\Bar{\xi}}_k^2} + \frac{d}{2} \sum_{k=1}^K \log \left(1 + \frac{4\Bar{\Bar{\xi}}_k^2 \beta T^2}{\alpha n}\right) + 2\sum_{k=1}^K \left(\log \frac{T}{\Bar{\Bar{b}}_k} + \frac{\Bar{\Bar{b}}_k - T}{T}\right)\Bigg] \Bigg\}.
\end{multline*}

\subsubsection{Second Case}

Recall the bound:
\begin{multline*}
    \E_{P_1, \dots, P_T}\E_{S_1, \dots, S_T}\E_{\pi\sim \hat{\Pi}}[\mathcal{E}(\pi)] - \mathcal{E}^* \leq 4 \inf_{x, b}\Bigg\{ \frac{d}{2\alpha n} \E_{K\sim \Mult(x)}\left[ \sum_{k=1}^K \log \left(\frac{4\alpha L n}{b_k} + 1\right)\right] \\
    + \frac{1}{\alpha n} \E_{K\sim \Mult(x)}\left[\sum_{k=1}^K b_k \Sigma_K(\mathcal{P}) + \log(2K)\right] + \frac{1}{2\beta T} \left(\log T - H(x)\right) \\
    + \frac{1}{2\beta T}\E_{K\sim \Mult(x)}\Bigg[ K\log (2K) + \frac{1}{2} \sum_{k=1}^K \frac{\Vert\tau_k\Vert^2}{\Bar{\Bar{\xi}}_k^2} + \frac{d}{2} \sum_{k=1}^K \log \left(1 + \frac{4b_k \Bar{\Bar{\xi}}_k^2 \beta T}{\alpha n}\right) \\
    + 2\sum_{k=1}^K \left(\log \frac{b_k}{\Bar{\Bar{b}}_k} + \frac{\Bar{\Bar{b}}_k - b_k}{b_k}\right)\Bigg] \Bigg\}.
\end{multline*}

\noindent In this case, we do not expect much improvement from meta-learning, and we will simply choose
\begin{equation*}
    \forall k\in \{1, \dots, T\}, \ b_k = 1.
\end{equation*}

\noindent The bound then becomes
\begin{multline*}
    \E_{P_1, \dots, P_T}\E_{S_1, \dots, S_T}\E_{\pi\sim \hat{\Pi}}[\mathcal{E}(\pi)] - \mathcal{E}^* \leq 4 \inf_{x}\Bigg\{ \frac{d\log \left(1 + 4 \alpha L n\right)}{2\alpha n} \E_{K\sim \Mult(x)}\left[ K \right] \\
    + \frac{1}{\alpha n} \E_{K\sim \Mult(x)}\left[K\Sigma_K(\mathcal{P}) + \log(2K)\right] + \frac{1}{2\beta T} \left(\log T - H(x)\right) + \frac{1}{2\beta T}\E_{K\sim \Mult(x)}\Bigg[ K\log (2K) \\
    + \frac{1}{2} \sum_{k=1}^K \frac{\Vert\tau_k\Vert^2}{\Bar{\Bar{\xi}}_k^2} + \frac{d}{2} \sum_{k=1}^K \log \left(1 + \frac{4\Bar{\Bar{\xi}}_k^2 \beta T}{\alpha n}\right) + 2\sum_{k=1}^K \left(\log \frac{1}{\Bar{\Bar{b}}_k} + \Bar{\Bar{b}}_k - 1\right)\Bigg] \Bigg\}.
\end{multline*}

\noindent Similarly as before, we restrict the minimization to the set of Dirac distributions and we deduce
\begin{multline*}
    \E_{P_1, \dots, P_T}\E_{S_1, \dots, S_T}\E_{\pi\sim \hat{\Pi}}[\mathcal{E}(\pi)] - \mathcal{E}^* \\
    \leq \inf_{K\in [K]}\Bigg\{ \frac{4\log(2K)}{\alpha n} + \frac{2dK\log \left(1 + 4 \alpha L n\right)}{\alpha n} + \frac{4K\Sigma_K(\mathcal{P})}{\alpha n} + \frac{2\log T}{\beta T} \\
    + \frac{2}{\beta T}\left[ K\log (2K) + \frac{1}{2} \sum_{k=1}^K \frac{\Vert\tau_k\Vert^2}{\Bar{\Bar{\xi}}_k^2} + \frac{d}{2} \sum_{k=1}^K \log \left(1 + \frac{4\Bar{\Bar{\xi}}_k^2 \beta T}{\alpha n}\right) + 2\sum_{k=1}^K \left(\log \frac{1}{\Bar{\Bar{b}}_k} + \Bar{\Bar{b}}_k - 1\right)\right] \Bigg\}.
\end{multline*}

\subsubsection{Overall Bound}

In summary, the bound can be written as
\begin{multline*}
    \E_{P_1, \dots, P_T}\E_{S_1, \dots, S_T}\E_{\pi\sim \hat{\Pi}}[\mathcal{E}(\pi)] - \mathcal{E}^* \leq \\
    \inf_{K\in [T]} \Bigg\{\text{CV}_{\text{finite}}(K, n) + K\times \text{CV}_{\text{Gaussian}}\left(d, \Sigma_K(\mathcal{P}), n, T\right) + \text{CV}_{\text{meta}}^{\text{unknown}}(T, n, d, K, \Bar{\Bar{b}}, \Bar{\Bar{\xi}}^2, \tau)\Bigg\},
\end{multline*}
where $\text{CV}_{\text{finite}}(K, n)$ and $\text{CV}_{\text{Gaussian}}\left(d, K, \Sigma_K(\mathcal{P}), n, T\right)$ are exactly the same terms as in the case where the number of mixtures $K$ is known, and the convergence term at the meta level becomes
\begin{equation*}
    \text{CV}_{\text{meta}}^{\text{unknown}}(T, n, d, K, \Bar{\Bar{b}}, \Bar{\Bar{\xi}}^2, \tau) = \text{CV}_{\text{meta}}(T, n, d, K, \Bar{\Bar{b}}, \Bar{\Bar{\xi}}^2, \tau) + \frac{2\log T}{\beta T}.
\end{equation*}

\noindent Even when the number of mixtures $K$ is unknown, the same bound as in the case of $K$ known can be achieved up to a $\frac{2\log T}{\beta T}$ term, which is the order of the time required to find the optimal number of components in the mixture at the meta level.

\end{document}